\documentclass{article}
\usepackage[backend=biber,style=numeric-comp,sorting=none,defernumbers=true]{biblatex}
\DeclareFieldFormat{numeric-comp}{\mkbibbrackets{#1}}
\AtEndPreamble{
    \DeclareRefcontext{apdx}{labelprefix={}}
}

\usepackage[margin=1in]{geometry}
\usepackage[utf8]{inputenc} % allow utf-8 input
\usepackage[T1]{fontenc}    % use 8-bit T1 fonts
\usepackage{xcolor}         % colors
\usepackage{url}            % simple URL typesetting
\usepackage{booktabs}       % professional-quality tables
\usepackage{amsfonts}       % blackboard math symbols
\usepackage{nicefrac}       % compact symbols for 1/2, etc.
\usepackage{microtype}      % microtypography
\usepackage{todonotes}
\usepackage{enumitem}
\usepackage{amsmath}
\usepackage{amsthm,amsfonts,amssymb}
\usepackage{graphicx,bm,amsthm,bbm}
\usepackage{graphbox}
\usepackage{multirow}
\usepackage{subcaption}
\definecolor{darkpastelblue}{rgb}{0.66, 0.13, 0.24}

\usepackage{multirow}
\usepackage[colorlinks, anchorcolor=darkpastelblue, linkcolor=darkpastelblue, urlcolor=darkpastelblue, citecolor=darkpastelblue]{hyperref}
\usepackage{cleveref}
\usepackage{cancel}

\usepackage{longtable}
\usepackage{algorithm,algorithmicx}
\usepackage[noend]{algpseudocode}
\usepackage{amsmath,amsfonts,bm}

\def\eqref#1{equation~\ref{#1}}
\def\1{\bm{1}}

\def\rvu{{\mathbf{i}}}

\def\rvu{{\mathbf{u}}}
\def\rvv{{\mathbf{v}}}
\def\rvw{{\mathbf{w}}}
\def\rvx{{\mathbf{x}}}

\def\rvz{{\mathbf{z}}}

\DeclareMathAlphabet{\mathsfit}{\encodingdefault}{\sfdefault}{m}{sl}
\SetMathAlphabet{\mathsfit}{bold}{\encodingdefault}{\sfdefault}{bx}{n}

\DeclareMathOperator*{\argmin}{arg\,min}

\usepackage{amssymb}
\usepackage{amsthm}
\usepackage{mathrsfs}
\usepackage{blindtext}
\usepackage{siunitx} % Add this line in your preamble
\newtheorem{thm}{Theorem}[section]
\newtheorem{lemma}[thm]{Lemma}
\newtheorem{defn}[thm]{Definition}
\newtheorem{proposition}[thm]{Proposition}
\newtheorem{corollary}[thm]{Corollary}
\newtheorem{ass}[thm]{Assumption}
\newtheorem{exmp}[thm]{Example}
\newtheorem{rmk}[thm]{Remark}
\newtheorem{fact}[thm]{Fact}
\newtheorem*{claim*}{Claim} 
\newtheorem{claim}{Claim} 

\usepackage{bbding}
\usepackage{pifont}
\newcommand{\dashint}{\displaystyle\int\!\!\!\!\!\!-}
\newcommand{\ooo}[0]{\mathscr{O}}
\newcommand{\hhh}[0]{\mathcal{H}}
\newcommand{\kkk}[0]{\mathcal{K}}
\newcommand{\aaa}[0]{\mathcal{A}}
\newcommand{\fff}[0]{\mathfrak{F}}
\newcommand{\nnn}[0]{\mathbb{N}}
\newcommand{\eee}[0]{\mathbb{E}}
\newcommand{\rrr}[0]{\mathbb{R}}
\newcommand{\cmark}{\ding{51}}
\newcommand{\xmark}{\ding{55}}

\newcommand{\ovu}[0]{\overline{u}}
\newcommand{\lie}[0]{\mathfrak{g}}

\newcommand{\xxx}[0]{\bm \xi}

\newcommand{\dv}[1]{\partial_{\rvv_{#1}}}
\newcommand{\dz}[1]{\partial_{\rvz_{#1}}}

\usepackage{authblk}

\title{Coarse Graining with Neural Operators \\for Simulating Chaotic Systems}
\author[1]{Chuwei Wang}
\author[2]{Boris Bonev}
\author[2]{Julius Berner}
\author[1]{Zongyi Li}
\author[3]{Di Zhou}
\author[1]{Jiayun Wang}
\author[2]{Thorsten Kurth}
\author[3]{H.~Jane~Bae}
\author[1]{Anima Anandkumar\thanks{Correspondence to: Anima Anandkumar \texttt{<anima@caltech.edu>}}}

\affil[1]{Department of Computing and Mathematical Sciences, California Institute of Technology, 1200 E California Blvd, Pasadena, CA 91125, United States}
\affil[2]{NVIDIA Research, 2788 San Tomas Express Way, Santa Clara, 95051, CA, United States }
\affil[3]{Graduate Aerospace Laboratories, California Institute of Technology, Pasadena, CA 91125, United States}

\affil[ ]{\texttt{\{chuweiw,zongyili,dizhou,peterw,jbae,anima\}@caltech.edu, \{bbonev,jberner,tkurth\}@nvidia.com}}

\begin{document}

\maketitle

%to approximate the overall information from fine scales, not captured by solvers on the coarse grids

%To achieve data efficiency, we design a multi-fidelity multi-resolution physics-informed training curriculum. The discretization-invariance property of neural operators allows us to train progressively at multiple resolutions

\begin{abstract}
\label{abs}
Predicting long-term behaviors of chaotic systems is central to many applications such as turbulence modeling. Numerical simulations require fine grids to fully resolve the dynamics, which are computationally expensive, and often intractable, for real-world scenarios. A central question is whether the cheaper approach of \textit{coarse graining}, which simulates adjusted dynamics on coarser grids, can match the accuracy of fully-resolved simulations. A dominant coarse-graining approach, known as closure modeling, uses an additional closure model to correct the errors made by a coarse-grid numerical solver. We provide a novel theoretical analysis that proves learning-based closures fare poorly in both data efficiency and training stability. Crucially, the amount of fully-resolved training data needed for closure models exceeds what is needed to estimate the desired long-term statistics of chaotic systems directly from such data; thus defeating the fundamental goal of coarse graining to reduce costs. 

We provide an alternative approach that directly learns the solution operator of short-term evolution using neural operators. It not only circumvents the limitations of closure models, but also achieves statistically optimal estimation of long-term statistics of chaotic systems under standard assumptions.  To achieve data efficiency, we progressively train the neural operator model at multiple resolutions: first with coarse-grid solver data, then fine-tuned with a small amount of fine-grid simulations and physics constraints. Our extensive experiments on representative systems, such as {frontier model training  for probabilistic weather forecasting} and the chaotic dynamics of the forced Navier Stokes equations, demonstrate the superiority of our approach, both in data efficiency and in matching the accuracy of fully-resolved simulations while being orders of magnitude faster.

This paper presents a paradigm shift and goes against the popular narrative of retaining existing numerical simulations and using machine learning only to correct them with closure models. Instead, we advocate fully replacing numerical solvers with machine-learned models, and counterintuitively, it is more data efficient and easier to train. This is because in nonlinear chaotic systems, information is highly entangled across scales, which neural operators can learn seamlessly, while closure models are unable to, due to their rigid scale-separation structure.

\end{abstract}

\section{Introduction}
\label{intro}

Multiscale structures are ubiquitous in nature, where interactions at microscopic scales give rise to emergent macroscopic behavior. Simulations are widely adopted to study these systems, offering an alternative to time consuming or even impractical physical experiments. However, fully-resolved numerical simulations incur a steep computational cost since they need a fine grid to accurately capture all the relevant scales. For instance, a fully-resolved simulation of a small region in the atmosphere takes several months and petabytes of memory \cite{ravikumar2019gpu,schneider2017climate}. \textit{Coarse Graining} (CG) seeks to reduce computational costs while matching the accuracy of fully resolved simulations by computing adjusted dynamics on coarse grids. It has different names in various disciplines including homogenization, renormalization groups and large-eddy simulations \cite{pope2004ten, wilson1971renormalization,charalambakis2010homogenization}.

A popular strategy for coarse graining is to use high-order numerical schemes on a coarse grid to reduce discretization errors and introduce numerical dissipation that implicitly accounts for unresolved dynamics~\cite{grinstein2007implicit}. 
Nevertheless, a more effective and widely adopted approach is to incorporate an explicit \textit{closure model} to correct the errors introduced by coarse-grid numerical solvers \cite{smagorinsky1963general,zhang2024iterated}.
{In this work, following common practice, by closure model we mean the additive version and coarse graining refers to a general scheme on coarse grids.}
Traditional closure models are designed based on physical insight, which requires substantial domain expertise or are derived by mathematical simplification under idealized modeling assumptions, such as scale separation. While these models exploit the presence of an inertial range and capture some degree of scale interactions, the nonlinear coupling across scales in many real-world flows is stronger and more complex than assumed, posing challenges to such models.

Further, many multiscale physical systems are inherently chaotic, meaning they exhibit extreme sensitivity to small perturbations~\cite{ottino1990mixing, sussman1992chaotic, korn2003there}. Examples include climate modeling, aircraft design, and plasma evolution in nuclear fusion~\cite{lima2000plasma,flato2014evaluation,schneider1974climate,slotnick2014cfd,wootton1999fluid}, all of which require studying long-term behavior. The study of chaos has a rich history, dating back to Lorenz’s work on atmospheric turbulence ~\cite{lorenz2017deterministic}, which highlighted the unpredictability of individual trajectories and the distinctive phenomenon of strange attractors—structures that are central to understanding long-term system behavior. Despite significant progress, such as the work winning the 2021 Nobel Prize in Physics, modeling of chaotic dynamics remains an open challenge.

The task of understanding chaotic systems can be framed as estimating statistics of the system in its dynamical equilibrium.
There have been explorations to sample directly from the dynamical equilibrium \cite{valencialearning}, but this direction remains challenging due to the complex structure of such equilibria \cite{lorenz2017deterministic}. The standard approach continues to rely on long-term numerical simulations, which is challenging for chaotic systems, particularly those with coarse grids, since errors accumulate catastrophically over time. Thus, designing coarse-graining schemes becomes even more challenging for chaotic systems. Further, the 
goal shifts from accurately tracking individual long trajectories, an impossible task in chaotic systems, to reliably capturing long-term statistical properties~\cite{moser2021statistical,valencialearning}.

Traditional closure models, which make rigid assumptions such as scale separation, fail to capture complicated fine-scale effects in many chaotic dynamics~\cite{meneveau2000scale,moser2021statistical,zhou2024sensitivity}. To improve the expressivity of closure models and reduce modeling errors, this past decade has witnessed extensive development of machine-learning methods for closure modeling~\cite{duraisamy2019turbulence, kochkov2021machine, duraisamy2021perspectives,sanderse2024scientific,maddu2024learning,shankar2023differentiable,huang2024consistent}.
However, these learning-based closure models \cite{gamahara2017searching, beck2019deep, maulik2019subgrid, guan2022stable} rely on a large amount of high-fidelity training data from expensive fully-resolved simulations, which is expensive or even impossible to generate for many problems of interest. 
Further, rigorous theoretical understanding of closure modeling remains largely undeveloped, except for certain simplified settings, such as ODE dynamics~\cite{levine2022framework} or systems with clear scale separation where analytical approximations are tractable~\cite{eyink1993renormalization,durbin2018some}, thus leaving a critical gap for more general and complex dynamics, and we address it in this paper.

Recently, there has been an alternative approach that uses only learned models without retaining any coarse grid numerical solvers at inference time. This is accomplished by directly learning the solution operator of the physical system using a neural operator framework~\cite{li2020fourier,kovachki2023neural}. Unlike traditional neural networks, which operate on fixed grids, neural operators approximate mappings between function spaces. By treating inputs at different resolutions as discrete representations of the same continuous function, neural operators ensure consistency across coarse and fine grids in the limit of mesh refinement, which is commonly referred to as \textit{resolution invariance}. 
Neural operators have been widely utilized as surrogate models, achieving impressive acceleration for fully-resolved simulations. 

The work closest to ours is Li et al.~\cite{li2022learning}, who studied using neural operators for learning the short-term evolution and then rolling them out autoregressively to predict long-term statistics in dissipative chaotic systems. However, their approach operates within the framework of fully-resolved simulations, where an ample amount of high-fidelity training data is available and the neural operator model simulates with high-resolution grids to make the prediction. In other words, their setting does not impose computational constraints as in coarse graining, and requires expensive training data.

In this work, we address the above core challenges, use neural operator as a coarse-graining framework, while also drastically reducing training data requirements for fully-resolved simulations.   The resolution-invariant property makes neural operators particularly well-suited for coarse graining since they can be run on grid of any size.

% We target the solution operator directly. Intuitively, the neural operator model stores fine-scale information after training with high-resolution inputs. 
% The missing information from fine-scale is then implicitly incorporated into coarse-grid simulations through nonlinear interactions between inputs and weights within the model (\Cref{fig:ant}(A)).

\begin{figure}[htbp]
\centering

% === First row: one figure centered ===
\begin{subfigure}{\textwidth}
  \centering
  \includegraphics[width=0.8\linewidth]{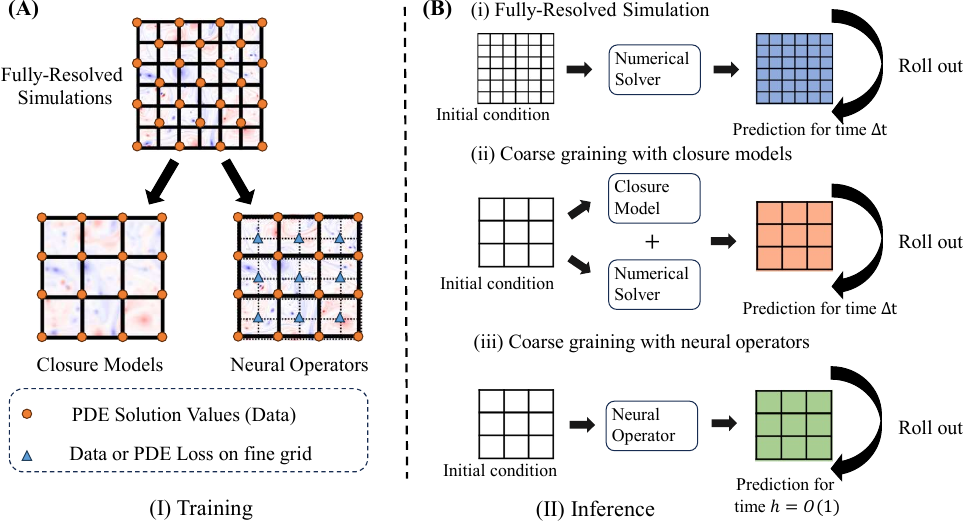} % ← width adjustable
  % \caption{Caption for first row figure}
\end{subfigure}

\vspace{1em} % ← space between rows, adjustable

% === Second row: two side-by-side figures ===
\begin{subfigure}{0.43\textwidth}
  \includegraphics[width=1\linewidth]{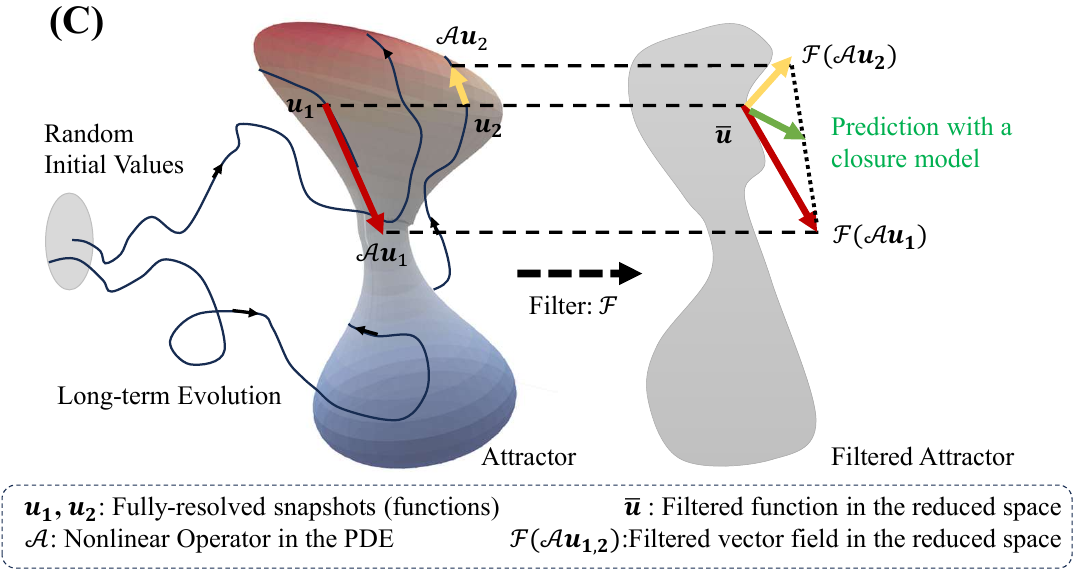} % ← size adjustable
  % \caption{Left subfigure caption}
\end{subfigure}\hfill
\begin{subfigure}{0.57\textwidth}
  \includegraphics[width=1\linewidth]{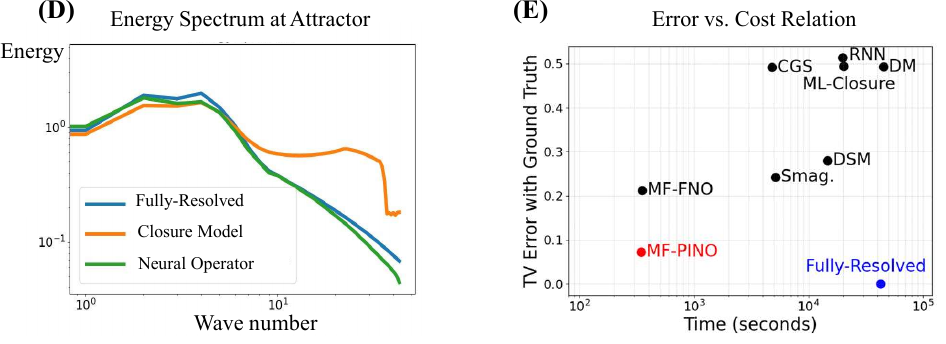} % ← size adjustable
  % \caption{Right subfigure caption}
\end{subfigure}

% \vspace{1em} % optional: space below all figures
\caption{
\textbf{(A-B):} \textbf{Conceptual Illustration of Different Methods (training and inference)}
\textbf{(A) }During training, learning-based closure models can only leverage information captured by the coarse grids, though it is derived from fully-resolved simulations. In contrast, neural operators can directly learn the fine-scale information through either supervised or physics-informed training with fine-grid inputs. 
\textbf{(B) }(i) Fully-resolved simulations are conducted on fine grids. They provide gold-standard estimations of statistics but are extremely costly. (ii) Coarse-graining framework with closure model ($\mathcal{A}\ovu+clos(\ovu;\theta)$), which adopts a closure model in conjunction with a coarse-grid numerical solver.
(iii) Coarse graining with neural operator. After training with fine-grid inputs in (A), neural operators seamlessly support coarse-grid inputs. fine-scale information is implicitly incorporated into coarse-grid simulations through the interaction between inputs and model parameters.
\textbf{(C): Non-Unique Issue of Closure Model.} The arrows represent the evolving direction of trajectories (i.e. the time derivative of the governing dynamics). In chaotic systems, trajectories with slightly different initializations rapidly diverge from one another, but ultimately they all converge to the attractor (if one exists) after long-term evolution.
Coarse graining can be viewed as designing a vector field that drives the dynamics in the reduced space. In closure model framework, the assigned vector field is $\mathcal{A}\ovu+clos(\ovu;\theta)$.
Multiple points (e.g., $u_1$ and $u_2$) of the ground-truth attractor (i.e., equilibrium state) map to the same filtered value $\ovu$, precluding the closure model from identifying the correct dynamics ($\mathcal{F}(\mathcal{A}u_1)$ and $\mathcal{F}(\mathcal{A} u_2)$) in the filtered space $\mathcal{F}(\hhh)$. By minimizing the loss function, the model learns to predict the average of these multiple choices (green arrow), leading the simulation to wrongly diverge from the filtered attractor.
\textbf{(D):}\textbf{ Prediction of energy spectrum at the attractor.} `Closure Model' represents a single-state learning-based closure.
\textbf{(E):}
\textbf{Total variation distance from ground-truth invariant measure versus computation cost. }
% See \Cref{sec: exp} for the abbreviation of method names.
`Fully-Resolved': gold-standard fully-resolved simulations, providing ground-truth statistics. `CGS': coarse-grid simulation without closure model. `Smag.': Smagorinsky closure model.`ML-Closure': learning-based single-state model. `DSM': Dynamical Smagorinsky closure model. `RNN': history-aware closure model with recurrent neural network. `DM': stochastic closure model based on diffusion model.
`MF-FNO': Multi-fidelity FNO.
`MF-PINO': multi-fidelity physics-informed neural operator.
All learning-based methods are trained with same amount of fully-resolved data. Our methods with neural operators (MF-FNO and MF-PINO)
are the fastest. MF-PINO is the closest to ground truth among all methods that take coarse-grid inputs, beating those methods using an explicit closure model.
}
\label{fig:ant}
% \vspace{-2em}
% \vskip -0.7cm
\end{figure}

\textbf{Our Approach:} 
In this paper, we propose a  machine learning (ML) framework for estimating the long-term statistics of any chaotic system. We learn the solution operator for short-term evolution using a neural operator. To compute the long-term statistics, we roll out the trained neural operator autoregressively over many time steps until certain statistical observables converge, indicating that the trajectory has reached dynamical equilibrium.
In contrast to closure modeling, ours is an end-to-end learning approach where no numerical solver is retained in the loop at test time. 
The discretization-invariance property of neural operators enables us to remove the constraint that closure models face, viz.,  the learned closure model is only on a coarse grid. Instead, neural operators employ a discretization-agnostic approach to learning. We can train neural operators on fine-grid simulations, and run them on a coarse grid at inference time to compute long-term statistics, and thus, satisfy the grid constraints of coarse graining.  

%infuse them with fine-scale information by training on fully-resolved simulations.The missing information from fine-scale is then implicitly incorporated into coarse-grid simulations through nonlinear interactions between inputs and weights within the model.

Compared to our neural operator approach, one might intuitively expect closure modeling to be better, since they leverage known coarse-grid numerical solver, and only require learning the unresolved effects in the closure model. However, in nonlinear (and especially chaotic) dynamics, information across scales is often highly entangled. This undermines the implicit assumption that the system can be explicitly divided into a linear combination of a \textit{known} low-frequency component (modeled by the coarse-grid numerical solver) and an \textit{unknown} high-frequency component (learned by the closure model)—a separation that typically relies on idealized assumptions such as scale separation.

In this work, we provide a new theoretical understanding of closure modeling and reveal that they are fundamentally unsuitable for generic chaotic systems.
% We formally prove in \Cref{thm: random} that for generic problems, such learning methods are fundamentally ill-posed since they are constrained to be on the same coarse grid as the solver, not able to accurately approximate the underlying chaotic system. 
Specifically, we  prove in \Cref{thm_clos_all} (i) that the mapping that closure models attempt to approximate in a reduced space (coarse grid) is non-unique, i.e., there are multiple potential outputs for a given input (\Cref{fig:ant} (C)). Hence, the standard approach to learning closure models
% under such non-uniqueness 
results in the average of all possible outputs, and that cannot accurately approximate the long-term statistics, no matter how large or expressive the closure model is.
We also prove that other attempts leveraging posteriori training \cite{guan2022stable,shankar2023differentiable}, history information \cite{ma2018model, wang2020recurrent} and stochasticity such as diffusion models \cite{wu2023learning} into closure modeling cannot overcome this limitation.
In particular, we show that mitigating the non-uniqueness issue of closure models requires previous methods to rely on the closeness between the empirical measure of training data and the limit distribution of the original dynamics, which necessitates extensive fully-resolved simulation data due to the slow convergence of empirical measures in high-dimensional spaces. This contradicts the goal of closure modeling, since such fully-resolved data alone would already suffice for accurate statistical estimation, eliminating the need to train a closure model.

A similar requirement for fully-resolved training data would also hold for neural operators, if done naively, as in prior works~\cite{li2022learning,kovachki2023neural}. We drastically reduce this through a  curriculum approach that builds on multi-resolution training~\cite{lyu2023multi} and physics-informed learning~\cite{goswami2023physics} of neural operators~\cite{li2021physics}.
We train a neural operator model first, on data obtained from coarse-grid numerical solvers. Since such solvers are relatively cheap to run, we can obtain sufficient data to train the neural operator to accurately emulate the coarse-grid solver. We then fine-tune the neural operator using (a small amount of) high-fidelity data and physics-based losses defined
on a fine grid. Since neural operators can operate on any grid, we can employ the same model to train on data from both coarse and fine-grid solvers. The addition of physics-based losses on a fine grid further reduces the requirement for fully-resolved simulations,  and improves generalization, in line with what has been seen in prior works on physics-informed learning ~\cite{li2021physics,wang2021learning}.

{In contrast to standard closure models, the neural operator model stores fine-scale information after training with high-resolution inputs. The missing information from the fine scales is
then implicitly incorporated into coarse-grid simulations through nonlinear interactions between inputs and
weights within the model.
Viewed under this lens, the Neural Operator approach can be seen as a generalization of the closure model,  where coarse graining is possible through non-linear mixing across scales, and learning can happen progressively using multi-resolution multi-fidelity training samples and physics knowledge.
}

We provide theoretical guarantees for the accuracy of neural operator-based coarse graining. 
The class of neural operators has been previously established as universal approximators in function spaces~\cite{kovachki2023neural,lanthaler2023nonlocal}, meaning they can accurately approximate any continuous operator (at any resolution). We further strengthen this result in this paper, and prove in \Cref{thm: pino} that a neural operator that approximates the underlying ground-truth solution operator can also provide sufficiently accurate estimates of the long-term statistics of a chaotic system, and there is no catastrophic build-up of errors over long autoregressive rollouts, under certain conditions.
We derive all the above theoretical results through the lens of measure flow in function spaces, introducing a novel theoretical framework, viz., functional Liouville flow.

We test the performance of our approach in several chaotic dynamical systems, including a forced homogeneous isotropic turbulence with Reynolds number $1.6\times 10^4$.
Forced turbulence is in general more challenging to model compared to decaying homogeneous isotropic turbulence, which many learning-based closure models have previously focused on (e.g. \cite{ guan2022stable, boral2023neural}). This is because forced turbulence requires continuous energy input to maintain a steady state and involves more complex multi-scale interactions. Forced turbulence thus provides richer statistical data and deeper insights into turbulence cascade physics.
% We
% test the performance of our approach in several instances from fluid dynamics, including a challenging chaotic system with Reynolds number $1.6\times 10^4$.
Our multi-fidelity physics-informed neural operator (MF-PINO), trained on multi-fidelity data under a curriculum approach as described above, achieves a 330x speedup over fully-resolved numerical solver, and achieves with a relative error  $\sim 10\%$  on the energy spectrum. In contrast, the learned closure model coupled with a coarse-grid numerical solver is $60$x slower than   PINO   while having a much higher error $\sim186\%$. 

% \brs{Can you help provide a few sentences briefly introducing the climate / weather model task you considered?}

The closure model has a significantly higher error under our setup compared to prior works~\cite{guan2022stable,maulik2019subgrid,ma2018model}. This is because we assume that only a few samples from fully-resolved simulations are available for training, both for PINO and the closure model, viz., just $\sim 10^2$-time steps from a single fully-resolved trajectory. In contrast, prior works on learning closure models assume thousands of time steps over hundreds of fully-resolved trajectories.  More details about the speed-accuracy performance of different approaches are shown in \Cref{fig:ant}(E).

{Finally, we apply our approach to machine-learning weather prediction 
to demonstrate the effectiveness of our approach on large, scientific problems. To this end, we train FourCastNet 3~\cite{bonev2025fourcastnet}, a frontier probabilistic weather prediction model based on geometric neural operators \cite{bonev2023spherical,liu2024neural}, using our multi-fidelity curriculum approach. Compared to the high-fidelity model, the multi-fidelity model achieves a 2.78x reduction in training time while maintaining comparable validation performance, as well as preserving spectral fidelity and high-resolution predictive skill of the high-fidelity ensemble.}

\section{Results}
We will first introduce the background of coarse graining (CG) and long-term statistics of chaotic systems. Then we will introduce neural operators and our method with multi-fidelity physics-informed neural operator (MF-PINO) along with its theoretical advantages. After that we will review the closure modeling approach and present our main theoretical results demonstrating why closure modeling ansatz is unsuitable for data-driven methods. Finally, we present the experiment evaluations that display superior performance of MF-PINO on estimating long-term statistics with coarse-grid simulations. See \Cref{fig:stat_main,fig:stat_main_hiRe} for some of the experimental results.

% both the problem and neural operator, followed by a discussion of the closure modeling approach and our method with physics-informed neural operator (PINO). Then we will present our main theoretical results demonstrating why closure modeling ansatz is unsuitable for data-driven methods, and the accuracy guarantee for PINO. Finally, we present the experiment evaluations that display superior performance of PINO on estimating long-term statistics with coarse-grid simulations. See \Cref{fig:stat_main,fig:stat_main_hiRe} for some of the experiment results.

\subsection{Background}
% \section{Background and Existing Methods}
\label{sec: sec2}
% We formally introduce the problem setting of evaluating long-term statistics and the goal of coarse graining and closure modeling. See \Cref{app:apdxA} for more backgrounds.
% \paragraph{Estimating Long-term Statistics with Coarse-grid Simulation}
Consider an evolution partial differential equation (PDE) that governs a (nonlinear) dynamical system in the function space, 
\begin{equation}
\label{eq:general-pde}
    \begin{cases}
    \partial_t u(x,t)=\mathcal{A}u(x,t)\\
    u(x,0)=u_0(x),\ u_0\in \mathcal{H},
    \end{cases}
\end{equation}
where $\aaa$ is a nonlinear differential operator, $u_0$ is the initial value and $\mathcal{H}$ is a function space containing functions of interests, e.g., fluid field, temperature distribution, etc. 
Boundary and regularity conditions are included in the notion of $\hhh$.
% The operator $\mathcal{A}$ is written as $\mathcal{L}+\mathcal{N}$, which are its linear and nonlinear parts. 
This equation naturally induces a semigroup $\{S(t)\}_{t\geq 0}$ defined as the mapping from the initial state to the state at time $t$, $S(t):u_0\to u(\cdot, t) $. We refer to the set $\{S(t)u_0\}_{t\geq 0}$ as \textit{trajectory from $u_0$}.

\paragraph{Attractor, invariant measure, long-term statistics:}
The (global) \textit{attractor} of the dynamics is defined as the maximal invariant set of $\{S(t)\}_t$ towards which all trajectories converge over time. For many relevant systems, 
the existence of a compact attractor is either rigorously proved~\cite{temam2012infinite,milnor1985concept} or demonstrated by extensive experiments~\cite{kuznetsov2011dynamical,dinicola2011systems,huang2009cancer}. The \textit{invariant measure} is the time average of any trajectory, independent of initial value as long as the system is ergodic,
% $    \mu^*:=\lim_{T\to\infty}\frac 1 T\int_{t=0}^T\delta_{S(t)u}dt,\  u\in\mathcal{H},\ a.e.$
\begin{equation}\label{inv_mes}
    \mu^*:=\lim_{T\to\infty}\frac 1 T\int_{t=0}^T\delta_{S(t)u}dt,\  u\in\mathcal{H},\ a.e.
\end{equation}
where $\delta$ is the Dirac measure. $\mu^*$ is a measure of functions and is supported on the attractor. Intuitively, the invariant measure captures the system's long-term behavior when it reaches a dynamical equilibrium. The \textit{long-term statistics} are expectations of functionals on the invariant measure. 
% They have great significance both in applications[cite] and research[cite], e.g. aircraft design and climatic prediction(?is it?)[cite].
The most straightforward approach to estimate statistics is to run an accurate simulation of trajectory and compute following the definition. In practice, we first fix a sufficiently large $T$ and then choose spatiotemporal grid size accordingly so that the overall error of the simulation within $[0,T]$ remains small. We will refer to this approach as high-fidelity simulations or fully-resolved simulations. 
See \Cref{app:apdxA} for more background and \Cref{app:ass} for assumptions.

Chaotic systems, characterized by positive Lyapunov exponents~\cite{medio2001nonlinear}, are known for their extreme sensitivity to perturbations and catastrophic accumulation of small errors over time.
To account for the unstable nature of chaotic systems, high-fidelity simulations have to be carried out on very dense spatiotemporal grids to make discretization errors small enough so that the overall error along the trajectory does not grow rapidly. This makes the fully-resolved simulations prohibitively expensive.

 \paragraph{Coarse Graining}
Given the computation cost of fully-resolved simulations and the fact that the ultimate goal is to evaluate the statistics instead of any individual trajectory, many works have been exploring ways to\textit{ give good estimations of statistics with simulations only conducted on coarse grids}, known as \textit{coarse graining}. Simulation on coarse grids is equivalent to evolving a filtered function $\ovu$ defined as $\ovu=\mathcal{F}u$, where $\mathcal{F}$ is a linear filtering operator. 
{Here $\mathcal{F}$ is understood in a generalized sense, abstracting from specific forms in various coarse-graining problems. Our results do not rely on its exact representation. 
In cases where only partial observables are available (e.g., coarse-grid measurements or certain physical quantities), $\mathcal{F}$ acts as a quotient map, identifying functions that yield indistinguishable observations.
}
The target of course-graining is to design a modified operator $\hat{\aaa}\ovu$ which governs the evolution of $\ovu$ in the reduced system $\mathcal{F}(\hhh)$ and enables good estimations of long-term statistics. In this work, we consider two approaches to coarse graining: the conventional closure-based method (either hand-crafted or data-driven) and a neural operator approach that directly models the effective dynamics on the coarse grid.

% We will refer to this approach as coarse-grid simulation (CGS). It serves as the core focus of this work.
% Let us formalize it as follows.

% Denote by $D$ the set of grid points used in fully-resolved simulations and $D'$ that in coarse-grid simulations, with $|D'|\ll|D|$.
% Simulating on $D'$ could be viewed as evolving a filtered function $\ovu$ defined as $\ovu=\mathcal{F}u$, where $\mathcal{F}$ is a linear filtering operator. For instance, $\mathcal{F}$ is a spatial convolution for cases like down-sampling in the finite difference method and Fourier-mode truncation in the spectral method.
% In general, directly simulating $\partial_t\ovu=\aaa\ovu$ will suffer from a catastrophic error. The target of course-graining is to design a modified operator $\hat{\aaa}\ovu$ which governs the evolution of $\ovu$ in the reduced system $\mathcal{F}(\hhh)$ and enables good estimations of long-term statistics.
% In this work, we consider two approaches to coarse graining: the conventional closure-based method (either hand-crafted or data-driven) and a neural operator approach that directly models the effective dynamics on the coarse grid.

\subsection{Multi-fidelity Physics-informed Neural Operator for Coarse Graining}

% \paragraph{Neural Operators}
% Neural operators are a class of recently-emerged machine learning models designed for solving PDEs, an area also known as operator learning.
% The goal of operator learning \cite{kovachki2023neural,lu2021learning} is to approximate mappings between infinite-dimensional function spaces rather than vector spaces. 
% One of the representatives is Fourier Neural Operator (FNO)~\cite{li2020fourier}, which serves as the backbone of our model, and its architecture is illustrated in \cref{fig:archi}.
% % whose architecture
% % can be described as: 
% % \begin{equation}
% %     \mathcal{G}_{FNO}:=\mathcal{Q} \circ ({W}_L + \mathcal{K}_L) \circ \cdots \circ \sigma ({W}_1 + \mathcal{K}_1) \circ \mathcal{P},
% % \end{equation}
% % where $\mathcal{P}$ and $\mathcal{Q}$ are pointwise lifting and projection operators. The intermediate layers consist of an activation function $\sigma$, pointwise operators $W_{\ell}$ and integral kernel operators $K_{\ell}:\ u\to \mathscr{F}^{-1}(R_{\ell} \cdot \mathscr{F}(u))$, 
% % where $R_{\ell}$ are weighted matrices and $\mathscr{F}$ denotes Fourier transform.
% Neural operators support input from different resolutions (grid sizes), and inputs are all viewed as discretization of an underlying function. 
Neural operators~\cite{kovachki2023neural,li2020fourier} have been vastly used as surrogate models. In this work, we will provide a new perspective that they naturally serve as a coarse-graining scheme with advantageous properties for long-term statistics.

\paragraph{Neural Operator as Coarse-graining Scheme }
As will be discussed subsequently, restricting the learning task in the filtered space suffers from the non-uniqueness of the associated mapping.
We propose to extend the learning task into function space $\mathcal{H}$ and directly deal with the solution operator $S(t)$ of the governing PDE, which is a well-defined mapping. We approximate it with a neural operator model. We adopt FNO as the backbone of our model. In particular, we employ FNO to learn the mapping $u\to \{S(t)u\}_{t\in[0,h]}$, where $h$ is a model parameter, usually of $O(1)$ scale. We discuss our empirical way to choose $h$ in \Cref{apdx-choose-h}.
Once trained to achieve low prediction error with fine-grid inputs, the model is seamlessly applied to coarse-grid initial conditions and rolled out autoregressively (i.e., using model output as input at the next step) to generate trajectories in the coarse-grid system.
Conceptually, the infinitesimal generator $\tilde{A}_\theta$ of the coarse-grid dynamics generated above plays the role of the modified operator $\hat{\aaa}$ in coarse graining.
$\tilde{A}_\theta$ accounts for both large- and fine-scale information, enabling nonlinear interactions between different scales. 
Intuitively, the neural operator model stores fine-scale information after training with high-resolution inputs. 
During coarse-grid rollouts at inference time, the missing information from fine-scale is implicitly incorporated through nonlinear interactions between inputs and weights within the model.

% we switch from the previous ansatz $\hat{A}=\aaa\ovu+clos(\ovu;\theta)$ to $\hat{\aaa}=\tilde{A}_\theta\ovu$, where the parameterized model 

% This idea is instantiated with neural operators, which 

% Neural operators support inputs from varying resolutions and ensure consistency across them. 

\paragraph{Multi-fidelity Physics-informed Neural Operator}
In practice, fine-grid training data from fully-resolved simulations are very expensive to generate.
Our practical method
is based on physics-informed neural operator (PINO) \cite{li2021physics}, and also drawing inspiration from multi-resolution training \cite{lyu2023multi}.
% The PINO model learns to approximate the semigroup $S(t)$, thus the infinitesimal generator of the learned operator corresponds to  $\tilde{A}_\theta$ in the discussion above.

% The backbone is a Fourier neural operator (FNO) \cite{li2020fourier}. In particular, we employ FNO to learn the mapping $u\to \{S(t)u\}_{t\in[0,h]}$, where $h$ is a model parameter, usually of $O(1)$ scale. 

The neural operator model is first pre-trained with cheap coarse-grid simulation (CGS) data and a small amount of fully-resolved data, and then trained by minimizing the physics-informed loss \cite{raissi2019physics} computed on fine grids. Arbitrary fine-grid input functions can be used to compute physics-informed loss instead of only input-output pairs from expensive simulations.
This approach significantly reduces the requirement on high-fidelity data from fully-resolved simulations.

\subsubsection{Theoretical Result}
\begin{thm}\label{thm: pino}
% \vspace{-1em}
(Approximation Guarantee for Course-Graining with Neural Operator)\\
% For any \( h>0 \) and \( \epsilon>0 \), 
Suppose a neural operator \( \mathcal{G}_\theta \) is trained to approximate the system evolution over time \( h \), and long-term statistics are estimated by iterating \( \mathcal{G}_\theta \) from an arbitrary coarse-grid initial condition. 

If the attractor is a hyperbolic set of the dynamics, then for any $\epsilon$ there exists a threshold \( \delta>0 \) such that if the per-step prediction error of \( \mathcal{G}_\theta \) in the original space $\hhh$ is below \( \delta \), the resulting statistical distribution approximates the filtered invariant measure within \( \epsilon \) accuracy.
\end{thm}

\begin{rmk}
(i) The hyperbolicity assumption is technical but standard in the analysis of chaotic attractors. Many physical systems of interest are believed to have hyperbolic structure~\cite{eckmann1985ergodic,chandler2013invariant,wang2014least}.\\
(ii) The result holds for any $h>0$. In practice, $h\sim O(1)$ and this enables faster convergence to the attractor in coarse-grid simulations.
\end{rmk}

%As suggested by \cite{de2022generic}, for a neural operator that is expressive enough and trained by supervised learning with $N$ high-fidelity data points from fully-resolved simulations, $\delta=O(N^{-\frac 1 2})$. When the optimization is successful, physics-informed training can further reduce the data requirement.

The universal approximation capability of neural operators has been previously proven~\cite{kovachki2023neural,lanthaler2023nonlocal}, and can be applied for approximating the solution operator for short-term evolution. However, when the learned neural operator is autoregressively rolled out to estimate the long-term statistics, one might worry that in chaotic systems, even small prediction errors will quickly accumulate, implying that a neural operator must match the ground truth with near-perfect accuracy, which is unrealistic. However, we have \Cref{thm: pino},
% Intuitively, we show that there exists a true fully-resolved trajectory (from a different initial value) whose filtered function is consistently close to the simulation we get with approximate neural operator and coarse-grid initialization. This is grounded in the resolution-invariance of neural operators. Since the invariant measure is independent of the initial condition, we obtain a good approximation of the (filtered) invariant measure.
% This result show
which shows that even if the trained operator jumps a large step $h$ in time and has errors as is in practice, we can still obtain a good estimation of statistics by rolling it out and computing the time averaged statistics. In practice, we notice a $10\%\sim20\%$ $L^2$- relative error   of single-step prediction suffices. Moreover, our method achieves faster convergence to the attractor because we roll it out using a $O(1)$ time step $h$ instead of fine temporal grids to compute the long-term statistics.
 The formal statement and proof of \Cref{thm: pino} can be found in \Cref{apdx: thm_pino}.

\subsection{Coarse Graining with Data-driven Closure Models}
\subsubsection{Closure Modeling}\label{sec:learn_clos}
Recall that coarse graining of the dynamics \cref{eq:general-pde} corresponds to a evolution of $\ovu$ in the filtered space $\mathcal{F}(\hhh)$.
Theoretically, the evolution of $\ovu$ is governed by $\partial_t\ovu=\mathcal{F}\mathcal{A}u=\mathcal{A}\ovu+(\mathcal{F}\mathcal{A}-\mathcal{A}\mathcal{F})u$, where $\mathcal{F}$ and $\mathcal{A}$ are not commuting due to the nonlinearity of $\mathcal{A}$.
However, the 
% so-called
commutator $(\mathcal{F}\mathcal{A}-\mathcal{A}\mathcal{F})u$ is intractable if restricted to coarse-grid systems (denoted by $D'$) since $u$  contains unresolved components and is unknown. To account for the effect of fine scales not captured by $D'$, in many coarse-grid simulation methods an adjusting term ${clos}(\ovu;\theta)$ ($\theta$ denotes the model parameters), known as \textit{closure model}, is added to the equation as a tractable surrogate model of $(\mathcal{F}\mathcal{A}-\mathcal{A}\mathcal{F})u$. The CGS trajectory is derived by simulating
\begin{equation}
\label{eq:les-pde}
    \begin{cases}
    \partial_t {v}(x,t)=\mathcal{A}{v}(x,t)+{clos}({v};\theta),\ x\in D'\\
    {v}(x,0)=\ovu_0(x),\ \ovu_0\in \mathcal{F}(\mathcal{H}),
    \end{cases}
\end{equation}
and the statistics are estimated as the time average of the corresponding functionals with ${v}(\cdot,t)$ input. We use the notation $v$ instead of $\ovu$ here to underscore the difference between coarse-grid trajectories and filtered fully-resolved trajectories, as they follow different dynamics in general. In the framework of coarse graining (as in former subsection), the modified operator $\hat{\aaa}$ in the reduced system corresponds to $\aaa+clos(\cdot;\theta)$.

% \textbf{Classical Closure Models:}
Closure modeling is a classical and relevant topic in computational methods for science and engineering, with rich literature available \cite{meneveau2000scale, moser2021statistical}. 
Despite their wide application in numerous scenarios, the design of closure models is more of an art than science. Many existing methods are grounded in physical intuition or derived by mathematical simplifications that incorporate strong assumptions, which often do not hold up under general conditions.
Additionally, selecting parameters in these closure models typically requires substantial domain expertise. Nevertheless, for several practical applications, such simple modeling assumptions are not sufficient to capture the complex optimal closure \cite{zhou2024sensitivity}.

% This decade has witnessed how data-driven approaches phenomenally extend traditional elegant 
% handcrafted methods in numerous areas. 
In recent years, there has been a growing interest in leveraging machine learning tools to design closure models (see \cite{sanderse2024scientific} for survey). 
However, in this work we will reveals the potential shortcomings of all mainstream methodologies following previous closure modeling ansatz.

\begin{table}[t]
    \caption{Comparison between different approaches for predicting long-term statistics of Navier-Stokes equations. The Reynolds number $Re$ is large in most applications. The top two are classical approaches, and the rest are machine learning approaches. Training data is counted in the number of snapshots and trajectories. The complexity takes into account both spatial grids and temporal grids. 
    Our approach is even cheaper than coarse-grid simulations because ours can evolve with $O(1)$ time step instead of small time-grids following the CFL condition, as is the case for other methods utilizing a coarse solver. $\delta t$ is the time-grid size for latent SDE in \cite{boral2023neural} or reverse SDE for sampling in the diffusion model as in \cite{dong2024data}.
    }
    \centering
    \resizebox{\textwidth}{!}{\begin{tabular}{l@{\hspace{0.3cm}}c@{\hspace{0.3cm}}c@{\hspace{0.3cm}}c}\toprule
    \multirow{2}*{Method} & Optimal & High-res. training data \@ & \multirow{2}*{Complexity} \\ 
    & statistics &  Snapshots \big| Trajs. & \\ \midrule
    Fully-resolved Simulation, e.g., DNS \cite{reynolds2005potential,choi2012grid} & \cmark & - & $Re^{3.52}$ \\
    Coarse-grid Simulation, e.g., LES \cite{reynolds2005potential, choi2012grid} & \xmark & - & $Re^{2.48}$ \\ \midrule
    Single-state model \cite{guan2022stable} & \xmark & 24000 \big| 8 & $Re^{2.48}$ \\
    History-aware model~\cite{wang2020recurrent} & \xmark & 250000 \big| 50 & $Re^{2.48}$ \\
    Latent Neural Stochastic Differential Equation (SDE)~\cite{boral2023neural} & \xmark & 179200 \big| 28 & $\frac{1}{\delta t}Re^{1.86}$ \\
    Diffusion Model-based stochastic closure model ~\cite{dong2024data} & \xmark & 150000 \big| 10 & $\frac{1}{\delta t}Re^{2.48}$ \\
    Online Learning \cite{sirignano2023dynamic} & \xmark & - & {$Re^{3.52}$} \\
    \textbf{Physics-Informed Operator Learning (Ours)} & \cmark & \phantom{00}384 \big| 1 & $Re^{1.86}$ \\
    \bottomrule
    \end{tabular}}
    \label{tab:compare}
    \end{table}

% \subsection{Non-uniqueness Issue of Closure Modeling Ansatz}\label{sec:learn_clos}
We summarize mainstream learning-based closure modeling methods as follows.
% and analyze their shortcomings as follows. 
% All the theoretical results are listed in \Cref{fig:thm}.

\textbf{Learning single-state closure model:} %\label{sec:learn_clos}
Broadly speaking, neural network-represented closure models ${clos}(v;\theta)$ are proposed with various ansatz and neural network architectures, and they are trained by minimizing an a priori loss function aiming at fitting the commutator~\cite{sanderse2024scientific}, 
% $    J_{ap}(\theta;\mathfrak{D})=\frac {1} {|\mathfrak{D}|}\sum\limits_{i\in\mathfrak{D}}\|clos(\ovu_i;\theta)-(\mathcal{F}\mathcal{A}-\mathcal{A}\mathcal{F})u_i\|^2,$
\begin{equation}\label{mthd1}
    J_{ap}(\theta;\mathfrak{D})=\frac {1} {|\mathfrak{D}|}\sum\limits_{i\in\mathfrak{D}}\|clos(\ovu_i;\theta)-(\mathcal{F}\mathcal{A}-\mathcal{A}\mathcal{F})u_i\|^2,
\end{equation}
where the training data $u_i$ come from snapshots of trajectories from fully-resolved simulations, i.e., $S(t)u_0$ for particular $t$.
There are also extensions of the learning framework above. Some works proposed to add a posterior loss
into the training objective \cite{sirignano2020dpm,list2022learned}, see Method section \ref{mthd_baseline} for more descriptions.

\textbf{Learning history-aware closure model:}
Some works \cite{ma2018model,wang2020recurrent} propose to take account of history information in the reduced space, namely a closure model whose input is $\{\ovu(x_i,t-s)\}_{x_i\in D',\ 0<s\leq t_0}$ at the moment $t$, where $t_0$ is a model parameter. However, we will prove below that this approach does not resolve the non-unique issue.

\textbf{Stochastic formulation of closure model:}
Another direction is to replace the deterministic closure model with a stochastic one, such as diffusion or flow-based models \cite{boral2023neural,lu2017data,dong2024data}. This line of work is inspired by \cite{langford1999optimal}, which shows that the optimal choice of closure model has the form of a conditional expectation.

We want to emphasize that in the context of learning closure models, it is not possible to incorporate multi-resolution pretraining and physics-informed learning as applied in our method with neural operator. The fundamental reason is that the learning procedure of these closure models is restricted to the fixed coarse grids. For coarse-grid functions $\ovu$, the closure term $(\mathcal{F}\aaa-\aaa\mathcal{F}){\ovu}=0$, thus the model cannot acquire any information from this data point. For physics-informed learning, since the PDE in the coarse-grid system contains an unknown closure term, one cannot compute the PDE residual with coarse-grid input functions.

% part on FourCastNet 3 % move to fig2
\begin{figure}[t]
    \centering
    \includegraphics[width=\linewidth]{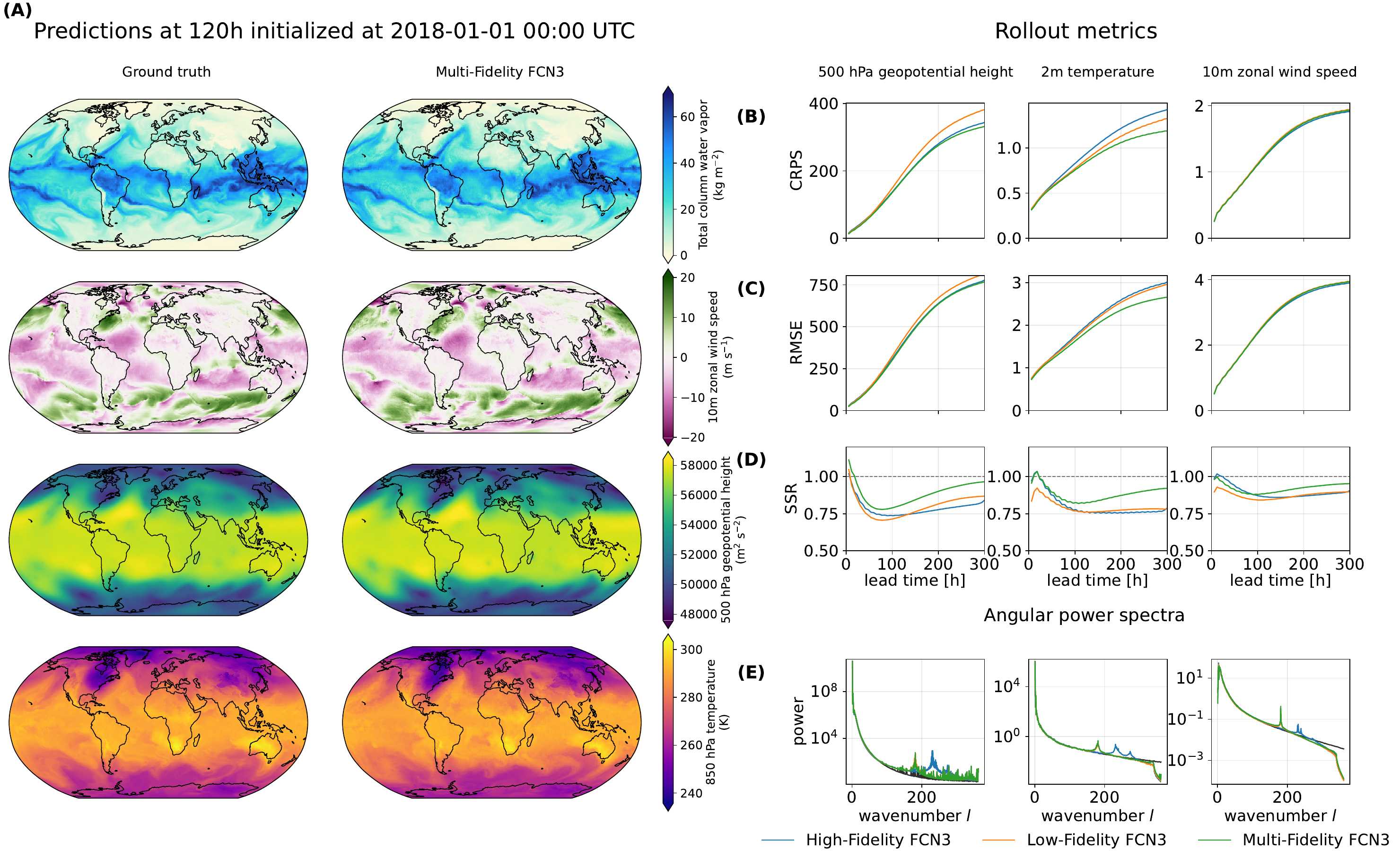}
    \caption{\textbf{(A-E) Experimental Results on Probabilistic Weather Prediction with high-fidelity, low-fidelity and mixed-fidelity FourCastNet 3 variants.} \textbf{A:} Single ensemble member predictions initialized at 2018-01-01 00:00:00 UTC and at a lead time of 120 hours. \textbf{(B-D):} Aggregated CRPS, RMSE and SSR metrics as a function of lead time for geopotential height at 500 hPa, 2m temperature and 10m zonal component of wind. \textbf{(E):} Angular power spectra for the same variables at 300 hours lead time.}
    \label{fig:fcn3_combined_panel}
\end{figure}

%%%%%

\begin{figure}[t]
\centering
\includegraphics[width=0.97\linewidth]{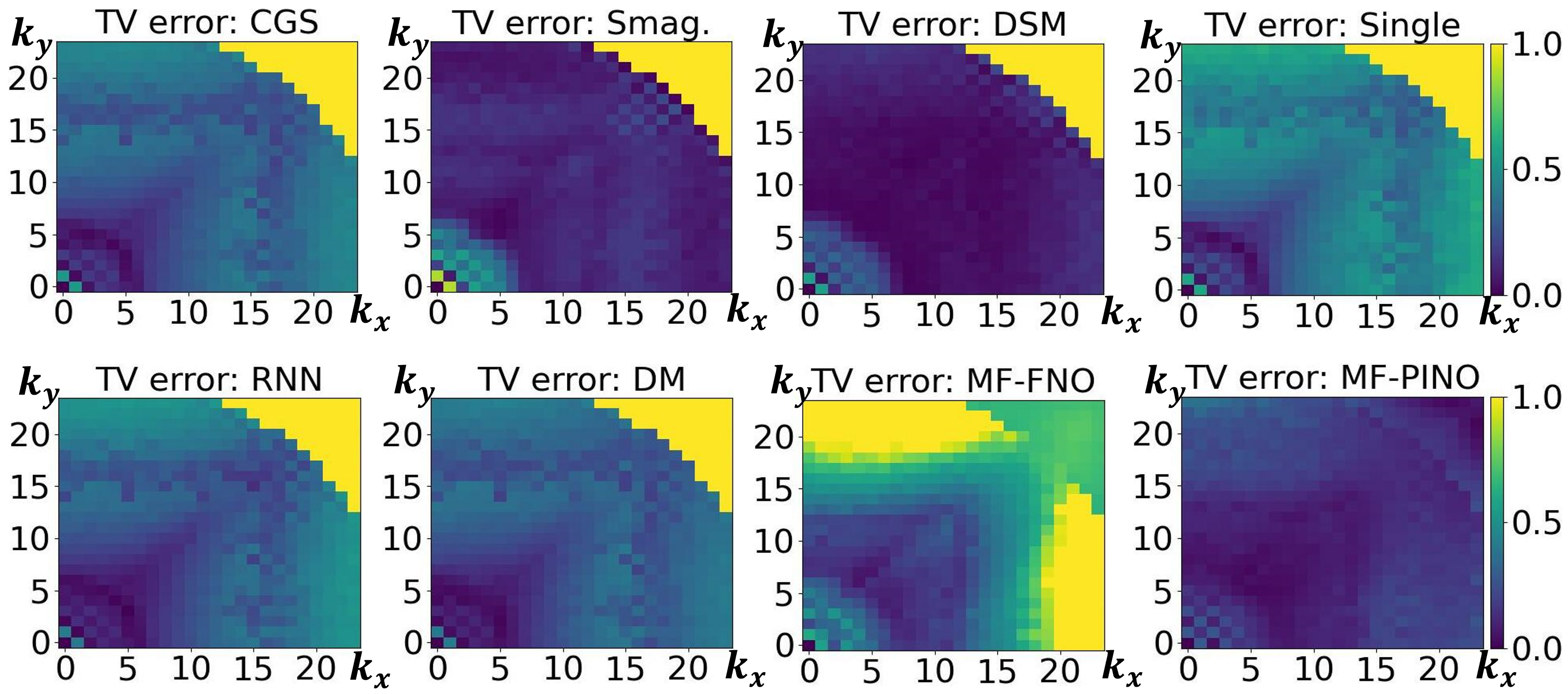}
% \caption{}
\caption{
\textbf{Total Variation (TV) error results for Navier-Stokes with $Re=1.6\times 10^4$.}
The $(k_x,k_y)$-element represents the TV error compared with ground-truth fully-resolved simulations regarding the distribution of the mode length of the component for $(k_x,k_y)$- Fourier basis $e^{i\frac{2\pi}{L}(k_xx+k_yy)}$. Total variation ranges within $[0,1]$. 
The smaller the TV error, the closer the distribution is to the ground truth.
From left to right: coarse-grid simulation without closure model (CGS), Smagorinsky model (Smag.), Dynamic Smagorinsky Model (DSM), single-state learning-based model (Single), history-aware closure model based on recurrent neural network (RNN), stochastic closure model based on diffusion model (DM), Multi-fidelity FNO (MFF), and multi-fidelity physics-informed neural operator (MF-PINO).
Our approach with MF-PINO performs the best among all methods that simulate on coarse grid systems.
}
\vspace{-1.2em}
\label{fig:stat_main}
\end{figure}

%%%%

\begin{figure}[t]
\centering

\begin{subfigure}{0.31\textwidth}
  \includegraphics[width=\linewidth]{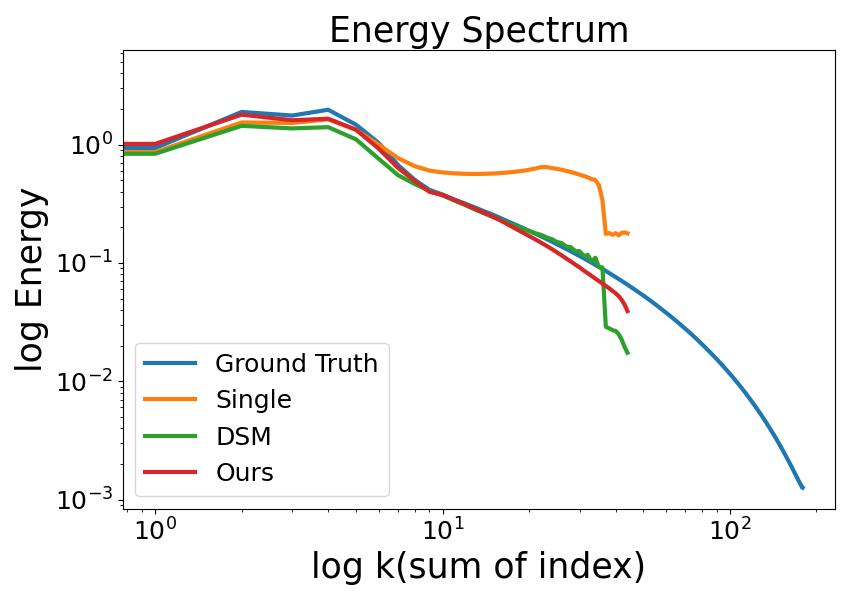}
  \caption{Energy Spectrum}%\label{fig:sub1}
\end{subfigure}\hfil % \hfil for filling the middle
\begin{subfigure}{0.31\textwidth}
  \includegraphics[width=\linewidth]{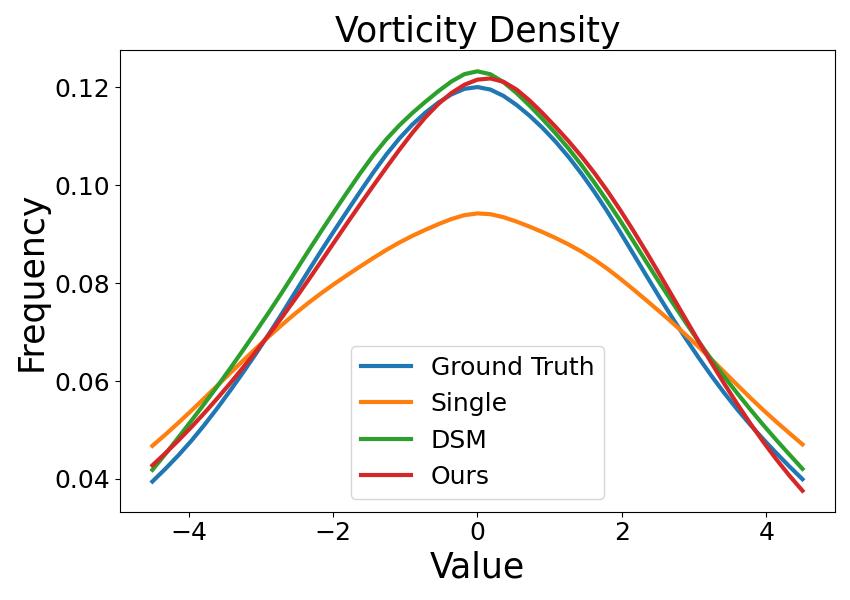}
  \caption{Vorticity Distribution }%\label{fig:sub5}
\end{subfigure}\hfil
\begin{subfigure}{0.31\textwidth}
  \includegraphics[width=\linewidth]{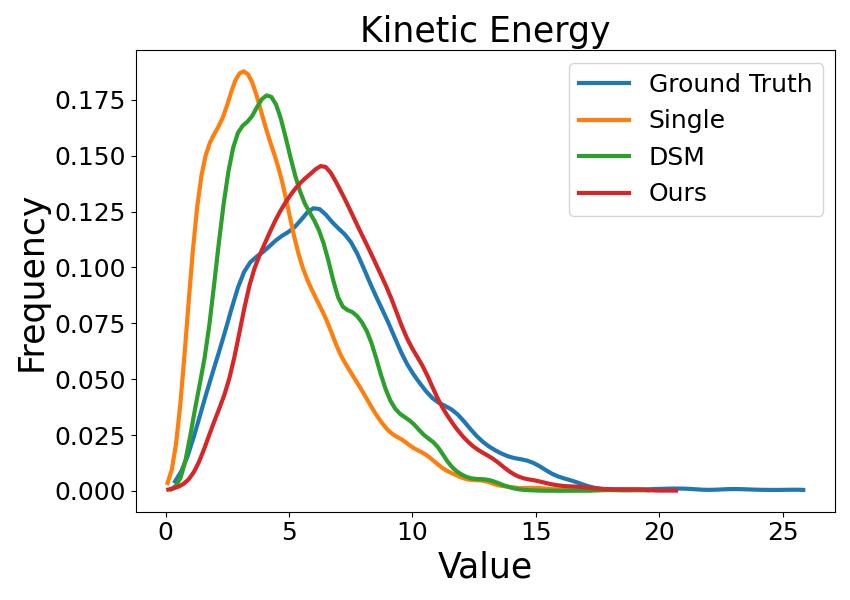}
  \caption{Kinetic Energy Distribution}%\label{fig:sub5}
\end{subfigure}
\caption{\textbf{Experiment Results of Some Statistics for NS Equation with $Re=1.6\times 10^4$.}
`Ground Truth'(blue curve) refers to fully-resolved simulation. `DSM': dynamical Smagorinsky model. `Single': learning-based single-state closure model. More results for other statistics and all methods can be found in the Appendix. Our method (red) is the closest to ground truth among all methods with coarse-grid inputs.
}
\label{fig:stat_main_hiRe}
\end{figure}

\subsubsection{Theoretical Results}
In contrast to the advantageous property of using neural operators for coarse graining, we also analyze the approach with data-driven closure models and reveal its fundamental limitations.
We first informally state our theoretical results. Recall that $\aaa$ is the nonlinear operator describing the evolution, $\mathcal{F}$ is the linear filtering operator, $S(t)$ is the semigroup, $\hhh$ is the (full) function space, and snapshots refer to $S(t)u_0$ for particular $t$ and initialization $u_0\in\hhh$. We assume the existence and uniqueness of a compact attractor and the system being ergodic.

\begin{thm}\label{thm_clos_all} (Non-uniqueness issue and Data-inefficiency of Learning-based Closures)\\
(i) In general, the target mapping of closure models $\ovu\to(\mathcal{F}\mathcal{A}-\mathcal{A}\mathcal{F})u$ is a multi-map. Consequently, the approximation error has a lower bound. The target of history-aware closures remains non-unique with a non-vanishing approximation error. One cannot obtain the best approximation of $\mu^*$ among distributions supported in the reduced space with a stochastic closure model driven by a stochastic differential equation.\\
(ii) Let $clos^*$ be the optimal closure model that provides the best approximation of $\mu^*$ among all closure models. Suppose $clos_{\theta^*}$ is the closure model achieving  the global minima of the training loss \cref{mthd1} with $N$ fully-resolved snapshots drawn i.i.d. from $\mu^*$ as training data, then 
    $\|clos_{\theta^*}-clos^*\|=\Omega(N^{-\frac 1{d_0+2}})$, where $d_0$ is the intrinsic dimension of the attractor.
\end{thm}
\begin{rmk}
    The approximation error in (i) arises from the inherent non-uniqueness of the target mapping and cannot be eliminated by increasing model capacity. In practice, this implies that the training loss admits a non-zero irreducible lower bound, which is unknown a priori and may lead to instability during optimization.  Regarding stochastic closure models, this result also applies to closures based on currently prevalent generative models such as diffusion models (DM) and flow-matching. See \Cref{apdx: instable} for empirical evidence and discussion regarding training instability.
\end{rmk}
\begin{rmk}
    $d_0$ in (ii) is typically large for PDE systems. For instance, in the Navier–Stokes equations, $d_0$ is known to grow polynomially with the Reynolds number, though the precise scaling depends on the flow configuration \cite{temam2012infinite}.  Consequently, the convergence rate of learning for data-driven closure models is inherently poor. 
    
    Notably, directly estimating statistics by averaging over fully-resolved samples achieves a convergence rate of $O(N^{-\frac 1 2})$ due to the Monte Carlo error bound. Thus, if a sufficient amount of fully-resolved data were already available, one could compute the statistics directly, eliminating the need to train a closure model.
    
% In contrast, our neural operator approach leverages limited fully-resolved data to enable efficient coarse-grained simulation. While its theoretical error bound under pure supervised learning framework is also $O(N^{-\frac 1 2})$, it is empirically more data-efficient than direct statistical estimation, and can further reduce data requirements when combined with physics-informed training.
\end{rmk}
\begin{rmk}
As can be seen from the explicit form of $clos^*$, the optimal closure model varies across different equation configurations, e.g. boundary shape and PDE coefficient
 (like Reynolds number in fluid dynamics). Consequently, one should not expect the learning-based closure
 models trained on one system to generalize well to another, unless the corresponding invariant measures $\mu^*$ remain similar.
\end{rmk}

We illustrate the non-uniqueness issue of learning closure models stated in \Cref{thm_clos_all}(i) in Fig.~\ref{fig:ant}(C). Due to dimension reduction from filtering, multiple fine-grid functions correspond to the same coarse-grid representative, but yield different filtered vector fields $\mathcal{F}(\mathcal{A}u)$. Minimizing the loss thus leads the model to predict an average vector field, which is not necessarily physically meaningful. Unlike typical inverse problems or regression tasks where averaging is reasonable, here predictions must respect the manifold structure of physical trajectories, and hence, learning closure models is ill-posed.

The under-determined nature of the closure term remains unsolved in history-aware models or with posteriori training. We remark here that there have been some theoretical results stating that historical information in the reduced space suffices to recover the underlying true trajectory \cite{levine2022framework}, but they are all derived in ODE systems and
highly rely on the finite-dimensionality of ODE systems.

%As an implication of the result for stochastic closures, incorporating randomness in the closure model might be redundant, as the parameters governing stochasticity tend to diminish after training. 

Stochasticity in closure models fails to improve performance because the theoretically optimal closure model is deterministic (\cref{eq:ideal_clos}). The ideal closure term is a conditional expectation given by the average of all possible high-resolution states. While stochasticity might still offer training advantages, such as optimization or regularization during training, it is not essential for expressiveness, as the optimal model itself is not random.
%it is a flawed approach because the optimal model itself is not random. 
% Similarly,  while stochasticity could offer benefits such as improved optimization or regularization during training, it does not fundamentally address the underlying issues in closure modeling. In fact, as shown in \cref{eq:ideal_clos}, the optimal closure model is deterministic.  
The complete theorem and proofs of the first claim of \Cref{thm_clos_all}, as well as more interpretation of the intuitions, can be found in \Cref{apdx:random}.

Now it might seem strange that learning-based closure models still manage to achieve competitive performance, as reported  in previous works~\cite{maulik2019subgrid,guan2022stable,sanderse2024scientific}. To address this, we show that their empirical result heavily relies on the availability of large amounts of fully-resolved training data.  As is the case, the reported usages of fine-grid data in previous works involve tens of thousands of samplings from either one long fully-resolved trajectory or dozens of fully-resolved trajectories, which are prohibitively expensive (see \Cref{tab:compare} for some examples). 

%We comprehensively analyze the role of closure term and present the second claim of \Cref{thm_clos_all}, whose . We propose to characterize the evolution of distribution directly. 

We derive the optimal closure model $clos^*$, which could achieve the best estimation of long-term statistics among all coarse-grid simulation with additive closure models, by generalizing Liouville equation into measures in function spaces.
\begin{equation}\label{eq:ideal_clos}
    clos^*(v)=\mathbb{E}_{u\sim\mu^*}[\mathcal{F}\aaa u\big|\mathcal{F}u=v]-\aaa v,\ \ v\in\mathcal{F}(\hhh),
\end{equation}
where $\mu^*$ is the invariant measure in the original space $\hhh$. This finding extends the seminal work by \cite{langford1999optimal}, which attempted to derive an optimal closure model but whose formulation still relies on the instantaneous distribution of unresolved fine-scale components, which makes their closure model inherently unresolvable in the coarse-grid system. In contrast, \cref{eq:ideal_clos} can be determined solely by information from the filtered space.
We further show that existing data-driven closure modeling methods implicitly aim to approximate $clos^*$. 
This connection, though previously unrecognized, explains their empirical success. However, since $clos^*$ cannot be directly queried, training relies on estimating its output via statistical averaging over states with similar filtered representations $\ovu$—--a process that emerges from the mechanics of these methods rather than being explicitly designed. This leads to both data-inefficiency and training-instability, as discussed in \Cref{apdx: liouville,apdx: thm_liouv,apdx: instable} and demonstrated in our experiments.

\subsection{Experimental Results}\label{sec: exp}

{
\subsubsection{Probing Long-term Behavior of Chaotic Dynamics}
}\label{sec 2.4.1}
We first verify the effectiveness of using neural operator as coarse-graining scheme with two representative chaotic systems, 1D Kuramoto-Sivashinsky (KS), and 2D forced Navier-Stokes (NS, also known as Kolmogorov flow in this setting) which includes two test cases with Reynolds number ($Re$) 100 and $1.6\times10^4$.
KS equation has applications in various problems, including modeling ion-sputtered surfaces
\cite{vitral2018nano}, chemical reactions
\cite{conte2003exact}, and  plasma propagation in flames \cite{sivashinsky1982instabilities}.
NS equation is well-known for its highly chaotic behavior in describing turbulent fluid dynamics \cite{pope2001turbulent}.
% (snapshots visualized in \Cref{fig:main: dataset}).
The code is available at 
$\texttt{https://github.com/neuraloperator/pino-closure-models}$.

We evaluate long-term statistical estimates against the gold-standard ground truth from fully resolved simulations. Coarse-graining methods are grouped into three categories; see \Cref{mthd_baseline} for details. Each method is subsequently referred to by its abbreviation.
\textbf{(I)} Numerical closure-model methods: CGS, Smag., DSM. \textbf{(II)} Learning-based closure models: Single-state, history-aware (RNN), and stochastic-closure model (DM).  \textbf{(III)} Coarse graining with neural operator: multi-fidelity FNO (MF-FNO) and MF-PINO, the only difference between the two is that PINO has an additional PDE loss function while FNO uses only data loss during training.
The simulations of all coarse-graining method start from random initializations on the same coarse grid and run on GPU.

Inspired by \Cref{thm_clos_all} where we analyze through the viewpoint of probability distributions, we propose to compare the predicted invariant measure and ground truth directly using total variation (TV) distance: a metric ranging $[0,1]$. The smaller it is, the closer two distributions are. It serves as a fundamental criterion in that 
a well-estimated distribution guarantees precise estimation of \textit{any} other statistics. In contrast, evaluation solely based on certain specific statistics is incomplete, as some methods may perform well on certain statistics while failing on others \cite{gamahara2017searching}. Moreover, good performance on certain selected statistics could result from error cancellation, despite poor estimation of the underlying distribution, which undermines the method's robustness across diverse tasks. 
In practice, we first fix a basis of functions and then compute the TV distance  between marginal distributions on every basis component.
To give a more convincing comparison, we also check well-known statistics
 like energy spectrum, auto-correlation, variance, velocity and vorticity density, kinetic energy and dissipation rate. 
% See \Cref{apdx: exp_visual} for their definitions.

The error of some statistics (compared with fully-resolved simulations) and the running time of a single trajectory for $t\in[0,100]$ in $Re=100$ test case is in \Cref{tab: kf_main}.
 A cost-error (in terms of average total variation distance from ground-truth invariant measure among all basis components) summary for this experiment is presented in \Cref{fig:ant}(E).
For the $Re=1.6\times 10^4$ case, visualizations of the TV error of marginal distributions for all basis components that are captured by coarse-grid systems, are presented in \Cref{fig:stat_main}. The results for some statistics are visualized in \Cref{fig:stat_main_hiRe}. The comparison on errors and running time are in \Cref{tab: ns1w:main} and \Cref{tab:ns1w_time_main}.
Full experiment results for all three dynamics along with definitions of all statistics we considered are in \Cref{apdx: exp_visual}.

\begin{table*}[t]
\centering
\setlength{\tabcolsep}{1.99pt}
\caption{\textbf{Experiment Results for Navier-Stokes Equation ($Re=100$).} \textbf{Left:} Errors on different statistics, i.e., average total variation (`Avg. TV'), energy spectrum (`Energy'), TV error for vorticity distribution (`Vorticity'), and velocity variance (`Variance'). Percentages refer to average relative errors. Other numbers refer to TV distances (ranging $[0,1]$) between ground truth and prediction. \textbf{Right:} Comparison of the inference time (seconds) of one trajectory for $t\in[0,100]$.
Best results are marked \textbf{bold}. Abbreviation of different methods: `FRS': fully-resolved simulations; `CGS': coarse-grid simulations without closure; `Smagorinsky': classical Smagorinsky closure model; 'Single-state': learning-based singles-state closure model; `DSM': dynamical Smagorinsky closure model; `MF-FNO': multi-fidelity FNO; `MF-PINO': our primary method with multi-fidelity physics-informed neural operator. The table is organized into three sections: numerical closure models (top), learning-based closure models (middle), and neural operator approaches (bottom). Our proposed method, MF-PINO, achieves the best results across all metrics.
}\label{tab: kf_main}
\begin{minipage}{0.68\linewidth}
\centering
\begin{tabular}{lcccc}\toprule
Method & Avg. TV & Energy $(\%)$  & Vorticity & Variance ($\%$) \\\midrule
CGS (No closure) & 0.4914 & 178.4651  & 0.1512 & 253.4234 \\
Smagorinsky %\cite{smagorinsky1963general}
& 0.2423 & 52.9511  & 0.0483 & 20.1740  \\
{DSM}%\cite{smagorinsky1963general}
& 0.2803 & 74.2150  & 0.0821 & 73.6158  \\\midrule
Single-state %\cite{guan2022stable} 
& 0.5137 & 205.3709  & 0.1648 & 298.2027  \\
History (RNN) %\cite{guan2022stable} 
& 0.4938 & 181.3914  & 0.1522 & 256.7264  \\
Stochastic (DM) %\cite{guan2022stable} 
& 0.4930 & 182.4392  & 0.1535 & 263.0987  \\\midrule
{MF-FNO} %\cite{smagorinsky1963general}
& 0.2123 & 20.7055  & 0.0115 & 20.4410  \\
{MF-PINO} & \textbf{0.0726} & \textbf{5.3276 } & \textbf{0.0091} & \textbf{2.8666 } \\\bottomrule

\end{tabular}
\end{minipage}\hfill
\begin{minipage}{0.30\linewidth}
\centering
%\subcaption{Inference Time (seconds)}
\begin{tabular}{lc}\toprule
% Method & Metric1 \\\midrule
%\vskip{1.2em}
\addlinespace[2.1mm]
Fully-Resolved & 39.70 \\
CGS (No closure) & 4.50 \\
Smagorinsky & 4.81 \\
DSM & 13.67 \\\midrule
Single-state & 18.57 \\
History (RNN)  & 18.92 \\
Stochastic (DM) & 42.15 \\\midrule
{MF-FNO} & \textbf{0.32} \\
{MF-PINO} & \textbf{0.32} \\\bottomrule
\end{tabular}
\end{minipage}
\end{table*}

For the higher $Re=1.6\times 10^4$ case, our method achieves an impressive speedup rate of 330x compared to fully-resolved simulation, with the best-performing estimation of most statistics.
For the low Reynolds number ($Re=100$) case, our model also achieves the best total variation error among all coarse-grid methods, and is 124x faster than  fully-resolved simulation. Our model is even faster than coarse-grid numerical solver (CGS) without any closure model because it can perform $O(1)$ time steps instead of fine temporal grids during simulations. Learning-based closure model is 58x slower than ours since it needs to invoke a neural network model for every single time step. 
Furthermore, the results of other practical statistics suggest that TV error is a reasonable and fundamental metric.

\begin{table}[htbp]
\centering
\setlength{\tabcolsep}{1.99pt}
\caption{\textbf{Error on Different Statistics: NS equation, $Re=1.6\times 10^4$.} From left to right: Average relative error on energy spectrum, max relative error on energy spectrum, total variation distance from (ground truth) vorticity distribution, average component-wise TV distance(error), relative error on root mean square of velocity, and relative error on the variance of vorticity. Percentages refer to average relative errors. Other numbers refer to TV distances (ranging [0, 1]) between ground truth and prediction.}
\begin{tabular}{lcccccc}\toprule
Method & Avg. Eng. ($\%$) & Max Eng. ($\%$) & Vorticity & Avg. TV & Mean Vel.($\%$) & Variance ($\%$) \\\midrule
CGS (No closure) & 139.5876  & 332.6311  & 0.0326 & 0.3642 & 15.9057  & 16.1087  \\
Smagorinsky & 32.8391  & 80.4986  & 0.0351 & 0.2051 & 29.8473  & 55.3507  \\
Dynamic Smag. (DSM) & 17.3967  & 71.3727  & 0.0143 & 0.1614 & 12.3541  & 33.5422  
\\\midrule
Single-state & 186.1610  & 422.6007  & 0.0525 & 0.4527 & 19.8919  & {11.7190 } \\
History-aware (RNN) & 72.6070  & 189.4918  & 0.0291 & 0.3611 & 9.2850  & \textbf{5.2505} \\
Stochastic closure (DM) & 91.3567  & 240.6798  & 0.0312 & 0.3425 & 8.7776  & {8.8782} \\\midrule
Multi-Fidelity FNO & 50.2734  & 101.6366  & 0.0237 & 0.5592 & 2.8406  & 27.2564\\  
{MF-PINO} & \textbf{13.9455 } & \textbf{34.6842 } & \textbf{0.0109} & \textbf{0.1401} & \textbf{0.2827 } & 18.7052  \\\bottomrule

\end{tabular}
\label{tab: ns1w:main}
\end{table}

%%%%%

\begin{table}[htbp]
\centering
\caption{\textbf{Inference Time for Navier-Stokes with $Re=1.6\times 10^4$}.
The comparison is based on the time cost (seconds) for the simulation of one trajectory with $t\in[0,100]$. 
 Abbreviation of different methods: `FRS': fully-resolved simulations; `CGS': coarse-grid simulations without closure; `Smag.': classical Smagorinsky closure model; `DSM': dynamical Smagorinsky closure model; `Single-state': learning-based singles-state closure model; `RNN': history-aware closure model with RNN; `DM': stochastic closure model with diffusion model; `MF-FNO': multi-fidelity FNO; `MF-PINO': multi-fidelity physics-informed neural operator.
}
\begin{tabular}{|l|c|c|c|c|c|c|c|c|c|}
\hline
Method & FRS & CGS & Smag. &DSM& Single-state & RNN & DM & {MF-FNO} & {MF-PINO} \\ \hline
 Time [s] & 525.74 & 14.29 & 29.27&71.35 & 96.78 & 54.23 &320.96 & \textbf{1.59} & \textbf{1.59} \\ \hline
\end{tabular}
\label{tab:ns1w_time_main}
\end{table}
%%%%%%%%%%%%%%%%

From the results, we see that even though using a minimal number of fully-resolved training data, neural operator manages to estimate long-term statistics accurately and efficiently, much better than all the baselines.
As for the learning-based methods following the previous closure modeling scheme, it cannot make use of the cheap CGS data. 
As suggested by our theoretical results, it performs poorly when restricted to a realistic usage of fully-resolved data, due to the large gap between the empirical measure of the training data and the desired invariant measure. The results are much worse than that in the original papers, where thousands of data points or hundreds of trajectories are used for training. Notably, aligning with the results in \cite{dong2024data}, diffusion model-based closure model (DM) is even slower than FRS,  contradicting the primary purpose of coarse graining — improving computational efficiency.
Smagorinsky model outperforms DSM in $Re=100$ case since DSM is derived based on scaling relations not well-satisfied for small $Re$, and that we have chose the best-performing parameter for Smagorinsky model. 
MF-FNO is able to leverage the abundant cheap coarse-grid simulation (CGS) data and be fine tuned on limited fully-resolved data, but the lack of PDE loss limits its generalization capabilities, which MF-PINO is able to overcome. %In contrast, MF-PINO leverages physics-informed loss on fine-grid inputs to overcome the shortage of fully-resolved data.
% In addition to all these results, we conduct an \textbf{ablation study} to demonstrate the effect of data-loss pertaining, as is in \Cref{apdx:abl}.

% Nevertheless, with reasonably more data, neural operators trained purely through fitting data remain a promising approach. This is supported by our theoretical results \Cref{thm: pino} and \ref{thm_clos_all}. We further conduct an empirical comparison of three approaches relying on fully-resolved data: neural operator, learning-based closure models, and direct computation of statistics based on the dataset. See \Cref{apdx-ablation}.

\subsubsection{Probabilistic Machine-Learning Weather Prediction}
Having demonstrated the benefits of multi-fidelity learning with neural operator across a diverse set of chaotic dynamics, we further investigate whether these advantages translate to a real-world forecasting problem at operational scale. Weather prediction provides a particularly relevant test case, as modern machine-learning forecasting systems demand substantial computational resources and large quantities of high-resolution training data. This setting therefore offers a direct opportunity to assess whether multi-fidelity training can reduce computational and data requirements while preserving predictive performance.

To this end, we consider FourCastNet3, a frontier machine-learning probabilistic weather prediction model based on the spherical neural operator paradigm~\cite{bonev2023spherical,bonev2025fourcastnet}. Consistent with the multi-fidelity framework developed throughout this work, we train both a high-fidelity model and a multi-fidelity variant of FourCastNet 3 for comparison. The high-fidelity model is trained exclusively on high-resolution weather data, whereas the multi-fidelity model first learns from a lower-resolution representation before being adapted to the target high-resolution regime.

More specifically, the high-fidelity variant is trained as an ensemble of two members on 40 years of hourly ERA5 reanalysis data at a spatial resolution of $0.25^\circ$ ($721 \times 1440$) and a lead time of 6 hours. To obtain a multi-fidelity variant a low-fidelity model is first pretrained on the same 40-year period at half the spatial resolution, $0.5^\circ$ ($361 \times 720$) with the same lead time of 6 hours. The multi-fidelity version is obtained by subsequently finetuning the low-fidelity model using only 5 years of high-resolution data. This setup mirrors the practical scenario in which lower-fidelity data are substantially cheaper to generate and store than their high-fidelity counterparts.

The high-fidelity model required approximately 89 hours of training on 128 NVIDIA H100 GPUs with a model-parallelism degree of four. In contrast, the multi-fidelity approach completed low-resolution pretraining in under 29 hours and high-resolution fine-tuning in under 3 hours using the same hardware configuration and number of optimization steps. Overall, multi-fidelity training reduced total training time by approximately $2.78\times$ relative to the high-fidelity training regime. Additional improvements are likely achievable, as the dataset was down-sampled on the fly and removing this I/O bottleneck could further accelerate training.

Despite this substantial reduction in computational cost and reliance on high-resolution data, the multi-fidelity model achieves forecasting performance comparable to that of the fully high-fidelity model. \Cref{fig:fcn3_combined_panel} depicts a summary of the numerical results. The continuous ranked probability score (CRPS), a standard metric for ensemble forecasts, is reported for 50-member FourCastNet 3 ensemble predictions up to a lead time of 300 hours, aggregated over 12 hourly initial conditions over the whole validation year of 2018. We observe that the multi-fidelity variant, although trained for a fraction of the computational cost and on lower-fidelity data, performs comparably to the high-fidelity model, either matching or out-performing it on most variables. The plots also depict a significant improvement over the low-fidelity base model that is the result of the pretraining for the multi-fidelity model. A similar trend is observed for the root mean square error (RMSE) of the ensemble-mean prediction. Moreover, the ensemble's calibration also measures the spread-skill ratio (SSR), a common metric to quantify if the ensemble is calibrated. A perfectly calibrated forecast is expected to have a SSR of 1, whereas lower or higher values indicate under- or overdispersive forecasts, respectively. Our results suggest that the multi-fidelity training generates the best-calibrated forecasts. This indicates that the multi-resolution training enables the model to better distribute uncertainty across spatial scales, thus leading to a better-calibrated forecast. Together with the results in \Cref{sec 2.4.1}, these findings suggest that multi-fidelity training not only reduces computational cost but also helps mitigate the spectral bias in operator learning, consistent with previous studies \cite{liu2024mitigating}.

Finally, the last panel depicts angular power spectra for both models at a lead time of 300 hours. We observe 
that the predictions accurately capture the spectrum of the ground truth solution. Moreover, for the multi-fidelity solution, we observe better agreement for high wavenumbers over the low-fidelity model. For all models, we observe spikes at the Nyquist frequency of the pre-training resolution. We hypothesize that these arise due to aliasing effects of the gradients in training, as these are present for all three models.
Comprehensive results that depict the model's performance on a more complete set of channels can be found in \Cref{apdx_fcn3}.

The results demonstrate that both methods faithfully capture the energy spectra of the predicted fields, even though the multi-fidelity model was mostly trained on lower-fidelity data at a significantly lower computational cost.

\section{Discussions}
Through the lens of measure flow in function spaces, we propose a theoretical framework to comprehensively study the coarse-graining ansatz. Our results reveal that the prevailing closure modeling approach, i.e., explicitly adding an error-correcting term (closure) to the coarse-grid dynamics, is in general unsuitable for simulating chaotic systems. Through this framework, we provide a unified interpretation for  the practical limitations of existing closure models. 
Specifically, the mapping that closure models try to approximate is in general non-unique. Data-driven closure models will learn to predict the average of all possible outputs that is not physically valid. Consequently, learning-based closure models cannot accurately approximate the long-term statistics, no matter how large or expressive the closure model is.
We also prove that other attempts leveraging history information and stochasticity cannot overcome this limitation.

{
% Notably, we demonstrate that the prerequisites for such empirical successes are, in general, unattainable, highlighting the inherent challenges in achieving a reliable closure model.
% As a remedy, building on our theoretical insight, we advocate an alternative coarse-graining ansatz (learning framework) that incorporates nonlinear interaction between information across different scales. This framework not only provides a theoretically guaranteed optimal estimation of long-term statistics with minimal FRS data requirements but also demonstrates superior performance in experimental evaluations.

% We formally prove in \Cref{thm: random} that for generic problems, 
% the mapping that closure models attempt to approximate in a reduced space (coarse grid) is non-unique, i.e., there are multiple potential outputs for a given input (\Cref{fig:ant}(C)). Hence, the standard approach to learning closure models
% % under such non-uniqueness 
% results in the average of all possible outputs, and that cannot accurately approximate the long-term statistics, no matter how large or expressive the closure model is.
% We also prove that other attempts leveraging history information and randomness cannot overcome this limitation.
}

To be more precise, we derive the  optimal closure model under coarse graining (\cref{eq:ideal_clos})
% (which extends the existing theoretical groundwork~\cite{langford1999optimal}) 
and show that even for sufficiently expressive models, mitigating the non-uniqueness issue of closure models requires prior methods to rely heavily on the closeness between the empirical measure of training data and the limit distribution of the original dynamics, which necessitates extensive fully-resolved simulations for training  due to the slow convergence of empirical measures in high-dimensional spaces.

Closure modeling has been the mainstream framework for coarse graining for decades and remains widely adopted across various fields. However, we provide a deeper understanding of the limitations underlying its previous empirical successes. Classical numerical closure models are often derived under idealized assumptions. For example, the derivation of the dynamical Smagorinsky model~\cite{lilly1992proposed} heavily relies on the scaling relationship between the energy spectrum and the wavenumber, which is theoretically valid only for infinite Reynolds numbers and high-wavenumber regimes. As a result, our experiments reveal that while this closure model performs well in high-Reynolds-number scenarios, it exhibits significant errors in capturing statistics associated with low-wavenumber dynamics and in low-Reynolds-number settings. For learning-based closure models, we show that the amount of high-fidelity data required to train an effective model exceeds the amount needed to directly estimate the target statistics, undermining the very purpose of training such a model.

Beyond the closure modeling framework, our results indicate a more promising coarse-graining direction that incorporates nonlinear interaction between information from different scales using neural operators, supported by both theoretical guarantee and experimental evidence for efficient and accurate estimations of long-term statistics.

\paragraph{Scope and Limitations}
Our study focuses on deterministic dynamical systems that admit a unique compact attractor and satisfy ergodicity. These conditions are essential for defining a unique invariant measure and the corresponding long-term statistics, as not all chaotic systems admit well-defined statistical quantities over long timescales. To ensure consistency with prior studies, we conduct experiments using the Kuramoto–Sivashinsky and Navier–Stokes equations~\cite{li2022learning}. However, for systems lacking these properties, the applicability of our framework remains uncertain.

The theoretical convergence guarantee of neural operators (\Cref{thm: pino}) further relies on the hyperbolicity of the attractor. While this assumption is satisfied in many practical systems, its validity in more general dynamical settings remains to be investigated. {The analysis and results in \Cref{thm_clos_all} are only for closure models and other coarse graining schemes do not fall under this.}

Although we employ a physics-informed training strategy to reduce data requirements, we note that optimization remains a challenge under limited fully-resolved data regime, particularly in complex systems such as Navier–Stokes with $Re=2 \times 10^5$. This is consistent with prior findings~\cite{li2021physics}. In highly turbulent settings, directly minimizing the PDE residual may become intractable due to instability in optimization. Incorporating additional physics-based constraints, e.g., energy dissipation rates~\cite{li2022learning} or Sobolev norms, may help regularize training and serve as intermediate objectives. Designing an effective curriculum for such losses remains an open direction for future research. We also note that MF-PINO is not universally superior. In systems where robust empirical scaling laws or well-characterized structures exist, handcrafted closures based on reasonable approximations may yield sufficiently accurate results and remain a practical choice in applications. However, in settings where such heuristics fail or are unavailable, neural operators provide a principled and data-efficient alternative for coarse graining. {Additionally, while numerical closures can be applied broadly without much modification, machine-learning-based closures, as well as our method, require retraining. In this study, we focus on the conceptual formulation and did not discuss generalization across different systems, which is technically feasible and represents a concrete direction for future work.}

% \paragraph{The Reliance on Data}
% Theoretically, as long as the optimization for physics-informed loss converges, MF-PINO requires no high-fidelity data. In practice, however, optimizing physics-informed loss from scratch is known to be challenging. To mitigate this, we use a combination of coarse-grid and fully-resolved data for pre-training, ensuring a well-initialized optimization process. 

% This naturally introduces a tradeoff between the reliance on high-fidelity data and the stability of training with physics-informed loss, particularly in terms of hyperparameter tuning and optimization robustness. As a general guideline for practitioners, utilize as much high-fidelity data as available, as it typically enhances performance. 

\paragraph{Future Work}
This work takes an initial step toward exploring a new coarse-graining ansatz with neural operators that incorporates nonlinear interactions across different scales. The results presented above highlight the effectiveness of our approach, paving the way for future research to address more challenging problems using advanced operator learning methods and model architectures.
For instance, though we did not conduct experiments on generalization across different dynamics, e.g. varying Reynolds numbers or domain geometries, this can be done by concatenating dynamics coefficients and parameterized geometries with the input functions and applying modifications of FNO \cite{li2024geometry,rahman2024pretraining,li2024scale}. {Moreover, our method is generally applicable in settings where the governing equation is unknown. In this case, the physics-informed step in our method is achieved by training the model to match known observables.}
Additionally, we have not systematically evaluated the amount of fully-resolved data required by our method. Data from fully-resolved simulations are primarily used to alleviate the challenges in optimizing the physics-informed loss, a process that is highly sensitive to the choice of optimizer and model initialization, making it difficult to quantify their impact. This warrants further investigation.

While our method is theoretically guaranteed to yield superior estimations of long-term statistics, this guarantee assumes that the model achieves reasonably low physics-informed (PDE) loss. Optimization remains a key bottleneck for all physics-informed learning methods, and our approach would benefit significantly from advances in optimization techniques for physics-informed losses, such as recent works using advanced optimizers~\cite{wang2025simulating}.

We have not considered stochastic dynamical systems in this work, such as stochastic PDEs. While ergodicity and nontrivial invariant measures are more prevalent in stochastic systems, coarse graining in this setting remains relatively underexplored and represents another avenue for future development.

%%% 0312 try_add
\begin{refcontext}[sorting = none]
\printbibliography
\end{refcontext}
\label{references}

\begin{refsection}

\newpage

\section{Methods}
In this section, we provide the details of our method, theoretical results and experiment settings.

\subsection{Physics-Informed Neural Operator with Multi-fidelity Pre-training}

We first introduce operator learning, followed by methods of physics-informed and multi-resolution multi-fidelity training, and finally provide the theoretical guarantee on estimating statistics of our new method.

\paragraph{Operator Learning}
The goal of operator learning \cite{kovachki2023neural,lu2021learning} is to approximate mappings between function spaces rather than vector spaces. 
One of the representatives is Fourier Neural Operator (FNO)~\cite{li2020fourier}, 
whose architecture
can be described as: 
% which has demonstrated considerable success in addressing PDE-based problems with substantial speed improvements [cite 1, 5, 9, 15, 18]. The structure of a neural operator architecture can be described as follows:
% $    \mathcal{G}_{FNO}:=\mathcal{Q} \circ ({W}_L + \mathcal{K}_L) \circ \cdots \circ \sigma ({W}_1 + \mathcal{K}_1) \circ \mathcal{P},$
\begin{equation}\label{fno_map}
    \mathcal{G}_{FNO}:=\mathcal{Q} \circ ({W}_L + \mathcal{K}_L) \circ \cdots \circ \sigma ({W}_1 + \mathcal{K}_1) \circ \mathcal{P},
\end{equation}
where $\mathcal{P}$ and $\mathcal{Q}$ are pointwise lifting and projection operators. The intermediate layers consist of an activation function $\sigma$, pointwise operators $W_{\ell}$ and integral kernel operators $K_{\ell}:\ u\to \mathscr{F}^{-1}(R_{\ell} \cdot \mathscr{F}(u))$, 
where $R_{\ell}$ are weighted matrices and $\mathscr{F}$ denotes Fourier transform. See \cref{fig:archi} for illustration.

\begin{figure}[t]
\centering
\includegraphics[width=1\linewidth]{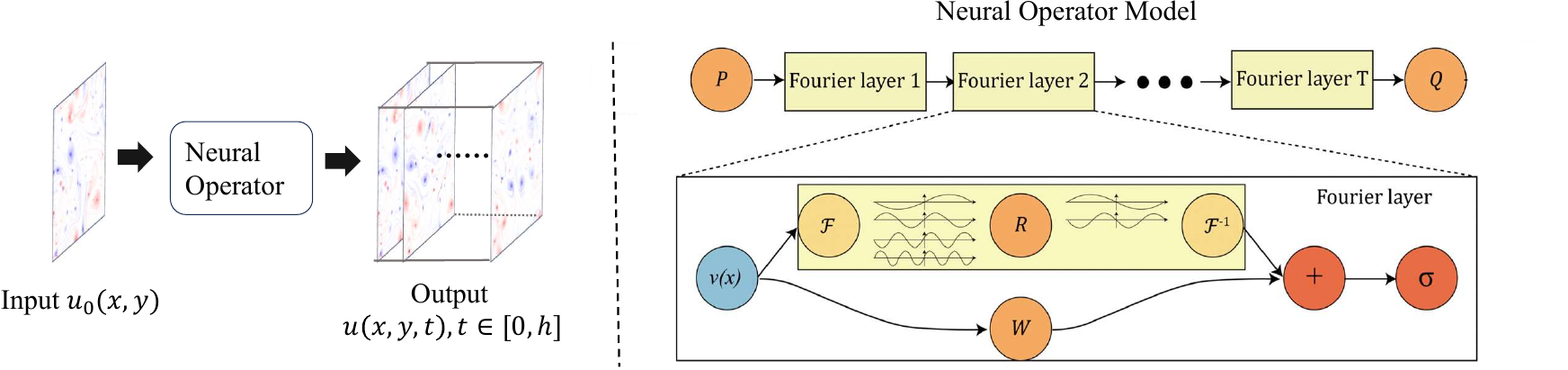}
\caption{
\textbf{Illustration of the model architecture.} 
\textbf{Left:} The neural operator model takes the initial condition (a spatial function) as input, e.g. $u_0(x,y)$ in 2D case. Its output is the trajectory $u(x,y,t)$ in time interval $t\in [0,h]$, where $h$ is a model parameter. \textbf{Right:} Architecture of neural operator. The input spatial function is first lifted into a spatial-temporal function by replicating $u_0$ in the temporal dimension and equipping it with temporal position embedding. Then it is transformed with point-wise operator $P$ which is usually a multi-layer perceptron (MLP) that increases the feature dimension. The Fourier layers consist of spectral convolution implemented based on Fourier transform $\mathcal{F}$, and point-wise operation. In the figure, $W$ denote the weight matrix in MLP and $\sigma$ denotes nonlinear activation functions. After several Fourier layers, the final output is obtained with a point-wise operator $Q$.
% The initial condition $u_0$ is repeated over the temporal dimension and combined with temporal position embeddings to form the input. This input is fed into a neural operator. The architecture of neural operator is introduced in \cref{fno_map}. At each step, the output's final time frame is used as the next initial condition to roll out the simulation recursively.
}
\label{fig:archi}
% \vspace{-2em}
\end{figure}

With FNO, we learn the mapping $u\to \{S(t)u\}_{t\in[0,h]}$, where $h$ is a model parameter. 
We note that most of the existing works apply neural operator to learn the semigroup mapping for a single time frame, e.g. $u\to S(h) u$. Modeling this single-time semigroup requires less parameter and computes than the mapping we target in this paper. The main reason why we try to approximate the mapping over a time interval is that we need to compute time derivative of the output function for physics-informed training, as is shown in \cref{pde-loss}.
Neural operator has two major advantages: (1) Resolution-Invariance: The model supports input from different resolutions (grid sizes), and inputs are all viewed as discretization of an underlying function. Consequently, when we feed a coarse-grid initial state to the well-trained model and roll out to generate a CGS trajectory, there exists an fully-resolved trajectory such that the CGS trajectory we obtain is its filtering. This CGS trajectory matches the optimal coarse-grid dynamics discussed in \Cref{sec: liouville}.
(2) Faster Convergence: The burning time $T_{burn}$ is the moment when a trajectory approaches the attractor close enough. For previous methods, after the learning-based closure models are trained, they are merged into a coarse-grid solver and evolve iteratively with relatively small time steps. In operator learning where $h$ is usually of $O(1)$ magnitude, the simulation arrives at $T_{burn}$ more quickly.

\paragraph{Physics-Informed Neural Operator}
We need to overcome the lack of fully-resolved training data in realistic situations. Note that the PDE(\ref{eq:general-pde}) contains all the information of the dynamical system. We adopt physics-informed methodologies \cite{karniadakis2021physics} to remove the reliance on data. To be specific, the operator model $\mathcal{G}_\theta$ is trained by minimizing the physics-informed loss function:
% $
%     J_{pde}(\theta;\mathfrak{D})=\frac{1}{|\mathfrak{D}|}\sum_{i\in\mathfrak{D}} \|(\partial_t-\mathcal{A})\mathcal{G}_\theta u_{0i}(x)\|_{L^2(\Omega\times[0,h])},$
\begin{equation}\label{pde-loss}
    J_{pde}(\theta;\mathfrak{D})=\frac{1}{|\mathfrak{D}|}\sum_{i\in\mathfrak{D}} \|(\partial_t-\mathcal{A})\mathcal{G}_\theta u_{0i}(x)\|_{L^2(\Omega\times[0,h])},
\end{equation}
where the initial values $u_{0i}$ in the loss function could be any fine-grid functions and do not have to come from fully-resolved simulations. $\Omega$ is the spatial domain of these functions. More background of this method is in \cite{li2021physics}.

We would like to remark here regarding the mapping being learned. Most existing works employ neural operators to approximate the single-step semigroup map $S(h)$, whereas our approach targets the full semigroup over a time interval. As shown in \cite{kovachki2023neural}, both formulations achieve comparable predictive accuracy in practice. Both models also exhibit sufficient expressiveness and training stability under supervised learning.
While the single-step formulation typically involves fewer parameters and offers faster inference, the interval-based model provides outputs at intermediate time steps, enabling the estimation of time derivatives. This capability is crucial for computing physics-informed losses, following the approach introduced in PINO \cite{li2021physics}.

 \paragraph{Practical Algorithm with Multi-Resolution Multi-Fidelity Pre-training} In practice, the optimization of physics-informed loss is hard \cite{rathore2024challenges} and might encounter some abnormal functions with small loss but large errors \cite{wang20222}. To tackle these challenges, we pre-train the model via supervised learning with a data loss function to achieve a good initialization of the model parameters for $J_{pde}$ optimization:
 % $    J_{data}(\theta;\mathfrak{D})=\frac{1}{|\mathfrak{D}|}\sum_{i\in\mathfrak{D}}
 %    \|\mathcal{G}_\theta u_i-S([0,h])u_i\|.$
\begin{equation}
    J_{data}(\theta;\mathfrak{D})=\frac{1}{|\mathfrak{D}|}\sum_{i\in\mathfrak{D}}
    \|\mathcal{G}_\theta u_i-S([0,h])u_i\|,
\end{equation}
where the norm $\|\cdot\|$ can be either $L^2$ norm (corresponding to MSE loss) or other norms, e.g. $H^1$.
To enhance the limited fully-resolved training data available, we pre-train with 
% $J_{data}$ using 
plenty of CGS data first and then add fully-resolved data into the loss function. After that, we gradually decrease the weight of CGS data loss in the loss function since CGS data is potentially incorrect. After warming up with data loss, we further train our model with physics-informed loss. The formalized algorithm and its implementation details can be found in \Cref{alg:mainn}.

For our experimetn results, we refer to neural operator trained with this approach as MF-PINO and the one only with multi-resolution multi-fidelity data pretraining as MF-FNO.

Although our training incorporates multi-fidelity and multi-resolution data, the test-time inference is strictly performed on the same coarse grid as other baseline methods. This ensures a fair and consistent comparison across all approaches.

\begin{algorithm}[t]
    \caption{Multi-Resolution Multi-Fidelity Physics-Informed Operator Learning}
    \label{alg:mainn}
    \hspace*{0.02in} \textbf{Input:} Neural operator $\mathcal{G}_{\theta}$; training data set $\mathfrak{D}_{c}$(snapshots from coarse-grid simulations), $\mathfrak{D}_{f}$(snapshots from fully-resolved simulations),\ $\mathfrak{D}_{p}$(randomly sampled).\\
    % \hspace*{0.02in} \textbf{Output:} Learned operator  $\mathcal{G}_{\theta}$\\
    \hspace*{0.02in} \textbf{Hyper-parameters:}
    Training iterations $N_i(i=1,2,3)$. Weights combining two loss $\lambda_i(t)\\(i=1,2)$, which decay as $t$ increases. Parameters regarding the optimizer.

     \hspace*{0.02in} \textbf{Loss Function:}
     $J_{data}(\theta;\mathfrak{D})=\frac{1}{|\mathfrak{D}|}\sum_{i\in\mathfrak{D}}
    \|\mathcal{G}_\theta u_i-S([0,h])u_i\|,$\\
     \hspace*{1.1in}
    $J_{pde}(\theta;\mathfrak{D})=\frac{1}{|\mathfrak{D}|}\sum_{i\in\mathfrak{D}} \|(\partial_t-\mathcal{A})\mathcal{G}_\theta u_{0i}(x)\|_{L^2(\Omega\times[0,h])}.$
    % Number of total training iterations $M$; number of iterations and step size of inner loop $K,\eta$; weight for combining the two loss term $\lambda$
    \begin{algorithmic}[1]
    \For{$t=1,\cdots, N_1$}
    \State Minimize $J_{data}(\theta;\mathfrak{D}_c)$
    \EndFor
    \For{$t=1,\cdots, N_2$}
    \State Minimize $\lambda_1(t) J_{data}(\theta;\mathfrak{D}_c)+J_{data}(\theta;\mathfrak{D}_f)$
    \EndFor
    \For{$t=1,\cdots, N_3$}
   \State Minimize $\lambda_2(t) J_{data}(\theta;\mathfrak{D}_f)+J_{pde}(\theta;\mathfrak{D}_p)$
   \EndFor
    \State \Return $\mathcal{G}_{\theta}$
    \end{algorithmic}
\end{algorithm}

\subsection{Perspective through Liouville Flow in Function Space}\label{sec: liouville}
% We have demonstrated that existing learning methods target a non-unique mapping, resulting in an average of all possible outputs, which can be undesirable. Despite this, these methods still manage to achieve competitive performance. In this section, we show that their empirical result heavily relies on the availability of a large amount of FRS training data. This dependency is a significant limitation, as FRS data are typically scarce. If a sufficient amount of FRS data were already available for training, we could directly compute the statistics using the data, eliminating the need for training a closure model or running coarse-grid simulations.

Recall that our task is to get good and efficient estimations of statistics related to the invariant measure, instead of tracking any single trajectory, which is impossible in chaotic dynamics given the unavoidable discretization error in coarse-grid simulations.
This motivates us to investigate the evolution of the distribution (or measure) of functions to determine whether it converges to $\mu^*$.

 In finite-dimensional dynamical systems (ODEs), the evolution of distribution is governed by the Liouville equation. This observation inspires us to generalize the Liouville equation into function space and conduct our study therein. Rigorous definitions of related notions and detailed proofs for all claims made in this section can be found in \Cref{apdx: liouville}.

\paragraph{Functional Liouville Flow:} 
If we expand functions onto an orthonormal basis, $u = \sum_i z_i \psi_i$,\label{zi}
a PDE system (of $u$) can be viewed as an infinite-dimensional ODE (of $\rvz$). In this way, we yield the functional version of the Liouville equation describing how the probability density of $u$ evolves. Under this framework, 
we only need to check the stationary Liouville equation to obtain the limit invariant distribution of a dynamical system and compare it with $\mu^*$. 

In the coarse-grid setting, we similarly derive the evolution of the density of $\ovu$ and yield the \textit{optimal dynamics of }$v\in\mathcal{F}(\mathcal{H})$ (different from CGS in \cref{eq:les-pde}), $v$ is exactly the same as $\ovu$ here),
\begin{equation}\label{eq:main_ideal}
    \partial_t v=\mathbb{E}_{u\sim \mu_t}[\mathcal{F}\mathcal{A}u | \mathcal{F}u=v],
    % \vspace{-0.5em}
\end{equation}
where $\mu_t$ is the distribution of $u\in\mathcal{H}$ following the original dynamics at time $t$ and this expectation is conditioned on the samplings of $u$ satisfying $\mathcal{F}u=v$. Here we arrive at the same result in~\cite{langford1999optimal}.
Unfortunately, $\mu_t$ depends on $t$ and the initial distribution of $u\in\mathcal{H}$, which is underdetermined and will suffer from non-unique issues if restricted to a coarse-grid system, similar to what is discussed in the previous section.
In practice, one can only fix one particular $\hat{\mu}$, a distribution in $\mathcal{H}$, and assign the dynamics in reduced space as $\partial_tv=\mathbb{E}_{u\sim\hat{\mu}}[\mathcal{F}\mathcal{A}u|\mathcal{F}u=v]$.
% [f_1(\textbf{v},\textbf{v}^\perp)\big|\textbf{v}]$. 
Checking the resulting Liouville equation,
% and its limit distribution
we show that $\mu^*$ is the correct choice for $\hat{\mu}$ to guarantee convergence towards $P_{\#}\rho^*$, the optimal approximation of $\mu^*$
in $\mathcal{F}(\mathcal{H})$, where $P$ is the orthogonal projection towards $\mathcal{F}({\hhh})$, identical to $\mathcal{F}$ in most cases. Back to the learning methods listed in \Cref{sec:learn_clos}, due to the $L^2$ variational characterization of conditional expectation, 
% in the ideal case where the closure model is expressive enough, 
the underlying choice of $\hat{\mu}$ is the empirical measure $\mu_{data}$ of those training data coming from fully-resolved simulations, ideally $\mu^*$. Consequently, one has to use numerous fully-resolved training data due to the slow convergence of empirical measures 
in high dimensions. In particular, if $\mu_{data}$ is already close to $\mu^*$, $\mathbb{E}_{u\sim\mu_{data}}\mathcal{O}(u)$, the estimation of a statistic characterized by a measurement functional $\mathcal{O}$ based merely on the training data, will naturally provide a good approximation of the ground truth $\mathbb{E}_{u\sim \mu^*}\mathcal{O}(u)$, eliminating further need to train a closure model.
% for high-dimensional distribution.
% $\rho^*$.

% As is the case, these learning methods often rely on a large amount of fine-grid data coming from one long FRS trajectory or multiple FRS trajectories which are expensive.
% % by expensive fine-grid simulations[cite]. 
% Furthermore, most methods still rely on a coarse-grid solver that iteratively evolves with relatively small time steps, and some methods require that the coarse simulation starts from a downsampled version of high-fidelity data close to the attractor. These aspects hinder the further application of these methods.

\subsection{Closure Modeling Methods}\label{mthd_baseline}
In this section, we provide more information on the closure modeling methods, both numerical and learning-based. Implementation details for results shown in the experiment section can be found in \Cref{apdx: implement}.

\subsubsection{Numerical Closure Models}
\paragraph{Vanilla coarse-grid simulation (\textbf{CGS})}

We refer to coarse-grid simulation without any additional closure term as `CGS' in our experiment results. It serves as a blank baseline. 

 We use spectral method for spatial discretization in the coarse-grid simulations. Due to the truncation of high-frequency modes from coarse discretization and the use of dealiasing, numerical dissipation is introduced. Therefore, this setting can also be viewed as an implicit closure method \cite{grinstein2007implicit}.

\paragraph{Eddy-Viscosity Closure Models}
(\textbf{`Smag.'}) Smagorinsky model \cite{smagorinsky1963general} is the most classical and popular closure model applied in computational fluid dynamics. We compare with Smagorinsky model for NS and its counterpart eddy-viscosity model for KS \cite{matharu2020optimal}. We have selected the best-performing parameter in these models. Dynamical Smagorinsky model \textbf{(DSM)} \cite{lilly1992proposed} is a strong variant of Smagorinsky model.

\subsubsection{Learning-based Closure Models}

\paragraph{Single-state closure models}

Besides the single-state closure models leanred by minimizing the apriori loss function as discussed in former sections (\ref{mthd1}), there are also extensions of the learning framework above. Some works proposed to add a posterior loss
% $J_{post}(\theta;\mathfrak{D})=\frac {1} {|\mathfrak{D}|}\sum\limits_{i\in\mathfrak{D}}\|v_i(\cdot,\Delta t;\theta)-\mathcal{F}(S(\Delta t)u_i)\|^2,$
into the training object \cite{sirignano2020dpm,list2022learned},
% $J_{post}(\theta;\mathfrak{D})=J_{ap}(\theta)+\frac {1} {|\mathfrak{D}|}\sum\limits_{i\in\mathfrak{D}}\|v_i(\cdot,\Delta t;\theta)-\mathcal{F}(S(\Delta t)u_i)\|^2,$
\begin{equation}\label{mthd_post}
    J_{post}(\theta;\mathfrak{D})=J_{ap}(\theta)+\frac {1} {|\mathfrak{D}|}\sum\limits_{i\in\mathfrak{D}}\|v_i(\cdot,\Delta t;\theta)-\mathcal{F}(S(\Delta t)u_i)\|^2,
\end{equation}
where $v_i$ comes from evolving (\ref{eq:les-pde}) with $\mathcal{F}u_i$ initialization for a time period $\Delta t$. Clearly, this modification still suffers from the issue resulting from the multimap as illustrated in \Cref{thm_clos_all}.

To leverage the up-to-date machine learning toolkits, we replace the convolution neural network (CNN) models in original papers \cite{guan2022stable} with a transformer-based model. We refer to this approach as `Single' in our empirical evaluations.

\paragraph{History-aware closure models}
We have shown its idea in \Cref{sec:learn_clos} and follows the work by Ma et al.\cite{ma2018model} to apply a recurrent neural network (RNN) as model backbone. We labeled this method as \textbf{RNN-closure} in our experiments.

\paragraph{Stochastic closure models}
One direction is to replace the deterministic closure model with a stochastic one \cite{boral2023neural,lu2017data,dong2024data}. More detailed discussion on their general form and theoretical results can be found in \Cref{apdx:thm:sde}.

In our experiments, we follow the setting of \cite{dong2024data} and apply diffusion models \cite{song2019generative,song2020denoising,ho2020denoising} to sample the closure term.
We label this method as \textbf{DM-closure}.

\paragraph{Other learning-based closure models}

Besides the three mainstream learning-based closure models we loosely classified above, there are other approaches that resort to an interactive use of fine-grid simulators~\cite{sorensen2024non,oommen2023rethinking} and leveraging online learning algorithms~\cite{frezat2023gradient, sirignano2023dynamic, duraisamy2021perspectives, sirignano2020dpm}. However, calling and auto-differentiating along fully-resolved simulations makes the training of these approaches prohibitively expensive.

\subsection{Estimating Long-term Statistics with Coarse-grid Simulations}

\paragraph{Kuramoto–Sivashinsky (KS) Equation}
We consider the one-dimensional KS equation
% for $u(x,t)$,
% \begin{equation}
% \partial_t u+u\partial_x u+\partial_{xx}u +\nu \partial_{xxxx}u=0,\ \ \quad (x,t) \in[0,6\pi]\times\mathbb{R}_{+},    
% \end{equation}
with periodic boundary conditions. The positive viscosity coefficient in the equation reflects the traceability of this equation. The smaller $\nu$ is, the more chaotic the system is. We study the case for $\nu=0.01$.

For our model, we choose $h=0.1$ (the time interval the neural operator model learns to predict). More details on the choice of $h$ can be found in \Cref{apdx-choose-h}. The total amount of fully-resolved training data is 105 snapshots coming from 3 trajectories.
When we complete the training, the $L^2$ relative error of our model on the test set is $12\%$. 
To make a fair comparison, other learning-based methods are restricted to the same amount of training data. This setting will be the same for NS.

More details on coarse-grid and fully-resolved simulations datasets, implementations of our method, and evaluation are presented in \Cref{KS_detail,apdx: implement}.

% The burning time of this system is around $50$.

% and complete the training when $L^2$ relative error on the test set is smaller than 

\paragraph{Navier-Stokes (NS) Equation}
We consider two-dimensional Kolmogorov flow (a form of the Navier-Stokes equations with force) for a viscous incompressible fluid (fluid field) $\rvu(x,y,t)\in\mathbb{R}^2$,
\begin{equation}
    \partial_t \rvu=-(\rvu\cdot\nabla)\rvu-\nabla p+\nu\Delta\rvu+(\sin(4y),0)^T,\quad \nabla\cdot \rvu=0,\quad (x,y,t)\in [0,2\pi]^2\times \mathbb{R}_{+},
\end{equation}
with periodic boundary conditions. 
Reynolds number ($Re$) is an important concept in fluid dynamics. It is defined as $\frac{\overline{u}l}{\nu}$, where $\nu$ is the viscosity appearing in the equation, $\overline{u}$ is the root mean square of velocity scales $|\rvu|$, and $l$ is the length scale of the domain. 
% The function $p$ is a known pressure. The positive coefficient $Re$ is Reynolds number. 
The larger $Re$ is, the more chaotic the system is. We consider a simple case with $Re=100$ and a challenging one with $Re=1.6\times 10^4$. 
% More details on NS equation and visualizations of the high $Re$ fluid field is in \Cref{apdx:NS basic}.
For our model, we choose $h=1$ for the simple case and 0.5 for the hard case. The total amounts of fully-resolved training data are 110 and 384 snapshots coming from 1 trajectory, respectively. When we complete the training, the $L^2$ relative error on the test set is $19\%$ for $Re=100$ case and $14\%$ for $Re=1.6\times 10^4$ experiment.

More details on coarse-grid and fully-resolved simulations datasets, implementations of our method, and evaluation are presented in \Cref{apdx:NS basic,apdx: implement}.

To estimate the long-term statistics for all of the dynamics considered in this work, we draw random coarse-grid initial conditions from a Gaussian random field and simulate multiple trajectories. Then we average over time and these trajectories, following the definition of the invariant measure. For neural operators, since the output of the model corresponds to the function value on a time interval with length $h$, the rollout is performed autoregressively, where at each step, the last temporal slice of the model's output tensor is used as the input for the subsequent forward pass,
\begin{equation}u_0,\ 
    \mathcal{G}_\theta u_0,\ \mathcal{G}_\theta(\mathcal{G}_\theta (u_0)\big|_{t=h}),\ \mathcal{G}_\theta\big(\mathcal{G}_\theta(\mathcal{G}_\theta (u_0)\big|_{t=h})\big)\big|_{t=h},...
\end{equation}
where $h$ is a model parameter related to the mapping the neural operator is approximating (\Cref{thm: pino}).
This aligns with previous works that only learns the short-term evolution and conducts rollouts to study long-term behavior \cite{watt2023ace,watt2025ace2}.
{To determine the length of the rollouts, we track the statistics such as the energy spectrum and terminate the emulation once these statistics converge to a stable pattern. Note that the attractors are dynamical equilibria of the chaotic systems we studied, meaning the functions keep evolving while the statistics converge. In the settings where the characteristic timescale of the PDE is known, we can terminate early rather than simulating to $t\to\infty$, which in practice is done by running simulations up to $T,\ 2T, \ 3T, ...$ for a sufficiently large $T$ until convergence is observed.}
See \Cref{apdx: data} for more details.

We want to clarify here that due to the system reduction nature of coarse-grid simulation, it cannot precisely recover the invariant measure $\mu^*$ even in ideal cases. Instead, as will be discussed in \Cref{sec: liouville}, the best approximation of $\mu^*$ one can obtain with coarse-grid simulations is $P_{\#}\mu^*$ (\Cref{apdx_prop_c5}), where $P$ is the orthogonal projection towards the filtered space $\mathcal{F}(\hhh)$. The gap between $\mu^*$ and $P_{\#}\mu^*$ heavily depends on the reduced space, determined by
the filter $\mathcal{F}$. Consequently, it is possible that estimations of certain statistics based on coarse-grid simulations can never get reasonably close to the ground truth.

% it is hard to quantify the effect of data due to its significant dependence on the optimizer and model initialization, which needs further analysis.

% We propose a new theoretical framework, functional Liouville flow, to analyze this problem. We rigorously demonstrate the inherent shortcomings of existing learning methods. Also inspired by our theoretical result, we leverage physics-informed neural operators to give an efficient and provably accurate estimation of long-term statistics with very limited fine-resolution data usage during training. As evaluated in the experiments, our method has the potential to address the challenging tasks regarding chaotic systems arising in various physical sciences.
% The implication of 
% this 
% work is not restricted to the specific task of estimating long-term statistics.
%  This work exhibits the benefit of going beyond the finite grid system and understanding problems through a function space viewpoint. Functional Liouville flow would be useful in investigating image generation tasks, by viewing images as functions represented on pixels.

\subsection*{Code Availability}

The code is publicly available at 
$\texttt{https://github.com/neuraloperator/pino-closure-models}$.

\subsection*{Data Availability}
We provide the source code for generating training dataset as in\\ 
$\texttt{https://github.com/neuraloperator/pino-closure-models}$.

\subsection*{Acknowledgements}
A. A. is supported in part by Bren endowed chair, ONR (MURI grant N00014-18-12624), and by the AI2050 senior fellow program at Schmidt Sciences. J. B. acknowledges support from the Wally Baer and Jeri Weiss Postdoctoral Fellowship. Z. L. is supported in part by Nvidia Fellowship. C. W. hopes to thank Andrew Stuart, Ricardo Baptista for helpful discussions. The authors also hope to thank Romit Maulik for helpful feedback.

\subsection*{Author Contributions}
C. W. implemented the main codebase and developed the theoretical results. J. B. and Z. L. gave suggestions on the experiments. D. Z. and H. J. B. contributed domain-specific background and context. C. W., Z. L., and J. W. implemented baseline comparisons and ablation studies. J. B. verified the theoretical results. B. B. and T. K. contributed to the weather forecasting aspect. C. W. led the manuscript writing. A. A., J. B., and B. B. contributed to editing, proofreading, and manuscript refinement.

\subsection*{Competing Interests}
The authors declare no competing interests.

% \newpage
% \bibliography{main}
% \bibliographystyle{unsrt}

% 0312 add
\printbibliography[heading=subbibintoc, title={Method References}]
\end{refsection}

% 0312 del
% \begin{refcontext}[sorting = none]
% \printbibliography
% \end{refcontext}
% \label{references}

%%%%%%%%%%%%%%%%%%%%%%%%%%%%%%%%%%%%%%%%%%%%%%%%%%%%%%%%%%%%

 \clearpage

%%%%%%%%%%%%%%%%%%%%%%%%%%%%%%%%%%%%%%%%%%%%%%%%%%%%%%%%%%%%

\appendix

%ncomm cite try
\begin{refsection}
% \newrefcontext{apdx}
\newrefcontext[sorting=none]{apdx}

\section*{Supplementary Information}
In this appendix, we will first
provide detailed proofs of our theoretical results (\ref{app:apdxA}-\ref{apdx: thm_pino}), and then present implementation details and all experiment results (\ref{apdx: data}-\ref{apdx: instable}). The structure of the appendix is as follows.
% In this appendix, we will first
% list related works (\ref{apdx: related work}), then
% provide detailed proofs of our theoretical results (\ref{app:apdxA}-\ref{apdx: thm_pino}), and finally present implementation details and all experiment results (\ref{apdx: data}-\ref{apdx: exp_visual}). The structure of the appendix is as follows.
\begin{itemize}
    % \item \Cref{apdx: related work} includes related works of this paper.
    \item \Cref{app:apdxA} provides a list of notations, along with an introduction of important background conceptions, preliminary results, and basic assumptions in this paper.
    % introduces the notation, auxiliary results, and basic assumptions in the paper.
    \item \Cref{apdx: liouville} first formally introduces functional Liouville flow, and then presents a detailed version of \Cref{sec: liouville}.
    \item \Cref{apdx:random} provides the formal statement and the proof of the first claim in \Cref{thm_clos_all}.
    \item \Cref{apdx: thm_liouv} provides the formal version and the proof of the second claim in \Cref{thm_clos_all}.
    \item \Cref{apdx: thm_pino} provides the proof of \Cref{thm: pino}.
    \item \Cref{apdx: data} contains information about the dataset in the experiments and a visualization of the Navier-Stokes dataset.
    \item \Cref{apdx: implement} provides the implementation details for our method and baseline methods.
    \item \Cref{apdx: exp_visual} first formally introduces the statistics we consider, followed by the full experiment results (table and plots) and ablation studies.
    \item \Cref{apdx: instable} discusses the training-instability of learning-based closure models.
\end{itemize}

% \section{Related Works}
% \label{apdx: related work}

\section{Notations, Auxiliary Results, and Basic Assumptions}
\label{app:apdxA}
In this section, we first summarize the notations in this paper, then 
% list classical textbooks that are good references for understanding our theoretical results, 
review and define some of the important concepts,
and finally state the basic assumptions in this work. We would encourage the readers to always check this section when they have any confusion regarding the proof.

\subsection{Notations}\label{apdxA: notation}

\begin{longtable}{cp{12cm}}
\caption{List of Notations}\\
\textbf{Notation} & \textbf{Description} \\
\endfirsthead
\textbf{Notation} & \textbf{Description} \\
\endhead
% Example entries
\(\mu\)  & Distributions. \\[1em]
\(F_\#\) & Push-forward of a mapping \(F\). If \(y=F(x)\) and the distribution of \(x\) is \(\mu_x\), then the distribution of \(y\) is \(F_{\#}\mu_x\). \\[1em]
\(F^\#\) & Pull-back of a mapping \(F\). If \(y=F(x)\) and the distribution of \(y\) is \(\mu_y\), then the distribution of \(x\) is \(F^{\#}\mu_y\). \\[1em]
\(\mathcal{H}\)  & (Original) Function Space, equipped with norm $\|\cdot\|_\hhh$, see \cref{eq:general-pde}. We will omit the subscript $\hhh$ when it is clear from the context.\\[1em]
\(\mathcal{A}\)  & The operator of the dynamics, see \cref{eq:general-pde}. \\[1em]
\(S(t)\)  & Semigroup induced by \cref{eq:general-pde}. \\[1em]
\(\mu^*,\ \rho^*\)  & \(\mu^*\) is the invariant measure of the dynamic system. \(\rho^*\) is its probability density. Check \Cref{apdx-b-1} to see how the concept of density is established.\\[1em]
\(u\)  & Functions in \(\mathcal{H}\). \\[1em]
\(\mathcal{F}\)  & Filter. \(\mathcal{F}:\mathcal{H}\to\mathcal{H}\). The image space is finite-rank and denoted as \(\mathcal{F}(\mathcal{H})\). \\[1em]
\(D\)  & The set of grids in fully-resolved simulations. The number of grids is \(|D|\). \\[1em]
\(D'\)  & The set of grids in coarse-grid simulations. The number of grids is \(|D'|\). \\[1em]
\(\mathscr{P}(\Omega)\)  & The set of all probability distributions supported on a set \(\Omega\). \\[1em]
\(\rho_1\)  & The marginal distribution on resolved course-grid system of a distribution $\rho$, see \Cref{apdx-b-2}. \\[1em]
\(\mathcal{W}_{\mathcal{H}}\)  & The Wasserstein distance for measures in \(\mathcal{H}\), with \(\|\cdot\|_\mathcal{H}\) being the cost function. \\[1em]
\(\langle\cdot,\cdot\rangle\)  & The pair of linear functionals and elements in a Banach space \(\mathscr{X}\). Specifically, for \(x\in\mathscr{X},\ f\in\mathscr{X}^*\), its dual space, \(\langle f,x\rangle:=f(x)\). Thanks to Riesz Representation Theorem, we will also use this notation for inner products in Hilbert space. \\[1em]
\(\otimes\)  & For a Hilbert space \(\mathcal{H}\), \(u\in\mathcal{H}\), \(v\in\mathcal{H}\), \(u\otimes v\) is defined as the linear operator \(w\to \langle v,w\rangle u,\ w\in\mathcal{H}\). \\[1em]
\(\oplus\)  & $u\oplus v:=(u,v)$. \\[1em]
\(\mathcal{T}\) & The isometric isomorphism between separable Hilbert space $\hhh$ and $\ell^2$, see \Cref{apdx-b-1}. \\[1em]
\(C_0\)  & Continuous function space, equipped with $L^{\infty}$ norm. \\[1em]
\(C_c^\infty\)  & Smooth and compactly-supported functions. \\[1em]
\(d\mathcal{G}(u,v)\)  & The Gateaux derivative of operator $\mathcal{G}$ at $u$ in the direction of $v$.\\[1em]
\(\aleph_0\)  &Aleph-zero, countably infinite. \\[1em]
\([n]\)  &$\{1,2,...n\}$ \\[1em]
\(D_T\)  &The set of time-grids.\\[1em]
\(M_{D_T}(u_0)\)  &The set of functions that are indistinguishable from $u_0$ merely based on values on spatiotemporal grid $D'\times D_T$, see \Cref{ass:unbound}.\\[1em]
\(I_x,\ I\)  & (Spatial) Grid-measurement operators, see \Cref{def Ix}.\\[1em]
\(\fff_{D_T},\ \fff\)  & Spatiotemporal grid-measurement operators, see \Cref{thm:hist:finite}.\\[1em]
\(B_0(r)\)  & The ball centered around origin with radius $r$. \\[1em]
\(spt\)  & Supporting set of a function. \\[1em]
\(w.r.t.\)  & With regard to \\[1em]
\(wlog\)  & Without loss of generality \\[1em]
% Add more symbols as required
\end{longtable}

% \subsection{Classical Materials}
% We would like to refer the readers to these classical textbooks if they have any confusion about the theoretical results in the following sections. 

% \begin{itemize}
%     \item On dynamical systems: \cite{brin2002introduction,wen2016differentiable}.
%     \item On ergodic theorem: \cite{cornfeld2012ergodic}.
%     \item On partial differential equations (PDE): \cite{evans2022partial}. \cite{gilbarg1977elliptic} provides more background on elliptic equations, and \cite{temam2012infinite} is a good reference for evolution equations (infinite dimensional dynamical systems).
%     \item On stochastic differential equations (SDE): \cite{mao2007stochastic}.
%     \item On measure theory (and probability) in function space:\cite{ledoux2013probability,vakhania2012probability}.
%     \item On optimal transport (OT): \cite{villani2009optimal,santambrogio2015optimal}.
%     \item On Riemannian geometry and manifolds: \cite{gallot1990riemannian,lee2012smooth}.
% \end{itemize}

\subsection{Optimal Transport in Function Space}
\label{apdx:wsst}
Given a Banach space $\mathscr{X}$ and two distributions $\mu_1,\mu_2\in\mathscr{P}(\mathscr{X})$, we want to measure the closeness of these two distributions.

Recall that in finite-dimensional $\mathscr{X}$, the (Monge formulation of) $c-Wasserstein$ distance is defined as 
\begin{equation}
\mathcal{W}_c(\mu_1,\mu_2):=\inf_{T} \int_{\mathscr{X}}c(x,Tx)\mu_1(dx),\quad s.t.\  T_{\#}\mu_1=\mu_2,
\end{equation}
where $T$ is a measurable mapping from $\mathscr{X}$ to $\mathscr{X}$, and $c=c(x,y)$ is non-negative bi-variate function known as cost function.

We could naturally generalize this concept into measures in arbitrary Banach space $\mathscr{X}$ and define the Wasserstein distance correspondingly. In particular, we use the metric in $\mathscr{X}$ as cost function and define
\begin{equation}
    \mathcal{W}_{\mathscr{X}}(\mu_1,\mu_2):=\inf_{T}\int_\mathscr{X}\|Tx-x\|\mu_1(dx),\quad s.t.\ T_{\#}\mu_1=\mu_2.
\end{equation}

For more backgrounds and rigorous definitions of concepts appeared above, \cite{villani2009optimal,santambrogio2015optimal} are standard references for optimal transport, and \cite{ledoux2013probability,vakhania2012probability} are good references for measure theory (and probability) in function space.

\subsection{Dynamical Systems}
We sketch some of the highly-relevant concepts and theorems from dynamical systems as follows.
We would like to refer the readers to classical textbooks on dynamical systems \cite{brin2002introduction,wen2016differentiable,temam2012infinite} and ergodic theorems \cite{cornfeld2012ergodic} for detailed proofs of lemmas stated in this subsection.
\subsubsection{Discrete Time System}
Let $X$ be a compact metric space and $f: X\to X$ be a continuous map. The generic form of discrete-time dynamical system is written as $x_{n+1}=f(x_n),\ n\in\mathbb{N}$ or ($n\in\mathbb{Z}$ if $f$ is a homeomorphism).

\begin{defn}
    $\Lambda\subsetneqq X$ is an {invariant set} of $f$ if $f(\Lambda)=\Lambda$.
\end{defn}
In the following discussion, we further assume $X$ to be a $C^{\infty}$ Riemannian manifold without boundary.
\begin{defn}
    An invariant set $\Lambda\subsetneqq X$ of $f$ is {hyperbolic} if for each $x\in\Lambda$,  the tangent space $T_xX$ splits into a direct sum 
    \begin{equation}
        T_xX=E^s(x)\oplus E^u(x),
    \end{equation}
    invariant in the sense that 
    \begin{equation}
        Tf(E^s(x))=E^s(f(x)),\ Tf(E^u(x))=E^u(f(x)),
    \end{equation}
    such that, for some constant $C\geq 1$ and $\lambda\in(0,1)$, the following uniform estimates hold:
    \begin{align}
        |Tf^n(v)|\leq C\lambda^n|v|,\ \forall x\in\Lambda,\ v\in E^s(x),\ n\geq 0,\\
       |Tf^n(v)|\geq \frac 1 C\lambda^{-n}|v|,\ \forall x\in\Lambda,\ v\in E^u(x),\ n\geq 0.
    \end{align}
\end{defn}
\begin{lemma}[Shadowing Lemma] \label{lemma:shadow}
Let $\Lambda\subset X$ be a hyperbolic set of $f$. For any $\epsilon>0$, there is $\eta_0,\eta_1>0$ such that for any $\{x_n\}_{n\in\mathbb{N}}$ satisfying (i) $d(x_n,\Lambda)<\eta_0$; (ii) $|x_{n+1}-f(x_n)|<\eta_1$ for all $n$, there exists $y\in X$ such that $|x_n-f^{(n)}y|<\epsilon,\ \forall n\in\nnn.$
\end{lemma}

\subsubsection{Continuous Time System}
Recall that the dynamical system we consider is
\begin{equation}
\label{eq:general-pde_apdx}
    \begin{cases}
    \partial_t u(x,t)=\mathcal{A}u(x,t)\\
    u(x,0)=u_0(x),\ u_0\in \mathcal{H},
    \end{cases}
\end{equation}
where $u_0$ is the initial value and $\mathcal{H}$ is a function space containing functions of interests. We will occasionally refer to $\mathcal{A}u$ as \textit{vector field governing the dynamics}. This dynamics induced a semigroup $\{S(t)\}_{t\geq 0}$ defined as the mapping from $u_0$ to $u(\cdot, t)$.

\begin{rmk}
    Boundary conditions and norm of this function space is contained in the notion of $\hhh$. For instance, $\hhh$ might be $H_0^1$ for problems with Dirichlet boundaries, and $L^2_{per}$ for periodic boundaries. Since we are carrying out our study in general dynamics, we will directly adopt $\|\cdot\|_\hhh$ for the norm of $\hhh$.
\end{rmk}

\begin{defn}
    Given a measure $\mu\in\mathscr{P}(\mathcal{H})$, the system is mixing if for any measurable set $A,\ B\subset\mathcal{H}$, $\lim\limits_{t\to\infty}\mu(A\cap S(t)(B))=\mu(A)\mu(B).$
\end{defn}
\begin{defn}
    A measure $\mu\in\mathscr{P}(\hhh)$ is said to be an invariant measure of this system if $S(t)_{\#}\mu=\mu$ for all $t>0.$
\end{defn}

\begin{lemma}\label{mix ergo}
    If a system is mixing, then it is ergodic.
\end{lemma}

For ergodic systems, there is an invariant measure independent of the initial condition, defined as
\begin{equation}\label{inv_mes_app}
    \mu^*:=\lim_{T\to\infty}\frac 1 T\int_{t=0}^T\delta_{S(t)u}dt,\  u\in\mathcal{H},\ a.e.
\end{equation}
where $\delta$ is the Dirac measure. 

\begin{defn}
    The semigroup is said to be uniformly compact for t large, if for every bounded set $B\subset\mathcal{H}$ there exists $t_0$ which may depend on $B$ such that $\cup_{t\geq t_0}S(t)B$ is relatively compact in $\mathcal{H}$.
\end{defn}

\begin{lemma}\label{lemma: compact}
    If S(t) is uniformly compact, then there is a compact attractor in this system.
\end{lemma}

\subsection{Assumptions}
\label{app:ass}
Without loss of generality, we carry out our discussion in the regime where the function space $\mathcal{H}$ is a separable Hilbert space to make the proof more readable and concise. We also make the following technical assumptions.
\begin{ass}\label{ass:h to c}
    $\mathcal{H}$ can be compactly embedded into $C_0$.
\end{ass}
\begin{ass}\label{ass:mixing and compact}
    The system \cref{eq:general-pde_apdx} is mixing. The semigroup is uniformly compact and Gateaux-differentiable in $\hhh$.
\end{ass}
\begin{ass}
    The attractor and invariant measure are unique.
\end{ass}
\begin{ass}
    The attractor is hyperbolic w.r.t $S(t)$ for any $t>0$.
\end{ass}
We remark that these assumptions are either proved or supported by experimental evidence in many real scenarios~\cite{temam2001navier,wang2014least,kuznetsov2011dynamical,kuznetsov2012hyperbolic}.

For brevity, we will ignore the difference between fully-resolved simulation (FRS) and the exact solution to \cref{eq:general-pde_apdx} in the following discussions.

\section{Formal Introduction of Functional Liouville Flow}
\label{apdx: liouville}
In this section, we first formally introduce functional Liouville flow in \Cref{apdx-b-1}. Based on this theoretical framework, we will reformulate the task of estimating long-term statistics of dynamical systems with coarse-grid simulations in \Cref{apdx-b-2}. Finally, we will provide in \Cref{apdx-b-3} a detailed version of the discussion in \Cref{sec: liouville}. The assumptions and notation conventions are summarized in \Cref{app:apdxA}.

\subsection{Framework of Functional Liouville Flow for Studying Invariant Measure}
\label{apdx-b-1}

Recall that instead of tracking any single trajectory, we are primarily interested in the limit distribution (invariant measure) of functions following the dynamics. As motivated in main text, we hope to leverage a similar idea to Liouville / Fokker-Planck equation in ODE systems to investigate how the distributions evolve along time. We will first build the connection between a general PDE dynamics (\ref{eq:general-pde_apdx}) and a finite-dimensional particle systems (ODE).

\paragraph{Functions as vectors:} As a corollary of Hahn-Banach theorem, we could always construct a set of orthonormal basis $\{\psi_i\}_i$ of $\mathcal{H}$ such that the filtered space $\mathcal{F}(\mathcal{H})=\mathrm{span}\{\psi_1,...\psi_n\}$, where we usually have $n=|D'|$. For any function $u\in\mathcal{H}$, there exists a unique decomposition $u(x)=\sum_{i=1}^{\infty}z_i\psi_i(x)$ with $z_i=\langle u,\psi_i\rangle$. This canonically induces an isometric isomorphism:
\begin{equation}\label{eqn: basis}
    \mathcal{T}:\ \mathcal{H}\to \ell^2,\quad u\mapsto \rvz=(z_1,z_2,...).
\end{equation}
By this means, we can rewrite the original PDE(\ref{eq:general-pde_apdx}) into an ODE in $\ell^2$, denoted by 
\begin{equation}\label{z:ode}
    \frac {d\rvz}{dt}=f(\rvz),\ \mathrm{where}\ f(\rvz)\in\ell^2,\ f(\rvz)_i=\langle\psi_i, \mathcal{A}\circ \mathcal{T}^{-1}\rvz\rangle,\ i\in\nnn,
\end{equation}
where $\langle,\rangle$ is the inner product in $\mathcal{H}$. 

\begin{exmp}
    For Kuramoto–Sivashinsky Equation
    \begin{equation}
\partial_t u+u\partial_x u+\partial_{xx}u +\partial_{xxxx}u=0,\ \ \quad (x,t) \in[0,2\pi]\times\mathbb{R}_{+},    
\end{equation}
if we choose $\{\psi_k\}$ as the Fourier basis $\{e^{ikx}\}_{k\in\mathbb{Z}}$, then $z_k$ is the coefficient of $k$-th Fourier mode and the ODE for $\rvz$ is (component-wise),
\begin{equation}
\frac {dz_k}{dt}=(-k^4+k^2)z_k-\frac{ik}2\sum_{j+l=k}z_jz_l.
\end{equation}

One could further make $z_k$ real numbers by choosing $\sin,\ \cos$ basis.
\end{exmp}

\paragraph{Functional Liouville flow:}
Recall that in ODE system $\frac{dx}{dt}=f(x),\ x\in\mathbb{R}^d$, if the initial state $x_0$ follows the distribution $\mu_0$ whose probability density is $\rho_0(x)$, then the probability density of $x(t)$, denoted by $\rho(x,t)$, satisfies the Liouville equation,
\begin{equation}\label{liou:ode}
    \partial_t\rho=-\nabla\cdot(f\rho).
\end{equation}
Now we want to generalize this result into function space. We need to address the issue that there is in general no probability density function for measures in function space. We will show that it is reasonable to carry on our study by fixing a sufficiently large $N$, investigating the truncated system of the first $N$ basis, and viewing the densities as the (weak-)limit when $N\to\infty$.
\begin{proposition}
    For any $\mu$ supported on a bounded set $B\subset\mathcal{H}$ and any $\epsilon>0$, there exsits $t_0$ and $N$ s.t. for any $u_0\sim\mu$ and any $t>t_0$, if we write $S(t)u_0$ as $\sum_{i=0}^{\infty}z_i\psi_i$, then $\|\sum_{i>N}z_i\psi_i\|<\epsilon$.
\end{proposition}
\begin{proof}
    Define $Q_m:=\sum_{i>m}\psi_i\otimes\psi_i$. Then the statement is equivalent to $\|Q_N S(t)u_0\|<\epsilon,\ $for all $u\in B, t>t_0.$

    Due to \Cref{ass:mixing and compact}, there exists $t_0$ such that $\cup_{t>t_0}S(t)B$ is relatively compact. This implies that there exists finite (denoted by $N_1$) points $\{u_i\}$ satisfying that for any $u_0\in B,\ t>t_0$, there exists $i\leq N_1$ s.t. $\|S(t)u_0-u_i\|<\frac \epsilon 3$. We define $M_i:=\min_j\{j| \|Q_ju_i\|<\frac \epsilon 5\}$. We have $M_i<\infty,\ \forall i.$
    Choosing $N$ as $\max\limits_{i\leq N_1} M_i $ completes the proof.
\end{proof}

\begin{rmk}
    We can always restrict our discussion within distributions supported on bounded set whose complement occurs with a probability smaller than machine precision.
\end{rmk}

Back to the dynamics \cref{eq:general-pde_apdx} or the equivalent ODE \cref{z:ode}, we first make a generalization.
Since the invariant measure is independent of initial condition, we know that for any distribution $\mu_0\in\mathscr{P}(\mathcal{H})$(instead of only delta distributions at $u_0$) from which we sample random initial conditions $u_0\sim\mu_0$ and evolve these functions, the long-term average $\lim\limits_{T\to\infty}\frac 1 T\int_{t=0}^T\big(S(t)_{\#}\big)\mu_0dt$ will still converge to $\mu^*$.
We will carry out our discussion in this generalized setting where the initial condition is sampled from a distribution $\mu_0$ in function space. We will denoted $\big(S(t)_{\#}\big)\mu_0$, the distribution at time $t$, as $\mu_t$, and denoted their density functions for corresponding $\rvz$ as $\rho(\cdot, t)$ (i.e., $\mu_t=\mathcal{T}^{\#}\rho(\cdot,t)$).
For brevity, we will view $u\in\mathcal{H}$ and $\rvz\in\ell^2$ as the same and not mention $\mathcal{T}^\#$ or $\mathcal{T}_{\#}$ for $\mu_t$ and $\rho(\cdot,t)$.

If we use the component-form of $f$, $f=(f_1,f_2,....)$, with each $f_i$ a mapping from $\ell^2\to\mathbb{R}$, with exactly the same argument to derive \cref{liou:ode}, we have
\begin{equation}\label{z:liou:pde}
    \partial_t\rho(\rvz,t)=-\sum_i^{\infty}\partial_{z_i}(f_i(\rvz)\rho(\rvz,t)):=-\nabla_{\rvz}\cdot(f\rho),\qquad \rho(\rvz,0)=\rho_0(\rvz).
\end{equation}
We will refer to this as functional Liouville flow, i.e., the Liouville equation in function space, and denote the R.H.S. operator $\rho\to-\nabla_{\rvz}\cdot(f\rho)$ as $\mathscr{L}\rho$.

\paragraph{Reinterpretation of Invariant Measure:}
With functional Liouville flow, we obtain a new characterization of invariant measure $\mu^*$ (whose distribution is denoted as $\rho^*$).
\begin{proposition}\label{prop: station liou}
    $\rho^*$ is the solution to stationary Liouville equation $\mathscr{L}\rho=0$.
\end{proposition}
\begin{proof}
    Denote $p(\rvz,t):=\frac 1 t \int_{s=0}^t\rho(\rvz,s)ds$, the finite-time average distribution.

Note that for any $\rvz$, 
\begin{align}
\rho(\rvz,t)=&\rho(\rvz,0)+\int_{s=0}^t\partial_t\rho(\rvz,s)ds\\
=& \int_0^t\mathscr{L}\rho(\rvz,s)ds+\rho(\rvz,0)\\
=&\mathscr{L}\int_{s=0}^t\rho(z,s)ds+\rho(\rvz,0)=\mathscr{L}(tp(\rvz,t))+\rho(\rvz,0).
\end{align}
Also, we have  
$\rho(\rvz,t)=\partial_t(\int_{s=0}^t\rho(\rvz,s)ds)=\partial_t(tp(\rvz,t))$, we conclude that 
\begin{equation}
    \partial_t(tp(\rvz,t))=t\mathscr{L}p(\rvz,t)+\rho(\rvz,0).
\end{equation}

From this, we yield
\begin{align}
    \partial_tp(\rvz,t)=\mathscr{L}p(\rvz,t)+\frac 1 t (\rho(\rvz,0)-p(\rvz,t)).
\end{align}
By definition we know $p(\rvz,t)\to\rho^*$ as $t\to\infty$, thus $\partial_tp\to0$. The term $\frac 1 t (\rho(\rvz,0)-p(\rvz,t))$ will also tend to zero as $t\to\infty$ (recall that they are probability density and thus are uniformly bounded in $L^1$). Therefore, the limit density $\rho^*$ satisfies $\mathscr{L}\rho^*=0$.
\end{proof}

\subsection{Reformulation of Estimating Long-term Statistics}
\label{apdx-b-2}
We reformulate the problem of \textit{estimating long-term statistics with coarse-grid simulation} with the help of functional Liouville flow.
Recall that $D'$ is the set of coarse grid points, and coarse-grid simulation (CGS) is equivalent to evolving functions in $\mathcal{F}(\mathcal{H})$, where $\mathcal{F}$ is the filter (see \Cref{sec:learn_clos}).

\subsubsection{Notations}
We start by defining several notations.

Define the orthonormal projection onto $\mathcal{F}(\mathcal{H})$ as ${P}=\sum\limits_{i=1}^n \psi_i\otimes\psi_i$. We remark here that in many situations, we have $P=\mathcal{F}$.

Let us decomposite $\rvz$ and $u$ into the resolved part and unresolved part,
\begin{align}
    \rvz&=\rvv\oplus\rvw,\ \rvv:=(c_1,c_2,...c_n),\ \rvw:=(c_{n+1},c_{n+2},...); \\
    u(x)&=v(x)+w(x),\ v(x):=\mathcal{T}^{-1}\rvv=Pu,\ w(x):=\mathcal{T}^{-1}\rvw=(I-P)u.
\end{align}
In particular, $w\in\mathcal{F}(\mathcal{H})^\perp$ is the unresolved part in coarse-grid simulations.
With this decomposition, we rewrite any density $\rho(\rvz)$ as a joint distribution $\rho(\rvv,\rvw)$ and define marginal distribution for $\rvv$ as $\rho_1(\rvv)$, and the conditional distribution of $\rvw$ given $\rvv$ as $\rho(\rvw|\rvv)$.
With a little abuse of notation, we will occasionally refer to the probability density as its distribution, and vice versa.

We will also divide the vector field $f$ into resolved part $f_r$ and unresolved part $f_u$, which are $(f_1,f_2,...,f_n)$ and $(f_{n+1},f_{n+2},...)$ respectively.

\subsubsection{Reformulation of Coarse-grid Simulation}
We first show that the optimal approximation of $\mu^*$ (or $\rho^*$) in the reduced space is its marginal distribution, if we construct densities with an orthonormal basis, as is in \cref{eqn: basis}.

\begin{proposition}
\label{apdx_prop_c5}
     $\rho_1^*=\argmin\limits_{\mu\in\mathscr{P}(\mathcal{F}(\mathcal{H}))}\mathcal{W}_\mathcal{H}(\mu,\mu^*)$.
\end{proposition}
\begin{proof}
    From the construction of $P$ and the definition of $\mathcal{W}_{\mathcal{H}}$, for any measurable mapping \\
    $\mathfrak{T}:\mathcal{H}\to\mathcal{F}(\mathcal{H})$,
    \begin{align}
    \int_{\mathcal{H}}\|\mathfrak{T}u-u\|\mu^*(du)\geq  \int_{\mathcal{H}}\|Pu-u\|\mu^*(du).
    \end{align}
    Thus, for any $\mu\in\mathscr{P}(\mathcal{F}(\mathcal{H}))$, 
    \begin{align}
        \mathcal{W}_{\mathcal{H}}(\mu,\mu^*)\geq \int_{\mathcal{H}}\|Pu-u\|\mu^*(du)=
        \mathcal{W}_{\mathcal{H}}(P_\#\mu^*,\mu^*).
    \end{align}
    Note that $\rho_1^*=P_\#\mu^*$, this completes the proof.
\end{proof}

This result motivates us to check the evolution of $\rho_1(\rvv,t)$, which should achieve the optimal approximation $\rho^*_1$.

Note that by definition, for any distribution $\rho\in\mathscr{P}(\mathcal{H})$, $\rho_1(\rvv)=\int \rho(\rvv,\rvw)d\rvw$.
Combine this with \cref{z:liou:pde}, we yield
\begin{align}
\partial_t\rho_1(\rvv,t)=&\int\partial_t\rho(\rvv,\rvw,t)d\rvw\\
    =&-\int \nabla_{\rvv}\cdot(f_r(\rvv,\rvw)\rho(\rvv,\rvw,t))d\rvw-\int \nabla_{\rvw}\cdot(f_u(\rvv,\rvw)\rho(\rvv,\rvw,t))d\rvw\\
    =& -\nabla_\rvv\cdot\bigg(\frac{\rho_1(\rvv,t)}{\rho_1(\rvv,t)}\int f_r(\rvv,\rvw)\rho(\rvv,\rvw,t)d\rvw\bigg)-0\\
    =&-\nabla_\rvv\cdot\bigg(\rho_1(\rvv,t)\int f_r(\rvv,\rvw)\frac{\rho(\rvv,\rvw,t)}{\int\rho(\rvv,\rvw',t)d\rvw'}d\rvw\bigg)\\
    =& -\nabla_\rvv\cdot\big(\rho_1(\rvv,t)\mathbb{E}_{\rvw\sim\rho(\rvw|\rvv;t)}[f_r(\rvv,\rvw)|\rvv]\big).\label{eq:liou_expec}
\end{align}
where we use the divergence theorem for the second term in the second line.

The corresponding ODE dynamics for this Liouville equation \cref{eq:liou_expec} is
\begin{equation}\label{ode: optimal}
    \frac{d\rvv}{dt}=\mathbb{E}_{\rvw\sim\rho(\rvw|\rvv;t)}[f_r(\rvv,\rvw)|\rvv].
\end{equation}
If we transform it back into $\mathcal{H}$ space, it becomes (informally)
\begin{equation}\label{expec pde}
    \partial_t v=\mathbb{E}_{u\sim \mu_t}[\mathcal{F}\mathcal{A}u | \mathcal{F}u=v],
\end{equation}
as is presented in \Cref{sec: liouville} in main text.
This describes (one of) the optimal dynamics in the reduced space.

\subsubsection{The Effect of Closure Modeling}
\label{b.2.3}
Apart from the original motivation of closure modeling to approximate the commutator $\mathcal{F}\mathcal{A}-\mathcal{A}\mathcal{F}$, we alternatively interpret it as assigning a vector field $\mathcal{A}_\theta$ in the reduced space $\mathcal{F}(\mathcal{H})$ and accordingly the coarse-grid dynamics is
\begin{equation}
    \partial_t v= \mathcal{A}_\theta v,
\end{equation}
here $\mathcal{A}_\theta$ plays the role of $\mathcal{A}v+clos(v;\theta)$ in \cref{eq:les-pde}.

We will refer to both $\mathcal{A}_\theta$ and $clos(\cdot;\theta)$ as the target of closure modeling for brevity.

As an application of \Cref{prop: station liou}, we only need to check the solution to the stationary Liouville equation related to this dynamics to decide whether or not the resulting limit distribution is the optimal one $\rho_1^*$.

\subsection{Detailes for Discussion in \Cref{sec: liouville}}
\label{apdx-b-3}
The dynamics of the filtered trajectory in \Cref{expec pde} (we will refer to the equivalent version \cref{ode: optimal} for convenience), which is also derived in \cite{langford1999optimal}, has inspired many works for the design of closure models. 
Unfortunately, we want to point out that it is impractical to utilize this result for closure model design. 

The decision regarding subsequent motion at the state $(\rvv,t)$ have to be made only based on 
information from the reduced space, which contains merely $\rvv$ itself and the distribution of $\rvv$. For any given $\rvv$, only one prediction can be made for the next time step. Similar to the non-uniqueness issue highlighted in \Cref{sec: sec2}, however, since $\rho(\rvw|\rvv;t)$ depends on $t$ and $\rho_0$(the initial distribution of $u\in\mathcal{H}$), typically there are multiple distinct $\rho(\rvv,\rvw,t)$ with exactly the same $\rvv$ and marginal distribution in $\mathcal{F}(\mathcal{H})$. 

In practice, to follow the form of conditional expectation as presented in \cref{ode: optimal}, one must select a specific distribution \( q(\rvv, \rvw) \) in \(\mathcal{H}\). The vector field in the reduced space will then be assigned as \(\mathbb{E}_{\rvw \sim q(\rvw|\rvv)}[f_r(\rvv,\rvw) \mid \rvv]\).

% In practice, if one hopes to follow the form of conditional expectation as in \cref{ode: optimal}, he can only fix one particular $q(\rvv,\rvw)$, a distribution in $\mathcal{H}$, and assign the vector field in reduced space as $\mathbb{E}_{\rvw\sim q(\rvw|\rvv)}[f_r(\rvv,\rvw)\big|\rvv]$.

Now, we check the limit distribution we will obtain with this dynamics. From \Cref{prop: station liou}, we know that limit distribution $\hat{\rho}_1(\rvv)$ is the solution to (usually in weak sense)
\begin{equation}
\label{43}
    \nabla_\rvv\cdot\bigg(\mathbb{E}_{\rvw\sim q(\rvw|\rvv)}[f_r(\rvv,\rvw)\big|\rvv]\rho_1(\rvv)\bigg)=0.
\end{equation}

\begin{proposition}\label{prop-optimal-closure}
    $\rho_1^*$ is the solution to \cref{43} if $q=\rho^*$.
\end{proposition}
\begin{proof}
By definition, $\rho^*(\rvv,\rvw)$ satisfies
\begin{equation}
\nabla_{\rvv}\cdot(f_r(\rvv,\rvw)\rho^*(\rvv,\rvw,t))+\nabla_{\rvw}\cdot(f_u(\rvv,\rvw)\rho^*(\rvv,\rvw,t))=0.
\end{equation}
Integral over $\rvw$ and use divergence theorem, we yield
\begin{align}
    0&=\int \nabla_\rvv\cdot\big(f_r(\rvv,\rvw)\rho^*(\rvv,\rvw)\big)d\rvw+0\\
    &=\nabla_\rvv\cdot \int f_r(\rvv,\rvw) \rho_1^*(\rvv)\rho^*(\rvw|\rvv)d\rvw\\
    &=\nabla_\rvv\cdot\bigg( \mathbb{E}_{\rvw\sim \rho^*(\rvw|\rvv)}[f_r(\rvv,\rvw)\big|\rvv]\rho_1(\rvv) \bigg)
\end{align}
This gives the proof.
\end{proof}

Thus, we show that $\rho^*$ is the correct choice for $q$  to guarantee convergence towards $\rho_1^*$ in $\mathcal{F}(\mathcal{H})$. Back to the learning methods discussed in \Cref{sec:learn_clos}, 
if we follow the new interpretation in \Cref{b.2.3}, the loss function is 
\begin{align}
J_{ap}(\theta)&=\mathbb{E}_{u\sim p_{data}}\|\mathcal{A}_\theta \mathcal{F}u-\mathcal{F}\mathcal{A}u\|^2\\
&=\mathbb{E}_{(\rvv,\rvw)\sim p_{data}(\rvv,\rvw)}|f_{r}(\rvv;\theta)-f_r(\rvv,\rvw)|^2,
\end{align}
where we transform the original objective function into $\ell^2$ space of $\rvz$ in the second line, and $f_r(\cdot;\theta)$ is the counterpart of $\mathcal{A}_\theta$ in $\ell^2$, $p_{data}$ is the empirical measure of training data from fully-resolved simulations (FRS).

Due to the $L^2$ variational characterization of conditional expectation, 
the underlying choice of $q(\rvv,\rvw)$ in those existing learning methods is ${p}_{data}$. 
As a remark, the learned closure model can only achieve this conditional expectation in the ideal case, i.e. the model class is expressive enough.
Consequently, one has to use numerous FRS training data due to the slow convergence of empirical measures of the high-dimensional distribution $\rho^*$.

\section{Proof of the first claim of \Cref{thm_clos_all}}
\label{apdx:random}

We first remind the readers of the three results in the first claim in \Cref{thm_clos_all}.

\begin{thm}\label{thm: random}
(Non-uniqueness of Closures)\\
    (i) In general, the target mapping of closure models $\ovu\to(\mathcal{F}\mathcal{A}-\mathcal{A}\mathcal{F})u$ is a multi-map. Consequently, the approximation error has a lower bound.\\
    (ii) For any $u$ and finite $\tau$, there exist \textbf{infinitely many} $u'\in\mathcal{H}$ such that $\mathcal{F}S(t)u'=\mathcal{F}S(t)u$ for all $ t\in[0,\tau)$. \\
    (iii) One cannot obtain the best approximation of $\mu^*$ among distributions supported in the reduced space with a stochastic closure model driven by continuous noice.
\end{thm}
\begin{rmk}
    For our second statement, since $\{\mathcal{F}S(t)u\}_{x\in D',\ t<\tau}$ contains the entire information that could be used to predict the closure term at time $\tau$, and $\ovu'(\cdot,t)$ will deviate dramatically from  $\ovu(\cdot,t)$ in a chaotic system, the under-determinancy issue remains unsolved with history-aware models. Clearly, for generic dynamics the approximation error of the closure term still has a lower bound due to the same reason as in the case for single-state closure models in (i).\\
    For our third claim, it indicates that randomness does not contribute to further enhancing the expressive power of the closure model, and thus it is theoretically redundant. Although it may have advantages in terms of training stability and implicit regularization. 
    An intuitive interpretation of the result is that stochasticity causes blurring, unless the noise vanishes, in which case the model reduces to a deterministic map.
\end{rmk}

For the first  claim in \Cref{thm: random}, it has already been motivated in the main text that the mapping of closure model $\ovu\to(\mathcal{F}\mathcal{A}-\mathcal{A}\mathcal{F})u$ is not well-defined. We will make this claim more precise in \Cref{apdx: thm aproxer}, and then give the proof for the second claim in \Cref{apdx:hist} and the proof for the third claim in \Cref{apdx:thm:sde}. The assumptions and notation conventions are summarized in \Cref{app:apdxA}. We also adhere to the equivalence between $u\in\hhh$ and its $\ell^2$ representation $\rvz$ established in \Cref{apdx-b-1}.

% For brevity, we will view $u\in\mathcal{H}$ and $\rvz\in\ell^2$ as the same and not mention $\mathcal{T}^\#$ or $\mathcal{T}_{\#}$ for $\mu_t$ and $\rho(\cdot,t)$.

\subsection{Proof of \Cref{thm: random}(i)}
\label{apdx: thm aproxer}
By transforming the original dynamics into the space of $\ell^2$, it is easier to see why the mapping of closure model is not well-defined (more specifically, it is a multi-valued mapping). Since $\aaa\mathcal{F} u=\aaa\ovu$, we only need to show that $\ovu\to\mathcal{F}\aaa u$ is not well defined. The counterpart of this mapping in $\ell^2$ space is $\rvv\to f_r(\rvv,\rvw)$. If it were a well-defined mapping, there would be a mapping $\tilde{f}_r(\rvv)$ such that $f_r(\rvv,\rvw)\equiv \tilde{f}_r(\rvv)$ for all $\rvw$. In other words, the reduced system is independent of the unresolved part. This property rarely holds in most dynamical systems, except for a few trivial cases like the heat equation.

Next we show that the approximation error has a positive lower bound.
% For the first claim in \Cref{thm: random}, it has already been shown in the main text that the mapping of closure models $\ovu\to(\mathcal{F}\mathcal{A}-\mathcal{A}\mathcal{F})u$ is not well-defined.

We could always construct $u_1,u_2\in\hhh$ such that $\ovu_1=\ovu_2$ and $\mathcal{F}\aaa u_1\neq \mathcal{F}\aaa u_2$. Therefore, for any model $clos(\ovu;\theta)$ the approximation error
\begin{align}
    &\sup_{u\in\hhh}\|clos(\ovu;\theta)-(\mathcal{F}\mathcal{A}-\mathcal{A}\mathcal{F})u\|_\hhh\\
    \geq &
     \sup_{u\in\{u_1,u_2\}}\|clos(\ovu;\theta)-(\mathcal{F}\mathcal{A}-\mathcal{A}\mathcal{F})u\|_\hhh \\
    \geq & \frac{1}{2} \big(||(\mathcal{F}\mathcal{A}-\mathcal{A}\mathcal{F})u_1-clos(\overline{u};\theta)||_{\mathcal{H}}+||clos(\overline{u};\theta)-(\mathcal{F}\mathcal{A}-\mathcal{A}\mathcal{F})u_2||_{\mathcal{H}}\big)\\
    \geq & \frac{1}{2} ||(\mathcal{F}\mathcal{A}-\mathcal{A}\mathcal{F})u_1-(\mathcal{F}\mathcal{A}-\mathcal{A}\mathcal{F})u_2||_\mathcal{H}\\
    = & \frac 1 2\|\mathcal{F}(\aaa u_1-\aaa u_2)\|_\hhh 
\end{align}
has a lower bound independent of the model, where we apply the fact that $\mathcal{F}u_1=\mathcal{F}u_2=\ovu$ in the last line.

% We will next give the proof for the second claim in \Cref{apdx:hist} and the proof for the third claim in \Cref{apdx:thm:sde}.

\paragraph{Implications for Training}
From a practical perspective, the derivation above justifies that the mapping of the closure model is non-unique, regardless of the ansatz of the model and how it is trained. Moreover, we directly obtain the lower bound for a priori training loss \cref{mthd1} with similar arguments.
There have been some works exploring posterior training \cite{shankar2023differentiable,guan2022stable}
\cref{mthd_post}, aiming to overcome this fundamental issue with non-uniqueness of the mapping. However, this issue persists. Specifically, there are multiple trajectories from fully-resolved simulations sharing the same filtered initial value, yet the filtered trajectories differ in the reduced space.

\subsection{Proof of Theorem \ref{thm: random}(ii)}
\label{apdx:hist}

In this subsection, we will prove that leveraging the history information in coarse-grid simulations can not resolve the multi-valued-mapping issue discussed in \Cref{thm: random}(i). The main idea of the proof is to construct multiple trajectories that are not distinguishable in the coarse-grid system within a finite time period. We will first handle the simple case where history information is only observed at a finite set of time grid (\Cref{thm:hist:finite}). This result itself already exposes the shortcomings of history-aware learning-based closure models, since in practice the learned model can only make use of the information from a fixed number of time grids. To strengthen our result, we present the proof of the original statement of \Cref{thm: random}(ii) in \ref{thm:hist_apdx}.

We start by introducing several notations.

\textbf{Notations:}
Recall that $\mathcal{H}$ is the function space, and 
$D'$, whose cardinal is $n$, is the set of coarse grid points, 
$D'=\{x_1,x_2,...,x_n\}$. The filtered value of two functions being the same is equivalent to the fact that these two functions have the same values on the grid points in $D'$.
\begin{defn}\label{def Ix}
Define the grid-measurement operator (at $x_0$) 
\begin{equation}
    I_{x_0}\colon \mathcal{H}\to\mathbb{R},\ u\mapsto u(x_0)
\end{equation}
For brevity, we will use $I_j$ for $I_{x_j}$. We further define $I_{D'}$ (abbreviated as $I$ if there is no ambiguity),
\begin{equation}
    I_{D'}\colon \mathcal{H}\to \mathbb{R}^n,\ u\mapsto (u(x_1),u(x_2),...u(x_n))^T.
\end{equation}
\end{defn}

Before we delve into the details of the proof, we would like to remind the readers of heat equation as an concrete and easy-to-check example where our result holds. 

Our original theorem is stated for a continuous time interval. We first prove its finite version.
\begin{thm}\label{thm:hist:finite}
    Given $D_T$ the set of time grids, with $|D_T|=N$ and $D_T=\{t_1,t_2,...t_N\}$, for any $r\in\mathbb{N}$, any function $u_0\in\mathcal{H}\cap C_0$, there exists an $r$-dimensional manifold $M_r\subset\mathcal{H}\cap C_0$ such that
    \begin{equation}
        IS(t)u'=IS(t)u_0,\ \forall t\in D_T, \ \forall u'\in M_r.
    \end{equation}
\end{thm}
\begin{proof}
    For $m\in\mathbb{N}$, given $m$ linearly independent functions $\{\phi_i\}_{1\leq i\leq m}\subset\mathcal{H}$, we can construct an affine manifold $A:= u_0+\operatorname{span}\{\phi_1,...\phi_m\}$ and define the following mapping (spatiotemporal grid-measurement operator):
    \begin{align}
        \fff=\mathfrak{F}_{D_T}:\ &A\to\mathbb{R}^{nN}\\
        &v\ \mapsto\ \mathop{\bigoplus_{j=1}^{N}}IS(t_j)v.
    \end{align}

Note that we have a canonical coordinate system for $A$:
    \begin{align}
        \ A\ &\leftrightarrow\mathbb{R}^{m}\\
        v=u_0+\sum_{i=1}^m c_i\phi_i\ &\leftrightarrow (c_1,...c_m).
    \end{align}
Thus, $\mathfrak{F}$ is a mapping between finite-dimensional manifolds, and we can compute its Jacobian\\
$\mathcal{J}(v)\in\mathbb{R}^{m\times Nn}$ for $v\in A$, whose elements consist of the grid-measurement of Gateaux derivatives, $I_jdS(t_k)(v,\phi_l),\ j\in[n], \ k\in[N],\ l\in[m] $.

By a generalization of Sard's theorem for Banach manifold~\cite{smale2000infinite}, we know that for any $u_0$, any $r$, there exists $m=Nn+r$ linearly independent functions $\{\phi_i\}_{i=1}^m$ such that the Jacobian is everywhere full-rank (i.e., rank$=Nn$) in the affine manifold $A$. 
By pre-image theorem, since $\mathfrak{F}^{-1}\{\mathfrak{F}u_0\}$ is non-empty (for at least $\mathfrak{F}u_0$ is in this set), it is an $m-Nn=r$ dimensional manifold. This gives the proof.
\end{proof}
\begin{corollary}
 Given $D_T$ the set of time grids, with $|D_T|=N$ and $D_T=\{t_1,t_2,...t_N\}$, for any $r\in\mathbb{N}$, any function $u_0\in\mathcal{H}\cap C_0$, there exists infinite $u'\in\mathcal{H}$ such that $IS(t)u'=IS(t)u_0$ for all $t\in D_T$.
\end{corollary}
\begin{proof}
    For any $r\geq 1$, there are infinite points in the manifold we yield in the theorem above.
\end{proof}

\begin{rmk}
   We assume $u \in C_0$ primarily to exclude pathological cases of modifying the function value on zero-measure sets.
\end{rmk}

From the result above, we see that for any $u_0$ and finite time-grid set $D_T$, $\mathfrak{F}^{-1}\{\mathfrak{F}u_0\}$ is an infinite-dimensional manifold.
To complete our proof, we make two technical assumptions. One can check these assumptions for specific dynamical systems to derive the final result. We also remark that they are not the weakest set of assumptions to guarantee the final result, we adopt them here primarily to keep the proof concise.

For a time-grid set $D_T$, and a function $u_0$, we denote as $M_{D_T}(u_0)$ the set of all functions $u'$ that $IS(t)u'=IS(t)u_0,\ \forall t\in D_T$.

\begin{ass}\label{ass:unbound}
    For every $u_0$ and finite $D_T$, the infinite-dimensional manifold $M_{D_T}(u_0)$ is unbounded.
\end{ass}

\begin{exmp}
    For heat equation, $M_{D_T}(u_0)=u_0+\operatorname{span}\{\psi_{n+1},\psi_{n+2},...\}$, which is an infinite-dimensional unbounded manifold.
\end{exmp}

\begin{ass}\label{ass:continuous mnfd}
Given any finite $\tau$, and an arbitrary bounded set $\Omega\subset\mathcal{H}$, 
we could assign a sequence of linearly independent functions $\{\phi_{i;u}\}_{i=1}^{\infty}\subset\hhh$ for each $u\in\hhh$ such that 
for any functions $u,\ v$, any subset $B\subset\nnn$, and any finite subset $D_T$ of $[0,\tau]$, there exists a continuous mapping 
\begin{equation}
    G\colon M_{D_T}(u)\bigcap \big\{ u + \operatorname{span}\{\phi_{i;u} \mid i \in B\} \big\} \to M_{D_T}(v) \bigcap \big\{ v + \operatorname{span}\{\phi_{i;v} \mid i \in B\} \big\}
\end{equation}
depending only on $u,\ v,\ B,\ D_T$, such that 
\begin{equation}
    \sup\limits_{w\in \Omega\cap M_{D_T}(u)} \|Gw-w\|\leq C_\Omega \|u-v\|,
\end{equation}
where the constant $C_\Omega$ only depends on $\Omega$ and $\tau$, and not on $u,\ v,\ B,\ D_T$.
\end{ass}
\begin{rmk}
    Intuitively, this assumption enforces certain continuity in the dependence of the indistinguishable set $M_{D_T}(u)$ on $u$.
\end{rmk}

Before moving on, we first review a well-known result.
\begin{fact}\label{lemma:n to n2}
    There exists a bijection between $\mathbb{N}$ and $\mathbb{N}^2$.
\end{fact}
% \begin{proof}
%    This is a standard result in set theory, and its proof can be found in many textbooks. 

%    We give an example on how to construct such a bijection (see \Cref{fig:diag}). We write 1,2,3,4,... zigzaggingly to fill the $\mathbb{N}^2$ plane. In this way, we construct a mapping $\iota:\mathbb{N}^2\to\mathbb{N}$, with $\iota(i,j)$ defined as the value written at $(i,j)$-position in the $\mathbb{N}^2$ plane. Clearly, $\iota$ is a bijection.
% \end{proof}

This is standard with Cantor's diagonal argument.

Denote by $\iota$ the bijection mapping from $\nnn^2$ to $\nnn$.
With $\iota$, we can partition $\mathbb{N}$ into 
$\aleph_0$(countably infinite) disjoint subsequences, $\{\iota(i,j)\}_{j\in\mathbb{N}}$ for each $i$.

% \begin{figure}[htbp] 
% \centering
% \begin{tikzpicture}
%   % Nodes
%   \node (1) at (0,0) {1};
%   \node (2) at (1,0) {2};
%   \node (3) at (0,-1) {3};
%   \node (4) at (0,-2) {4};
%   \node (5) at (1,-1) {5};
%   \node (6) at (2,0) {6};
%   \node (7) at (3,0) {7};
%   \node (8) at (2,-1) {8};
%   \node (9) at (1,-2) {9};
%   \node (10) at (0,-3) {10};
%   \node(12) at (1,-3){12};
%   \node(13) at (2,-2){13};
%   \node(14) at (3,-1){14};

%   % Arrows
%   \draw[->,thick] (1) -- (2);
%   \draw[->,thick] (2) -- (3);
%   \draw[->,thick] (3) -- (4);
%   \draw[->,thick] (4) -- (5);
%   \draw[->,thick] (5) -- (6);
%   \draw[->,thick] (9) -- (10);
%   \draw[->,thick] (8) -- (9);
%   \draw[->,thick] (7) -- (8);
%   \draw[->,thick] (6) -- (7);
%   \draw[->,thick] (12) -- (13);
%   \draw[->,thick] (13) -- (14);

%   % Ellipses
%   \node at (4,-1) {$\cdots$};
%   \node at (4,-2) {$\cdots$};
%   \node at (4,-3) {$\cdots$};
%   \node at (0,-4) {$\vdots$};
%   \node at (1,-4) {$\vdots$};
%   \node at (2,-4) {$\vdots$};
%   \node at (3,-4) {$\vdots$};
%   \node at (4,0) {$\cdots$};
% \end{tikzpicture}
% \caption{Illustration of a bijection between $\mathbb{N}$ and $\mathbb{N}^2$ using a zigzag numbering scheme.}
% \label{fig:diag} 
% \end{figure}

We are finally ready for the proof of \Cref{thm: random}(ii).
\begin{thm} 
For any $u\in\mathcal{H}$ and finite $\tau$, there exist \textbf{infinitely many} $u'\in\mathcal{H}$ such that $\mathcal{F}S(t)u'=\mathcal{F}S(t)u$ for all $ t\in[0,\tau)$. 
\label{thm:hist_apdx}
\end{thm}
\begin{proof}
    Let $\kkk$ be a dense Hilbert subspace of $\hhh$ that can
    compactly embed into $\hhh$, with norm $\|\cdot\|_\kkk$. For instance, if $\hhh=H^k$, the Sobolev space $W^{2,k}$, then we can choose $\kkk$ as $H^{k+1}$.
    
    \textit{Step 1:} We first deal with the case when $u\in\kkk$.

Define a sequence of time-grid set $D_j$ as $\{\frac i {2^j}\tau| 0\leq i<2^j\}$. Similar to the argument in \Cref{thm:hist:finite}, we can construct a sequence $\{\phi_i\}_{i=1}^{\infty}\subset\kkk$ satisfying the following properties.
\begin{enumerate}[label=(\roman*)]
    \item They are linearly independent.
    \item For any $j\in \nnn$, $m>2^j$, $B\subset\nnn$ with $|B|=m$, the spatiotemporal grid-measurement $\fff$ for $D'\times D_j$ has full-rank Jacobian everywhere in the affine manifold $u+\operatorname{span}\{\phi_k:\ k\in B\}.$
\end{enumerate}
Based on the $\nnn-\nnn^2$ bijection $\iota$, we define the following subspaces:
\begin{equation}
    E_k:=\operatorname{span}\{\phi_{\iota(k,i)}:\ i\in\nnn\}.
\end{equation}
There are $\aleph_0$ such subspaces in total, we will next find a point $u'$ in each $E_k$ such that\\
$IS(t)u'=IS(t)u$ for all $ t\in[0,\tau)$.

WLOG, we will only show how to construct $u'$ in $E_1$.

We denote 
\begin{equation}
    M_j=E_1\cap M_{D_j}(u).
\end{equation}
By the construction of $D_j$ we have $M_1\supset M_2\supset M_3\supset...$.

We first fix three constants $0<B_0<B_1,\ B_2>0$ 
and construct a sequence of $\{u_i\}\subset E_1$ such that 
\begin{enumerate}[label=(\roman*)]
    \item $u_i\in M_i$.
    \item $B_0<\|u_i-u\|_\hhh<B_1$.
    \item $\|u_i\|_\kkk<B_2$.
\end{enumerate}
This construction is achievable due to \Cref{ass:unbound} and the fact that $u\in\kkk$.

Since $\kkk$ can be compactly embedded into $\hhh$, there there exists a subsequence $\{u_{i_j}\}$ of $\{u_i\}$ that is convergent in $\hhh$. We denote its limit as $u_{\infty}$. From (ii), we have $\|u_\infty\|_\hhh<\infty$ and $u_\infty\neq u$.

Due to \Cref{ass:h to c}, we have $u_{i_j}\to u_\infty$ in $C_0$, which implies that for any $D_j$, $\fff_{D_j}u_\infty=\fff_{D_j}u$.

Because of the continuity of the mapping $t\mapsto I_x S(t)v$ for any $x\in D',\ v\in\hhh$, we know that $IS(t)u'=IS(t)u$ for all $ t\in[0,\tau)$.

To conclude, for each $E_k$, there exists $u'\in E_k$ that is not distinguishable from $u$ merely based on function values restricted to the grid $D'\times D_T$. Recall that for any $j\neq k$, by construction we have $E_j\cap E_k=\{u\}$, thus these $u'$ in different $E_k$ are mutually different. This completes the proof for the case $u\in\kkk$.

\textit{Step 2} Now we give the proof for general $u\in\hhh$.

Since $\kkk$ is dense in $\hhh$, there exists a sequence $\{u^n\}_{n\in\nnn}\subset\kkk$ such that $\|u^n-u\|_\hhh<\frac 1 {2^n}$. We keep using the constant $B_0,B_1,B_2$ and define $\Omega:=\{v\in\hhh\mid\|v-u\|_\hhh<B_1+1\}$. Following \Cref{ass:continuous mnfd}, we obtain the constant $C_\Omega$ and linearly independent set $\{\phi_{i;u^n}\}_{i=1}^{\infty}$ for each $n$.
Following the first part of this proof, we define
\begin{equation}
    E_k^n:=\operatorname{span}\{\phi_{\iota(k,i);u^n}\mid i\in\nnn\}\quad M_{j,k}^n:=E_k^n\cap M_{D_j}(u^n),
\end{equation}
and again we only need to consider the case when $k=1$, and thus abbreviate $M_{j,k}^n$ as $M_j^n$.

We will restrict our discussion within $n>n_0:=\bigg\lceil\max\{\log_2\frac{6(C_\Omega+1)}{B_0},\log_2 3(C_\Omega+1)\}\bigg\rceil+1$.

We inductively construct a sequence (indexed by $n$) of sequence $\{u_j^n\}_j\subset E_1^n$ as follows:

(I) For $n=n_0$, we construct $\{u^n_j\}_j$ the same as the first part of the proof.
\begin{enumerate}[label=(\roman*)]
    \item $u_j^n\in M_j^n$.
    \item $B_0<\|u_j^n-u^n\|_\hhh<B_1$.
    \item $\|u_j^n\|_\kkk<B_2$.
\end{enumerate}

(II) Now suppose we have constructed $\{u^n_j\}_j$,
we apply \Cref{ass:continuous mnfd} for $u^n,\ u^{n+1}$,\\
$B=\{\iota(1,i)\mid i\in\nnn\}$ and $D_j$ and obtain a continuous mapping $G$. We choose $u^{n+1}_j$ as $Gu_j^n$.

Next, we give some estimations for $u_j^{n+1}$.

First, note that we have 
\begin{equation}
    \|u^n-u^{n+1}\|_\hhh\leq \|u^n-u\|_\hhh+\|u-u^{n+1}\|_\hhh<\frac 3 {2^{n+1}},
\end{equation}
and thus $\|u_j^n-u_j^{n+1}\|_\hhh\leq C_\Omega\frac 3 {2^{n+1}}$ by construction. Based on this, we have
\begin{align}
    \|u_j^{n+1}-u^{n+1}\|_\hhh&\geq\|u_j^n-u^n\|_\hhh-\|u_j^n-u_j^{n+1}\|_\hhh-\|u^n-u^{n+1}\|_\hhh\\
    &\geq
    \|u_j^n-u^n\|_\hhh-\frac{3(C_\Omega+1)}{2^{n+1}}.
\end{align}
By induction, we have 
\begin{align}
    \|u_j^{n+1}-u^{n+1}\|_\hhh\geq\|u_j^{n_0}-u^{n_0}\|_\hhh-\sum_{n'>n_0}\frac{3(C_\Omega+1)}{2^{n'+1}}>\frac {B_0}2.
\end{align}
We also have
\begin{align}
    \|u_j^{n+1}-u^{n+1}\|_\hhh&\leq\|u_j^n-u^n\|_\hhh+\|u_j^n-u_j^{n+1}\|_\hhh+\|u^n-u^{n+1}\|_\hhh\\
    &\leq
    \|u_j^n-u^n\|_\hhh+\frac{3(C_\Omega+1)}{2^{n+1}}.
\end{align}
By induction, we have 
\begin{align}
    \|u_j^{n+1}-u^{n+1}\|_\hhh\leq\|u_j^{n_0}-u^{n_0}\|_\hhh+\sum_{n'>n_0}\frac{3(C_\Omega+1)}{2^{n'+1}}<B_1+\frac 1 2.
\end{align}

Similar to what is done in the first part of the proof, we can choose $v^{n_0}$ as one of the limit points of $\{u^{n_0}_j\}_j$. Inductively, we can construct a sequence of $v^n$ such that 
\begin{enumerate}[label=(\roman*)]
    \item $v^n$ is one of the limit points of  $\{u^{n}_j\}_j$
    \item $\|v^n-v^{n-1}\|_\hhh\leq \frac{3C_\Omega+1}{2^n}$.
\end{enumerate}
Thus, $\{v^n\}_n$ is a Cauchy sequence in $\hhh$ and we denote its limit as $v$. Because of \Cref{ass:h to c}, we have that $IS(t)v=IS(t)u,\ \forall t\in[0,\tau)$. It is also clear that 
\begin{equation}
v\in M_{[0,\tau)}(u)\cap \big\{u+\operatorname{span}\{\phi_{\iota(1,i);u} \mid i \in \nnn\} \big\},    
\end{equation}
and 
\begin{equation}
 \|v-u\|_\hhh=\|\lim\limits_{n\to\infty}(v^n-u^n)\|_\hhh\geq\liminf_{n\to\infty}\|u^n-v^n\|_\hhh\geq \frac{B_0}  2.
\end{equation}
With exactly the same argument as in the first step, we construct infinitely many mutually-different functions that are not distinguishable from $u$ on $D'\times [0,\tau)$. This completes the proof.
\end{proof}

\subsection{Proof of Theorem \ref{thm: random}(iii)}
\label{apdx:thm:sde}

In this subsection, we formally investigate the effect of randomness in closure modeling. 

Similar to the sections above, the key idea is to check how the distributions of functions evolves. 
The formal version of \Cref{thm: random}(iii) is stated as \Cref{thm_random_apdx} (for stochastic closure model with explicit additive randomness) and \Cref{apdx-thm-random-2} (for stochastic closure model with hidden randomness), and their proofs therein. 

\subsubsection{Formulation of Stochastic Closure Models}
We start by reminding readers of the notations.
We use $u$ for functions in $\hhh$ and $v$ for functions in the reduced space $\mathcal{F}(\hhh)$. We omit the difference between functions and their $\ell^2$ representations, namely $u\leftrightarrow\rvz\in\ell^2$, with $\rvz=\rvv+\rvw$, as discussed in \Cref{apdx-b-1}. Usually $\rvv$ is a $n=dim(\mathcal{F}(\hhh))$ dimensional vector.
In this subsection, we use the notation $\rho$ for general probability distributions, not necessarily coming from $\mathscr{P}(\hhh)$. In most cases, it represents the density functions for distributions in $\mathscr{P}(\mathcal{F}(\hhh))$.

We divide the mainstream stochastic closure modeling approaches into two types, \textit{models with explicit additive randomness}~\cite{culina2011stochastic,franzke2005low} and \textit{models with hidden stochasticity} \cite{wu2023learning,dong2024data,lu2017data}.

For the first type, the coarse-grained dynamics after transforming into $\ell^2$ space takes the form 
\begin{equation}
\label{sde}
    d\rvv=b(\rvv)dt+\sigma(\rvv)dW,
\end{equation}
where $b$ is an $n$-dimensional vector field, $dW$ represents $m$-dimensional Brownian motion for certain $m$, and $\sigma\in\rrr^{n\times m}$.
We remark here that for approaches that adopt additive noise in a latent space, e.g. \cite{boral2023neural}, the dynamics still take this form due to Itô's formula.

For the second type, its general form is 
\begin{align}
\label{sde-2}
    d\rvv &= (b(\rvv)+g(\rvz))dt\\
    d\rvz&= h(\rvv,\rvz)dt+\sigma dW,
\end{align}
where $b$ is an $n$-dimensional vector field, usually chosen to be $f_r$. $\rvz$ is an $n-$ dimensional latent variable used to represent the unresolved information in the coarse-grid system. $g:\rrr^n\to\rrr^n$ is an $n$-dimensional vector field. In most papers, $g$ is identity mapping. $h$ is a mapping from $\rrr^n\times\rrr^n$ to $\rrr^n$. $\sigma\in\rrr^{n\times n}$ is a non-degenerate matrix, and $dW$ represents $n$-dimensional Brownian motion.

All methods that sample random closure terms from a smooth conditional distribution $p(\cdot|\rvv)$ can be unified into the framework above. This is particularly evident for methods that generate samples using a diffusion model \cite{song2020score}.

\subsubsection{Stochastic Closure with Additive Randomness}
\begin{thm}
Following the notations in \cref{sde}, if $\sigma$ is Lipschitz-continuous and non-vanishing everywhere, i.e. $\sigma\neq\textbf{0},\ \forall \rvv\in \rrr^n$, then $\rho_1^*$ is not the limit distribution of \cref{sde}. In other word, one cannot obtain obtain $\rho_1^*$ through the dynamics with stochastic closure model (\ref{sde}).
\label{thm_random_apdx}
\end{thm}

\paragraph{Preliminary of the Proof}
We start by reminding readers of classical results on second-order elliptic operators that will play a vital role in the proof.

%%%%

    Consider a homogeneous second-order differential operator of general form, defined on an $n$-dimensional region $\Omega$,
    \begin{equation}
        \mathfrak{L}\colon u(x)\mapsto \sum_{1\leq i,j\leq n}a_{ij}(x)\partial_{ij}^2u(x)+\sum_{1\leq i\leq n}b_i(x)\partial_iu(x)+c(x)u(x) ,\ x\in\Omega\subset\rrr^n,
    \end{equation}
    where $a_{ij},\ b_i,$ and $c$ are scalar functions, and $a_{ij}\equiv a_{ji}$.
    We denote matrix $\big( a_{ij}(x)\big)_{ij}$ as $A(x)$.
\begin{defn}
The operator $\mathfrak{L}$ is said to be \emph{uniformly elliptic} if there exist constants $0 < \lambda \leq \Lambda$ such that for all $x \in \Omega$ and for all $\xi \in \mathbb{R}^n$, the following inequality holds:
\[
\lambda |\xi|^2 \leq \xi^TA(x)\xi \leq \Lambda |\xi|^2.
\]
\end{defn}

\begin{lemma}[Unique Continuation Principle]
\label{lemma:ucp}
    For a uniformly elliptic operator $\mathfrak{L}$ whose coefficient functions $a_{ij}$ are Lipschitz-continuous, if $u$ is the solution to $\mathfrak{L}u=0$ and $u(x)=0$ on an open subset of  $\Omega$, then $u\equiv 0$ on $\Omega$.
\end{lemma}
The proof can be found in \cite{aronszajn1956unique,garofalo1986monotonicity}.

\paragraph{Intuition of the Proof to \Cref{thm_random_apdx}}
The evolution of the distribution of $\rvv$ (functions in the reduced space) follows the Fokker-Planck equation
\begin{equation}
    \partial_t\rho(\rvv,t)=-\nabla\cdot(b\rho)+\frac 1 2 \nabla^2:(\sigma\sigma^T\rho),
\end{equation}
where $:$ denotes the Frobenius inner product.
The limit distribution serves as the solution to its stationary equation (by setting the L.H.S.\@ to zero). 
We can expand the L.H.S.
and write it in the following form
\begin{equation}
\label{stat_fkplk}
A(\rvv):\nabla^2\rho+B(\rvv)\cdot\nabla\rho+C(\rvv)\rho=0,
\end{equation}
which is an elliptic equation.
In particular, $A=\frac 1 2\sigma\sigma^T$.

We have shown in \Cref{apdx-b-3} that the best approximation of $\mu^*$ one can obtain in the coarse-grid system, namely $\rho_1^*$, is the solution to a first-order PDE (\ref{prop: station liou}). 
Due to the significant differences in the nature of first-order PDEs and second-order elliptic equations, $\rho_1^*$ is not likely to be a solution of the latter.

Our proof will be based on \Cref{lemma:ucp}. Intuitively, if $\sigma$ is full-rank everywhere and $A$ is uniformly elliptic, it is clear that $\rho_1^*$ can not be a solution to \cref{stat_fkplk} since it is compactly supported on the attractor. It remains to address the case where $A$ occasionally degenerates for some $\rvv$. The key idea is that even when $A$ degenerates, the second-order coefficient matrix is full-rank if restricted to certain linear subspace depending on $\rvv$. We will turn to examine the density function of an auxiliary variable induced by the original dynamics (\ref{sde}). We will apply \Cref{lemma:ucp} for the equation of the auxiliary variable to derive the final result.

\paragraph{Formal Proof}
\begin{proof}
We argue by contradiction and suppose $\rho_1^*$ is the solution to \cref{stat_fkplk}.

WLOG, we assume the attractor $\mathfrak{A}$ is contained in the ball $B_0(\frac R 2)\subset\ell^2$.

Next, we will construct an auxiliary variable with the desired properties to apply the unique continuation principle. We first define the following mapping for $C^1$- vector fields in $\rrr^n$:
\begin{align}
    V\colon C^1(\rrr^n,\rrr^n) &\to C(\rrr^n,\rrr^m)\\
    l(\rvv)&\mapsto V_l(\rvv)=\sigma^Tl+\sigma^T\mathcal{J}_l(\rvv)\rvv,
\end{align}
where $\mathcal{J}$ denotes the Jacobian, and $m$ is the dimension of Brownian motion in \cref{sde}.

Since $\sigma\neq \mathbf{0}$ for all $\rvv\in\rrr^n$, we could always find $l$ such that $\vec{0}\notin V_l(B_0(R))$. We will fix $l$ as one such vector field in the following discussion. Due to the linearity of $V$, we could further enforce that $|l(\rvv)|\leq 1$ for all $\rvv\in B_0(R)$ by scaling $l$ with a positive constant factor.

Consider a scalar quantity $y_t:=l(\rvv_t)^T\rvv_t$. Here we use the subscript $t$ to notify the dependence on time.
By Itô's formula, the evolution of $y$ is described by
\begin{equation}\label{sde:y}
    dy=b_ydt+V_l(\rvv_t)^TdW,
\end{equation}
    where we use $b_y$ to represent the time-derivative term.

Then, we check the (limit) distribution of $y$. Since the augmented system for $(\rvv_t, y_t)$ does not possess a proper probability density function (the joint distribution is $\rho(\rvv,y)=\rho(\rvv)\delta_{l(\rvv)^T\rvv}$), we derive the distribution of $y$ in a indirect way. Consider a stochastic dynamics for $(\rvv_t^\epsilon, y_t^\epsilon)$ defined as
\begin{equation}
    d
\begin{pmatrix}
\rvv^\epsilon_t \\
y^\epsilon_t
\end{pmatrix}
=
\begin{pmatrix}
b(\rvv^\epsilon_t) \\
b_y(\rvv^\epsilon_t)
\end{pmatrix}
dt
+
\begin{pmatrix}
\sigma(\rvv^\epsilon_t)\\
V_l(\rvv^\epsilon_t)^T
\end{pmatrix}
dW+\epsilon d\tilde{W},
\end{equation}
where $b,\ b_y$ comes from \cref{sde} and \cref{sde:y} respectively, and $d\tilde{W}$ is a $(n+1)$-dimensional Brownian motion independent from the $m$-dimensional $dW$.
We could write out the Fokker-Planck equation for the joint distribution of $(\rvv^\epsilon_t,y^\epsilon_t)$ at time $t$, denoted as $\rho^\epsilon(\rvv,y,t)$. The second-order term of the equation is
\begin{align}
    &\frac 1 2\begin{pmatrix}
    \sigma & \epsilon I_n & 0\\
    V_l^T & 0 & \epsilon
    \end{pmatrix}
        \begin{pmatrix}
    \sigma & \epsilon I_n & 0\\
    V_l^T & 0 & \epsilon
    \end{pmatrix}^T:
        \begin{pmatrix}
    \nabla^2_{\rvv\rvv}\rho^\epsilon &\nabla^2_{\rvv y}\rho^\epsilon \\
    \nabla^2_{y\rvv}\rho^\epsilon&\partial^2_{yy}\rho^\epsilon
    \end{pmatrix}\\
    =& \frac 1 2\big( \epsilon^2 \Delta_{\rvv\rvv}\rho^\epsilon+(\sigma\sigma^T):\nabla^2_{\rvv\rvv}\rho^\epsilon+2V_l^T\sigma^T\nabla^2_{\rvv y}\rho^\epsilon+(|V_l|^2+\epsilon^2)\partial^2_{yy}\rho^\epsilon\big).
    \label{eq85}
\end{align}
We can then yield the limit distribution of $y^\epsilon$ as the marginal distribution of $\rho^\epsilon(\rvv^\epsilon,y^\epsilon;\infty)$, the solution to stationary Fokker-Planck equation,
\begin{equation}
    \rho_{y^\epsilon}(y)=\int_{\rrr^n}\rho^\epsilon(\rvv,y)d\rvv,
\end{equation}
where we abbreviate $\rho^\epsilon(\rvv^\epsilon,y^\epsilon;\infty)$ as $\rho^\epsilon(\rvv,y)$.
From this relation we know $\rho_{y^\epsilon}(y)$ satisfies a second-order equation, whose second-order term comes from 
\begin{align}
    &\int_{\rrr^n}(|V_l|^2+\epsilon^2)\partial^2_{yy}\rho^\epsilon(\rvv,y)d\rvv\\
=&D^2_{yy}\bigg(\int_{\rrr^n}\rho^\epsilon(\rvv,y)d\rvv \frac{\int_{\rrr^n}(|V_l|^2+\epsilon^2)\rho^\epsilon(\rvv,y)d\rvv}{\int_{\rrr^n}\rho^\epsilon(\rvv,y)d\rvv}  \bigg)\\
:=& D_{yy}^2(\rho_{y^\epsilon}(y)K^\epsilon(y)).
\end{align}
This results from the fact that only the last term in \cref{eq85} contributes to the second-order derivative w.r.t. $y$. The second order term of the equation $\rho_{y^\epsilon}$ satisfies is $K^{\epsilon}(y)D^2_{yy}\rho_{y^\epsilon}(y)$.

Now we take the limit $\epsilon\to 0$. Due to the Lipschitz continuity of $\sigma(\rvv)$ and $V_l(\rvv)$ (easy to verify) on $B_0(R)$, $\rho_{y^\epsilon}$ will converge to $\big(l(\rvv)^T\rvv\big)_{\#}\rho_1^*$ (see Theorem 2.2 in \cite{mao2007stochastic} for related results on stochastic differential equations).
Recall that $\rho_1^*$ is supported on $P(\mathfrak{A})\subset B_0(\frac R 2)$ ($P$ defined in \Cref{apdx-b-2}). For $y\in \big((l(\rvv)^T\rvv)\circ P\big)(\mathfrak{A})$, $K^0(y)$ takes the form of an conditional expectation 
\begin{equation}
    K^0(y):=\lim\limits_{\epsilon\to 0}K^{\epsilon}(y)=\mathbb{E}_{\rvv\sim \rho_1^*}[|V_l(\rvv)|^2\big | l(\rvv)^T\rvv=y].
\end{equation}
Recall that our construction of $V_l$ requires $\vec{0}\notin V_l(B_0(R))$, which implies that there exists a constant $\eta>0$ such that $|V_l(\rvv)|^2>\eta$ for all $\rvv\in B_0(R)\supset P(\mathfrak{A})$. Consequently, $K^0(y)\geq \eta$.

In the other case, for $y\notin \big((l(\rvv)^T\rvv)\circ P\big)(\mathfrak{A})$, note that
\begin{align}
    K^\epsilon(y)&\geq \epsilon^2+\frac{\eta \int_{B_0(R)}\rho^\epsilon(\rvv,y)d\rvv}{\int_{B_0(R)}\rho^\epsilon(\rvv,y)d\rvv+\int_{B_0(R)^c}\rho^\epsilon(\rvv,y)d\rvv}\\
    &=\epsilon^2+\eta(1-\frac{\int_{B_0(R)^c}\rho^\epsilon(\rvv,y)d\rvv}{\int_{\rrr^n}\rho^\epsilon(\rvv,y)d\rvv}).
\end{align}
Consider the limit behavior of $\rho^\epsilon$ as $\epsilon\to 0$, since $\rho_1^*$ is supported on $B_0(\frac R 2)$, we know that for every $|y|\leq \frac {3R}4$, $\limsup\limits_{\epsilon\to 0}\frac{\int_{B_0(R)^c}\rho^\epsilon(\rvv,y)d\rvv}{\int_{\rrr^n}\rho^\epsilon(\rvv,y)d\rvv}\leq \frac 1 2$.
To conclude, we have $K^0(y)\geq \frac \eta 2$ for all $|y|\leq \frac{3R}4$, no matter whether $y\in \big((l(\rvv)^T\rvv)\circ P\big)(\mathfrak{A})$ or not.

For $\rvv\in spt(\rho_1^*)$, the supporting set of $\rho_1^*$, since $|\rvv|\leq \frac R 2$ and we have enforced $|l(\rvv)|\leq 1$ on $B_0(R)$, we have $|l(\rvv)^T\rvv|\leq \frac R 2$. This indicates that $spt\Big(\big(l(\rvv)^T\rvv\big)_{\#}\rho_1^*\Big)\subset B_0(\frac R 2)$, the $\frac R 2$-ball for $y$, or more specifically, $\{y:|y|\leq \frac R 2\}$.

On the other hand, as discussed above, $\big(l(\rvv)^T\rvv\big)_{\#}\rho_1^*$ satisfies a second-order equation that is uniformly elliptic on $y\in B_0(\frac {3R}4)$. Since $\big(l(\rvv)^T\rvv\big)_{\#}\rho_1^*(y)=0$ for $y\in (0.6R, 0.7R)$, by \Cref{lemma:ucp}, we know that $\big(l(\rvv)^T\rvv\big)_{\#}\rho_1^*(y)\equiv 0$ on $B_0(\frac{3R}4)$. However, since $\big(l(\rvv)^T\rvv\big)_{\#}\rho_1^*$ is a distribution, $\int_{-\frac R2}^{\frac R2}\big(l(\rvv)^T\rvv\big)_{\#}\rho_1^*(y)dy=1$. Contradiction.

This completes the proof.
\end{proof}

\subsubsection{Stochastic Closure with Latent Randomness}

\begin{thm}\label{apdx-thm-random-2}
    Following the notations in \cref{sde-2}, assume that $\sigma$ is a constant non-degenerate matrix independent of $\rvz$.
    \begin{enumerate}[label=(\roman*)]
        \item  If the Jacobian of $g$ is everywhere non-degenerate, then $\rho_1^*$ is not the limit distribution of $\rvv$ following dynamics \cref{sde-2}.
        \item     If $g$ is an element-wise mapping, i.e. there exists $g_0:\rrr\to\rrr$ such that $g(\rvz)_i=g_0(\rvz_i)$, and for any $x\in\rrr$, there always exists $k\in\nnn_{\geq 1}$ such that $g_0^{(k)}(x)\neq 0$, then $\rho_1^*$ is not the limit distribution of $\rvv$ following dynamics \cref{sde-2}.
    \end{enumerate}
\end{thm}
\begin{rmk}
    \begin{enumerate}[label=(\roman*)]
        \item  In the most popular case where $g$ is identity map, the Jacobian of $g$ is identity matrix, which is non-degenerate.
        \item    In the setting for the second claim, if we further know that \( g \) (or \( g_0 \)) is locally analytical, failure of this condition implies the existence of regions where \( g_0 \) is constant, which means that the parameters of the stochastic closure model do not play any role in these regions.
    \end{enumerate}
\end{rmk}

\paragraph{Preliminary of the Proof}
We start by reminding readers of classical results by Hörmander~\cite{hormander1967hypoelliptic} that will play a vital role in the proof. \cite{blessing2025underdamped} also provides some related references.

\begin{defn}[Hörmander Condition]
    A system of vector fields $\{X_0,\ X_1,...,\ X_m\}$ on a smooth manifold $M$ is said to satisfy Hörmander condition if for any $x\in M$, the Lie algebra generated by $\{X_0,\ X_1,...,\ X_m\}$ spans the entire tangent space at $x$,
    \begin{equation}
    \mathrm{Lie}(X_0,...X_m):=
        \mathrm{span}\{X_i(x),\ [X_i(x),X_j(x)],\ \big[X_i(x),[X_j(x),X_k(x)]\big],...\}=T_xM,\ \ \forall x\in M,
    \end{equation}
    where $[\cdot,\ \cdot]$ is the Lie bracket.
\end{defn}

\begin{lemma}[Hörmander Theorem for Stochastic Differential Equations (SDE)]\label{hormander-thm}
    Consider an $n$-dimensional stochastic differential equation
    \begin{equation}\label{hormander_sde}
        d\rvx_t=X_0(\rvx_t)dt+\sum_{i=1}^m X_i(\rvx_t)\circ dW_t^i,
    \end{equation}
    where $\rvx\in\rrr^n$, $X_0,...X_m$ are vector fields on $\rrr^n$, $W^i$ represents $m$ i.i.d. 1D Brownian motion, and $\circ$ denotes the Stratonovich integral. If $\{X_0,...X_m\}$ satisfies Hörmander condition, then for any $t>0$ and initial condition $x\in \rrr^n$, the transition probability $p(y;t,x)$ of $\rvx_t$ starting from $x$ has a smooth density with respect to the Lebesgue measure on $\rrr^n$, and this density is strictly positive on $\rrr^n$.

\end{lemma}

\paragraph{Intuition of the Proof}
For both claims in \Cref{apdx-thm-random-2}, we prove by checking Hörmander condition of the corresponding SDE. By Hörmander theorem, if the condition holds, the support set of the limit distribution will be $\rrr^n\times\rrr^n$ (for $(\rvv,\rvz)$). The fact that $\rho_1^*(\rvv)$ is compactly supported (on the filtered attractor) will lead to a contradiction.

\paragraph{Formal Proof}
The SDE for the dynamics with closure model is (Ito integral)
\begin{proof}
\begin{equation}\label{eqn98}
    d
\begin{pmatrix}
\rvv_t\\
\rvz_t
\end{pmatrix}
=
\begin{pmatrix}
b(\rvv_t)+g(\rvz_t) \\
h(\rvv_t,\rvz_t)
\end{pmatrix}
dt
+
\begin{pmatrix}
0\\
\sigma
\end{pmatrix}
dW,
\end{equation}
where $\rvv,\ \rvz$ are $n-$dimensional vectors, and $\sigma$ is an $n\times n$ constant matrix. Since $\sigma$ is independent of $\rvv$ and $\rvz$, the Stratonovich formulation of this dynamics follows the same form as above. Write \cref{eqn98} in the form of \cref{hormander_sde} as
    \begin{equation}
        d\rvx_t=X_0(\rvx_t)dt+\sum_{i=1}^n X_i(\rvx_t)\circ dW_t^i,
    \end{equation}
where $\rvx$ is the concatenated variable of $(\rvv,\rvz)$ and $X_i\ (i=0,1,...n)$ are $2n$-dimensional vector fields with
\begin{align}
    X_0(\rvx)&=\sum_{i=1}^n (b_i(\rvv)+g_i(\rvz))\dv{i}+\sum_{i=1}^n h_i(\rvv,\rvz)\dz{i},\\
    X_k(\rvx)&=\sum_{i=1}^n \sigma_{ik}\dz{i},\ \ k=1,2,...n,
\end{align}
where $\dv{i}, \dz{i}$ are coordinate tangent vectors.

    \textbf{Proof for the first claim:}
We verify that these vector fields satisfy the Hörmander condition as follows.
Denote the Lie algebra generated by these vector fields $\{X_i:\ i=0,1,...n\}$ as $\lie$.

Since $[X_1,...X_n]=[\dz{1},...,\dz{n}]\sigma$ and $\sigma$ is invertible, we have $\dz{i}\in\lie$ for all $i\in[n]$. 
As a consequence, $\lie=\mathrm{Lie}(\tilde{X}_0,X_1,...X_n)$ where $ \tilde{X}_0(\rvx):=\sum_{i=1}^n (b_i(\rvv)+g_i(\rvz))\dv{i}$.
Next, we show that $\dv{i}\in\lie$.

Compute
\begin{align}
    [X_k,\tilde{X}_0]&=\sum_{i=1}^n \big(\sum_{j=1}^n \sigma_{jk}\dz{j}(b_i(\rvv)+g_i(\rvz)) \big) \dv{i}+(\mathrm{\dz{i}\ terms})\\
    &=\sum_{i=1}^n \big(\sum_{j=1}^n \sigma_{jk}\dz{j}g_i(z) \big)\dv{i},\ \ k=1,2,...n.
\end{align}
Thus 
\begin{equation}
    \bigg[ [X_1,\tilde{X}_0],...[X_n,\tilde{X}_0)] \bigg]=\bigg[\dv{1},...\dv{n}\bigg] (\nabla_{\rvv}g) \sigma,
\end{equation}
where $\nabla_{\rvv}g$ is the Jacobian of $g$ at $\rvz$, with $(\nabla_{\rvv}g)_{ij}=\dz{j}g_i(\rvz)$. Here $[\cdot]$ represents the matrix whose columns are the enclosed vectors.

Since both $\sigma$ and the Jacobian are invertible by assumption, we have $\dv{i}\in\lie$ for any $i\in[n]$. By \Cref{hormander-thm}, the transition probability of $(\rvv,\rvz)$ system governed by \cref{eqn98} possesses smooth and strictly positive transition probability density functions $p(\rvx';\rvx, t)$ for any $t>0$ and $\rvx\in\rrr^{2n}$. 

Denote by $\rho_\infty(\rvv,\rvz)$ (or $\rho_\infty(\rvx)$) the limit distribution of \cref{eqn98}, which is the solution of the corresponding stationary Fokker-Planck equation. For any $t>0$, we have
\begin{equation}
\rho_\infty(\rvx)=\int\rho_\infty(\rvx_0)p(\rvx;\rvx_0,t)d\rvx_0,\ \ \forall\rvx\in\rrr^{2n},
\end{equation}
where $p(\rvx;\rvx_0,t)$ is the transition probability from $\rvx_0$ to $\rvx$ in a time period of $t$ following the dynamics in \cref{eqn98}.
This comes from the fact that the limit distribution is invariant under evolution.
This suggests $\rho_\infty(\rvx)>0$ for any $\rvx\in\rrr^{2n}$. As a result, the limit distribution of $\rvv$, $\rho_{\infty,\rvv}(\rvv)$ which is the marginal distribution of $\rho_\infty(\rvx)$ is strictly positive for any $\rvv\in\rrr^n$. Therefore, $\rho_1^*\neq\rho_{\infty,\rvv}$ since $\rho_1^*$ is supported on the filtered attractor, a compact set in $\rrr^n$.

    \textbf{Proof for the second claim:}
Following exactly the same approach as in the proof for the first claim, we only need to verify the set of vector fields $\{{X}_0, X_1,..., X_n\}$ satisfies the Hörmander condition. It suffices to show that $\{\dv{1},...\dv{n}\}\subset \mathrm{Lie}(\tilde{X}_0,X_1,...X_n)$.

Compute
\begin{align}
    [X_k,\tilde{X}_0]&=\sum_{i=1}^n \big(\sum_{j=1}^n \sigma_{jk}\dz{j}(b_i(\rvv)+g_0(\rvz_i)) \big) \dv{i}+(\mathrm{\dz{i}\ terms})\\
    &=\sum_{i=1}^n \big(\sum_{j=1}^n \sigma_{jk}\dz{j}g_0(\rvz_i) \big)\dv{i}\\
    &=\sum_{i=1}^n\sigma_{ik}g_0^{(1)}(\rvz_i)\dv{i}    ,\ \ k=1,2,...n,
\end{align}
where we use $g^{(l)}$ to denote the $l-$th order derivative of $g_0$.

It is easy to further show that for any $k_1,k_2,...k_m\in[n]$,
\begin{equation}\label{hi-order-lie}
    [X_{k_1},[X_{k_2},[..., [X_{k_m}, \tilde{X}_0]] ]]=\sum_{i=1}^n\bigg(g_0^{(m)}(\rvz_i)\prod_{l=1}^{m}\sigma_{ik_l}\bigg)\dv{i}.
\end{equation}

For any $i_0\in[n]$, by the assumption of $g_0$, there exists $m_0\in\nnn_{\geq 1}$ such that $g_0^{(m_0)}(\rvz_{i_0})\neq 0$. Furthermore, since $\sigma$ is non-degenerate, there exists $k_0\in [n]$ such that $\sigma_{i_0k_0}\neq 0$. Choose $m=m_0$, $k_1=k_2=...=k_{m_0-1}=k_0$ in \cref{hi-order-lie} and for $k_{m_0}=k\in [n]$, denote 
    \begin{equation}
    Y_k:=[X_{k_1},[X_{k_2},[..., [X_{k}, \tilde{X}_0]] ]]=\sum_{i=1}^n\bigg(g_0^{(m_0)}(\rvz_i) \sigma_{ik_0}^{m_0-1}\sigma_{ik}\bigg)\dv{i}.
\end{equation}

Then we have
\begin{equation}
    [Y_1,...Y_n]=[\dv{1},...\dv{n}]\mathrm{diag}\big\{g_0^{(m_0)}(\rvz_i)\big\}_i \mathrm{diag}\big\{\sigma_{ik_0}^{m_0-1}\big\}_i\ \sigma.
\end{equation}

Denote $I:=\{i\ |\ g_0^{(m_0)}(\rvz_i)\sigma_{ik_0}\neq 0\}$. Clearly we have $I\neq \emptyset$ because $i_0\in I$. We have
\begin{equation}
    [Y_1,...Y_n]=[\dv{i}]_{i\in I}\mathrm{diag}\big\{g_0^{(m_0)}(\rvz_i)\big\}_{i\in I} \mathrm{diag}\big\{\sigma_{ik_0}^{m_0-1}\big\}_{i\in I}\ \sigma_{I,:},
\end{equation}
where $[\dv{i}]_{i\in I}$ represents the matrix with $\dv{i}$ at each column, and $\sigma_{I,:}$ is a $|I|\times n$ submatrix of $\sigma$.

Since $\sigma$ is invertible, its row vectors are linearly independent. Thus $\mathrm{rank}(\sigma_{I,:})=|I|$. Since\\
$\mathrm{diag}\big\{g_0^{(m_0)}(\rvz_i)\big\}_{i\in I} \mathrm{diag}\big\{\sigma_{ik_0}^{m_0-1}\big\}_{i\in I}$ is invertible by the construction of $I$, $\mathrm{rank}([Y_1,...Y_n])=|I|$, by which we conclude that $\{\dv{i}\ :\ i\in I\}\subset \mathrm{span}(Y_1,...Y_n)\subset\lie$. In particular, $i_0\in\lie$. This completes the proof.

\end{proof}

%%%%%%%%%%%%%%%%%%%%%%%%%%%%%%%

\section{Proof of the second claim of \Cref{thm_clos_all}}
\label{apdx: thm_liouv}
In this section, we prove the second claim of \Cref{thm_clos_all}, demonstrating that the amount of high-fidelity data required to train a closure model far exceeds what is needed to estimate long-term statistics. 

We will carry out our analysis in the equivalent $\ell^2$ space, as elaborated in \Cref{apdx-b-1}. We would like to first remind the readers of the notations and preliminaries we have discussed in previous sections and then provide the formal version along with its proof of our theoretical result.

\subsection{Notations}\label{thm-lower-notation}
Following \Cref{apdx-b-1}, we analyze the original PDE dynamics (\ref{eq:general-pde_apdx}) through its equivalent $\ell^2$ representation (\ref{z:ode}). $\rvv\in\rrr^n$ represents the components in the resolved space (filtered space $\mathcal{F}(\hhh)$), and $\rvw$ represents the unresolved components. More precisely, each function $u(x)\in\hhh$ equals $\sum_{i=1}^n\rvv_i\psi_i(x)+\sum_{i=n+1}^{\infty}\rvw_i \psi_i(x)$. The vector field $f_r=(f_1, ..., f_n)$ is the filtered-space component of the original dynamics and $f_u=(f_{n+1},...)$ represents the unresolved components.

We use $\mu^*$ or $\rho^*$ to denote the invariant measure in the original space $\hhh$. 
We denote as $\Omega$ the attractor, either in function space $\hhh$ or in the representation space $\ell^2$, and $\mathcal{F}(\Omega)$ the filtered attractor.
We have shown in \Cref{apdx_prop_c5} that $\rho_1^*$, the marginal distribution of $\rho^*$ with respect to $\rvv$, is the best approximation of $\mu^*$ one can obtain in the filtered space.

We have shown in \Cref{prop-optimal-closure} that the tractable filtered-space dynamics that can derive the $\rho_1^*$ is
\begin{equation}
    \frac{d\rvv}{dt}=\mathbb{E}_{\rvw\sim\rho^*(\rvw|\rvv)}[f_r(\rvv,\rvw)|\rvv]=f_r(\rvv,0)+\eee_{\rvw\sim \rho^*(\rvw|\rvv)}[f_r(\rvv,\rvw)-f_r(\rvv,0)|\rvv ].
\end{equation}
Therefore, the optimal closure model is a conditional expectation w.r.t. $\rho^*$,
\begin{equation}\label{apdx_optimal_clos}
    clos^*(\rvv)=\eee_{\rvw\sim \rho^*(\rvw|\rvv)}[f_r(\rvv,\rvw)-f_r(\rvv,0)|\rvv ],
\end{equation}
which is a mapping from $\rrr^n$ to $\rrr^n$, or $\mathcal{F}(\hhh)$ to $\mathcal{F}(\hhh)$, equivalently.

We would also like to remind the readers of learning-based closure models. In this section, we analyze the theoretical lower bound on the convergence rate for single-state closure models trained with an a priori loss function,
\begin{equation}
    J_{ap}(\theta;\mathfrak{D})=\frac {1} {|\mathfrak{D}|}\sum\limits_{i\in\mathfrak{D}}\|clos(\ovu_i;\theta)-(\mathcal{F}\mathcal{A}-\mathcal{A}\mathcal{F})u_i\|^2,
\end{equation}
where the training data $u_i$ come from snapshots of trajectories from fully-resolved simulations , i.e., $S(t)u_0$ for particular $t$. The equivalent form in the $\ell^2$ space is,
\begin{equation}\label{mthd1_apdx}
    J_{ap}(\theta;\mathfrak{D})=\frac {1} {|\mathfrak{D}|}\sum\limits_{i\in\mathfrak{D}}\|clos(\rvv^i;\theta)-\big( 
f_r(\rvv^i,\rvw^i)-f_r(\rvv^i,0) \big)\|^2,
\end{equation}

We expect that similar limitations (slow convergence rate for data requirement) apply to other variants of learning-based closure models, such as those trained with an a posteriori loss or incorporating history information. This is because the fundamental cause of the slow convergence—detailed below—lies in the inherent difficulty of approximating a multi-valued map.
A rigorous treatment of these broader cases is left for future work.

\subsection{Formal Statement and Proof}

The formal statement of the second claim of \Cref{thm_clos_all} is as follows.

\begin{thm}\label{apdx-thm-lower-bound}
Let $clos^*$ be the optimal closure model as defined in \cref{apdx_optimal_clos} and assume it to be a Lipschitz mapping in the filtered space. Assume the filtered attractor $\Omega^*=\mathcal{F}(\Omega)$ (in the $\ell^2$ space) is a $ d_0$-dimensional manifold. Suppose that a learning-based closure model is trained to minimize the loss function \cref{mthd1_apdx}, where the dataset $\mathfrak{D}$ consists of $N$ fully-resolved snapshots (functions) drawn i.i.d. from $\mu^*$. Denote the resulting closure model after training as $clos_{\theta^*}$ and assume that it is a mapping with Lipschitz constant $L<\infty$. Assume there exist $\hat{f} \in \underset{f \in \mathrm{Lip}(L)}{\arg\min} {J}(f;\mathfrak{D})$ and $\delta_0$ such that $\|\hat{f}-clos_{\theta^*}\|_{L^\infty(\Omega^*;\rvv)}\leq\delta_0$, then
\begin{equation}\label{lower_bound}
 \eee_{\mathfrak{D}\sim\mu^{* \otimes N}} \|clos_{\theta^*}-clos^*\|_{L^{\infty}(\Omega^*;\rvv)}=\Omega(N^{-\frac 1{d_0+2}})-\delta_0,   
\end{equation}
where the expectation is taken with regard to the randomness of the dataset, and $\Omega$ in the R.H.S. is the notation for asymptotic complexity lower bound.
\end{thm}

\begin{rmk}
 $d_0\approx \min\{|D'|, d^*\}$, where $|D'|$ is the number of spatial-grids in the coarse-grid simulation (alternatively, the degree of freedom of the filtered function space) and $d^*$ is the intrinsic dimension of the original attractor.
\end{rmk}
\begin{rmk}
    In practice, the expressive power of the neural network-parameterized function class is typically sufficient to ensure that $\delta_0 \approx 0$.
    Some existing works adopt highly restricted function classes for closure models, often based on numerical closure ansatzes that perform well on certain benchmark systems.
    In such cases, although the learned closure $clos_{\theta^*}$ may exhibit small approximation error, as $\delta_0$ might not necessarily be negligible, its success primarily reflects the effectiveness of the underlying ansatz rather than the learning capability of the model.
    The role of machine learning in these settings is thus marginal compared to the choice of ansatz itself.
    More importantly, the motivation for using machine learning in coarse graining is precisely to address systems where existing numerical closure methods fall short.
    Hence, ML-based closure models built upon already-successful numerical ansatzes offer limited practical value, as they fail to address the core challenges of coarse graining in complex systems.
\end{rmk}

\begin{rmk}\label{rmk-d-4}
    The bound in \cref{lower_bound} is given in $\ell^2$ representation space of variable $\rvv$. It can be naturally transformed back to a lower bound in the function space of $\ovu(x)$ or $v(x)$.
    \begin{equation}
      \eee_{\mathfrak{D}\sim\mu^{* \otimes N}} \sup_{v\in \mathcal{F}(\Omega)} \|clos_{\theta^*}(v)-clos^*(v)\|_{\kkk}=\Omega(N^{-\frac 1{d_0+2}})-\delta_0,     
    \end{equation}
    for any norm of functions $\|\cdot\|_\kkk$ as long as there exists $i\leq n$ such that $\|\psi_i\|_{\kkk^*}<\infty$, where $\kkk^*$ is the dual space of $\kkk$.

This is obvious from the fact that 
\begin{equation}\label{v117}
  \|clos_{\theta^*}(v)-clos^*(v)\|_{\kkk}\geq \frac 1 {\|\psi_i\|_{\kkk^*}} |\langle \psi_i, clos_{\theta^*}(v)-clos^*(v) \rangle |=\frac 1 {\|\psi_i\|_{\kkk^*}} |clos_{\theta^*}(\rvv)_i-clos^*(\rvv)_i|
\end{equation}
for any function $v$ in the reduced space with $\rvv$ being its $\ell^2$ representative. Check the end of the proof for more details.
\end{rmk}

\begin{proof}
We divide the proof into a few steps. First, we transform the problem to a $d_0$-dimensional local coordinate chart space to facilitate the subsequent analysis. All $C$ and $C_i$ in the proof below denote positive constants.

\textit{Step 1: Reduction to local coordinates.}

Since $\Omega^*$ is a $d_0$-dimensional manifold, for any $\rvv\in\Omega^*$, there exists a neighborhood $U\ni\rvv$ that is diffeomorphic to $[0,1]^{d_0}$ via a diffeomorphism $\phi: U\to [0,1]^{d_0}$. There exists $K<\infty$ such that $\max\{\text{Lip}(\phi),\text{Lip}(\phi^{-1})\}<K$.

Apparently, $\|clos_{\theta^*}-clos^*\|_{L^{\infty}(\Omega^*;\rvv)}\geq \|clos_{\theta^*}-clos^*\|_{L^{\infty}(U;\rvv)}$, thus we will mainly focus on the approximation lower bound for a region $U$. We denote the chart coordinate as $\xxx=(\xi_1,...\xi_{d_0})\in[0,1]^{d_0}$. 
Denote $g(\xxx,\rvw)$ as $f_r(\phi^{-1}(\xxx),\rvw)-f_r(\phi^{-1}(\xxx),0)$.
The closure models are denoted with $h: [0,1]^{d_0}\to \mathcal{F}(\hhh)$. To be specific, $h^*(\xxx):=clos^*(\phi^{-1}(\xxx))$, $h_{\theta^*}(\xxx):=clos_{\theta^*}(\phi^{-1}(\xxx))$, and $\hat h(\xxx):=\hat{f}(\phi^{-1}(\xxx))$, where $\hat{f}$ is the minimizer of the training loss function within the Lipschitz function class, as defined in \Cref{apdx-thm-lower-bound}. 

We further define 
\begin{align}
    q^*(\xxx,\rvw):&=(\phi,I)_{\#}\rho=\rho^*(\phi^{-1}(\xxx),\rvw)|det(\nabla\phi|_{\phi^{-1}(\xxx)})|^{-1}\\
    q_1^*(\xxx):&=\phi_{\#}\rho_1^*=\rho_1^*(\phi^{-1}(\xxx))|det(\nabla\phi|_{\phi^{-1}(\xxx)})|^{-1}
\end{align}
We remark here that $q^*$ and $q_1^*$ are not normalized probability density functions since we only define and evaluate them on a subset of the filtered attractor $\Omega^*$.

With these notations, the optimal closure model is characterized by $h^*(\xxx)=\eee_{\rvw\sim q^*(\rvw|\xxx)}[g(\xxx,\rvw)|\xxx]$ for $\xxx \in [0,1]^{d_0}$.

\textit{Step 2: Formulating the randomness of data samples.}

Recall that the training dataset of a neural-network parameterized closure model $clos_\theta$ is drawn i.i.d. from $\mu^*$. We denote the data set as $\mathfrak{D}=\{(\rvv_i,\rvw_i)\}_{i=1}^N$. Choose a subset $U\subset\Omega$ such that there exists positive constant $C_0$ satisfying
\begin{equation}
    q_1^*(\xxx)\geq C_0,\ \forall\xxx\in\phi(U)=[0,1]^{d_0},
\end{equation}
where $\phi$ is the diffeomorphism described in Step 1. Denote $C_1:=\max\limits_{\xxx\in[0,1]^{d_0}}q_1^*(\xxx)$.
Since $U=\phi^{-1}([0,1]^{d_0})$ is compact, there exist $0<C_2<C_3$ such that
\begin{equation}\label{eqn121}
    \text{Var}_{\rvw\sim q^*(\rvw|\xxx)}[g_i(\xxx,\rvw)]=\int |g_i(\xxx,\rvw)-h^*_i(\xxx)|^2q^*(\rvw|\xxx)d\rvw \in (C_2,C_3),\ \ \forall \xi\in[0,1]^{d_0}\  \text{and} \ i\in[n],
\end{equation}
where $n$ is the dimension of vector $f_r$.

Define event $A:=\{\exists (\rvv_i,\rvw_i)\in\mathfrak{D}\ s.t. \ \rvv_i\in U\}$. Since $U$ contains an open set in the submanifold topology of ${\Omega}^*$, we have $\mathbb{P}(A)>0$.

In the following discussion, we will only analyze the scalar functions $clos_i,\ h_i$ and $g_i$ for one index $i$. The derivation below will hold for all components of the vector fields. With a slightly abuse of notation and for simplicity, we will refer to these scalar functions $clos_i,\ h_i,\ g_i$ as $clos,\ h,\ g$, omitting the $i$-index in the following discussions, unless specifically mentioned.

\textit{Step 3: Partitioning the domain.}

Assume the Lipschitz constant of $h^*$ is $L_0$ and $\hat{h}$ corresponds to the minimizer of training loss among all functions in $\text{Lip}(L_1)$. Define $\tau=(\frac{C_2}{8(L_0+L_1)\sqrt{d_0C_1C_3}})^{\frac{2}{d_0+2}} N^{-\frac 1 {d_0+2}}$.

Partition the chart domain into cubes with length $\tau$, 
\[[0,1]^{d_0} = \biguplus_{i=1}^{\tau^{-d_0}} V_i,\]
where $\biguplus$ denotes disjoint union, and $V_i$ are the $\tau$-cubes. Define the following random variables
\begin{align}
    N_i:&=\#\{(\rvv,\rvw)\in\mathfrak{D}\big| \ \rvv\in U;\  \phi(\rvv)\in V_i\}.\quad i=1,2,...\tau^{-d_0};\\
    Z:&=\min\{N_i\big| N_i>0\}.
\end{align}

Under event $A$, there exists non-zero $N_i$. Therefore, $Z$ is well-defined.

To derive the error estimation of $clos_{\theta^*}$ or $h_{\theta^*}$, we start by analyzing the error of $\hat{h}$.
We have 
\begin{align}\label{eqq124}
    \eee\|\hat{f}-clos^*\|_{L^\infty(U,\rvv)}=&\eee\|\hat{h}-h^*\|_{L^{\infty}([0,1]^{d_0});\xxx)}\\
    \geq& \mathbb{P}(A)\eee\big[\|\hat{h}-h^*\|_{L^{\infty}([0,1]^{d_0});\xxx)}\big|A\big]\\
    =& \mathbb{P}(A)\eee_Z\Big[\eee\big[\|\hat{h}-h^*\|_{L^{\infty}([0,1]^{d_0});\xxx)}\big|Z,A\big]\Big].
\end{align}

To estimate the conditional expectation of error w.r.t $Z$, we start by analyzing the error in each small $\tau$-cube.

\textit{Step 4: Error estimation in each cube $V_i$.}

Without loss of generality, we will only carry out our analysis for the cube $V_1$, and the results naturally hold for all other cubes. We discuss under the condition $N_1\neq 0$. 

Assume the data points with filtered component falling in $\phi^{-1}(V_1)$ are $\{(\phi^{-1}(\xxx_i),\rvw_i)\}_{i=1}^{N_1}$.
Construct a probability distribution $p(\xxx,\rvw)$ with 
\begin{equation}\label{eqn127}
    p(\xxx,\rvw)\propto \bm{1}_{\xxx\in V_1}\cdot q_1^*(\xxx)q^*(\rvw|\xxx).
\end{equation}
Conditioned on $N_1\neq 0$, these $(\xxx_i,\rvw_i)$ can be viewed as drawn i.i.d. from $p(\xxx,\rvw)$. 

Define $M:=\frac 1 {N_1}\sum_{i=1}^{N_1} g(\xxx_i,\rvw_i)$ and $M^*:=\eee_{(\xxx,\rvw)\sim p(\xxx,\rvw)}[g(\xxx,\rvw)]$.

\begin{claim}
    $\|h^*(\xxx)-\hat{h}(\xxx)\|_{L^{\infty}(V_1)}\geq |M^*-M|-\tau\sqrt{d_0}(L_0+L_1).$
\end{claim}
\begin{proof}
We first show that for all $\xxx\in V_1$, $|h^*(\xxx)-M^*|\leq \tau\sqrt{d_0}L_0$. 
Note that \[M^*=\eee_{\xxx\sim p_1}\big[\eee_{\rvw\sim p(\rvw|\xxx)}[g(\xxx,\rvw)|\xxx] \big]=\eee_{\xxx\sim p_1} h^*(\xxx),\] where $p_1$ is the marginal distribution of $p$ w.r.t $\xxx$.
For any $\xxx^{\dagger}\in V_1$,
\begin{align}\label{eq128}
    |h^*(\xxx^{\dagger})-M^*|=|\eee_{\xxx\sim p_1}h^*(\xxx^\dagger)-h^*(\xxx)|
    \leq \eee_{\xxx\sim p_1}|h^*(\xxx^\dagger)-h^*(\xxx)|
    \leq \eee_{\xxx\sim p_1} L_0|\xxx^\dagger-\xxx|\leq L_0(\sqrt{d_0}\tau).
\end{align}

Next we fix an arbitrary $\epsilon>0$ and show that for all $\xxx\in V_1$, $|\hat{h}(\xxx)-M|\leq (1+\epsilon)\tau\sqrt{d_0}L_1$.

Argue by contradiction. Assume there exists a $\xxx^\dagger\in V_1$ such that $|\hat{h}(\xxx^\dagger)-M|> (1+\epsilon)\tau\sqrt{d_0}L_1$. WLOG, we assume $\hat h(\xxx^\dagger)>M$.
We consider a family of perturbed functions $\{clos_t\}_{t\in\rrr}$ and $h_t:=clos_t\big|_U \circ\phi^{-1}$ satisfying
\begin{enumerate}[label=(\roman*)]
\item $\texttt{Lip}(h_t)\leq L_1$
    \item $h_t(\xxx)=\hat h(\xxx)+t$ for $\xxx\in \phi(\mathfrak{D}\bigcap U)\bigcap V_1$;
    \item $clos_t(\rvv)=\hat{f}(\rvv)$ for $\rvv\in \mathfrak{D}\bigcap \phi^{-1}(V_1)^c$.
\end{enumerate}

Then we have
\begin{align}
    J(clos_t;\mathfrak{D})=&J_h(h_t;\mathfrak{D})+C\\
    =&\sum_{i=1}^{N_1} |h_t(\xxx_i)-g(\xxx_i,\rvw_i)|^2+C':=\mathcal{J}(t),
\end{align}
where $J_h$ corresponds to the loss function taking input functions defined on $[0,1]^d$ ; $C$ and $C'$ does not depend on $t$.

Compute
\begin{align}
    \frac{d\mathcal{J}}{dt}=&\sum_{i=1}^{N_1}\frac d {dt} |t+\hat h (\xxx_i)-g(\xxx_i,\rvw_i)|^2=2\sum_{i=1}^{N_1}[t+\hat h (\xxx_i)-g(\xxx_i,\rvw_i)].
\end{align}
\begin{align}
      \frac{d\mathcal{J}}{dt}\Bigg|_{t=0}
      =&
      2\sum_{i=1}^{N_1}[\hat h (\xxx_i)-g(\xxx_i,\rvw_i)]\\
      =& 2\sum_{i=1}^{N_1}[\hat h (\xxx_i)-M]+
      2\sum_{i=1}^{N_1}[M-g(\xxx_i,\rvw_i)]\\
       =& 2\sum_{i=1}^{N_1}[\hat h (\xxx_i)-M]\\
       \geq& 2\sum_{i=1}^{N_1}\big[ \big( \hat{h}(\xxx^\dagger)-L_1\tau\sqrt{d_0}\big) -M \big]>2\epsilon\sum_{i=1}^{N_1}\tau\sqrt{d_0}L_1>0,       
\end{align}
where the second line comes from the fact that $M$ is the mean of these $g(\xxx_i,\rvw_i)$, and the last line leverages the definition of Lipschitz continuity. This is a contradiction to the assumption that $\hat{f}=clos_0$ is the minimizer of the loss function. Therefore, we conclude that 
 for all $\xxx\in V_1$, $|\hat{h}(\xxx)-M|\leq (1+\epsilon)\tau\sqrt{d_0}L_1$. Taking $\epsilon\to 0$ yields $|\hat{h}(\xxx)-M|\leq \tau\sqrt{d_0}L_1$ and thus for any $\xxx\in V_1$,
 \begin{align}
     |h^*(\xxx)-\hat{h}(\xxx)|=&|\big(h^*(\xxx)-M^*\big)+(M-M^*)+\big(M-\hat{h}(\xxx)\big)|\\
     \geq & |M^*-M|-|h^*(\xxx)-M^*|-|M-\hat{h}(\xxx)|\\
     \geq& |M^*-M|-\tau\sqrt{d_0}(L_0+L_1).
 \end{align}

    This completes the proof of the claim.
\end{proof}

Now we come back to \textit{Step 4} in the main proof.

Based on the claim above, we have
\begin{align}
&\eee\big[\|h^*(\xxx)-\hat{h}(\xxx)\|^2_{L^{\infty}(V_1)}\big| N_1,\ \{N_1\neq 0\} \big]\\
\geq& \eee\big[\big(|M^*-M|-\tau\sqrt{d_0}(L_0+L_1)  \big)^2\big|  N_1,\ \{N_1\neq 0\}\big]  \\
\geq& \eee\big[ (M-M^*)^2\big|  N_1,\ \{N_1\neq 0\}\big]-2 \tau\sqrt{d_0}(L_0+L_1)\eee\big[ |M-M^*|\ \ \big|  N_1,\ \{N_1\neq 0\}\big]\\
=&\frac 1 {N_1} \text{Var}_{(\xxx,\rvw)\sim p}[g(\xxx,\rvw)]-2 \tau\sqrt{d_0}(L_0+L_1)\eee\big[ |M-M^*|\ \ \big|  N_1,\ \{N_1\neq 0\}\big].\label{eqn142}
\end{align}

For the first term, 
\begin{equation}\label{eq143}
    \text{Var}[g]=\eee_{\xxx\sim p_1}\big[\text{Var}[g(\xxx,\rvw)|\xxx]\big]+\text{Var}\big[\eee_{p(\rvw|\xxx)}[g|\xxx] \big],
\end{equation}
where $p_1$ is the marginal distribution of $p$ defined in (\ref{eqn127}). By the construction of $U$ (see \cref{eqn121}), $\eee_{\xxx\sim p_1}\big[\text{Var}[g(\xxx,\rvw)|\xxx]\big]\in(C_2,C_3)$. Thus the first term in (\ref{eqn142}) is lower bounded by $C_2$. For the second term, note that $\eee|M-M^*|\leq\sqrt{\eee|M-M^*|^2}=\sqrt{\frac{\text{Var}[g]}{N_1}}$. To give an lower bound of \cref{eqn142}, we only need to give an upper bound of the variance. The first term in \cref{eq143} is upper bounded by $C_3$. The second term is upper bounded by $\tau^2L^2_0d_0$ due to \cref{eq128}. Since $\tau=O(N^{-\frac 1 {d_0+2}})$, $\text{Var}[g]\leq C_3+o(1)$. Without loss of generality, since we are studying the asymptotic convergence rate, we carry out following analysis when $N$ is large enough to have $\text{Var}[g]<2C_3$.

To conclude, we have
\begin{equation}\label{eq144}
    \eee\big[\|h^*(\xxx)-\hat{h}(\xxx)\|^2_{L^{\infty}(V_1)}\big| N_1,\ \{N_1\neq 0\} \big]\geq \frac {C_2}{N_1}-2\tau(L_0+L_1)\sqrt{\frac{2d_0C_3}{N_1}}.
\end{equation}

Similar bound holds for other $V_i$ and $N_i$.

\textit{Step 5: Bounding the conditional expectation w.r.t. $Z$.}

Following our discussion in Step 3, it suffices to bound 
$\eee_Z\Big[\eee\big[\|\hat{h}-h^*\|_{L^{\infty}([0,1]^{d_0});\xxx)}\big|Z\big]\Big]$ under event $A$.

Let $i^*=\text{argmin}\{N_i| N_i\neq0 \}$. Under event $A$, the set is non-empty and we have $N_{i^*}=Z$. Since the estimation in \cref{eq144} holds for arbitrary $\tau$-cube, we have 
\begin{equation}
 \eee\big[\|\hat{h}-h^*\|^2_{L^{\infty}([0,1]^{d_0};\xxx)}\big|Z\big]\geq \max\{\frac {C_2}{Z}-2\tau(L_0+L_1)\sqrt{\frac{2d_0C_3}{Z}},\ 0\}.
\end{equation}

\textit{Step 6: Estimating the distribution of $Z$.}

First note that $N_i$ follows binomial $B(N,\alpha_i)$ with 
\begin{equation}
    \alpha_i=\mathbb{P}(\rvv\in \phi^{-1}(V_i))=\int_{\phi^{-1}(V_i)}\rho_1^*(\rvv)d\rvv=\int_{V_i}q_1^*(\xxx)d\xxx=\tau^{d_0}\dashint q_1^*(\xxx)d\xxx\leq C_1\tau^{d_0}:=\alpha_0.
\end{equation}

Define $\tilde{N}_i$ as i.i.d. samples from $B(N, \alpha_0)$.
% , and $\tilde{Z}:=\min \tilde{N}_i\bm{1}_{\tilde{N}_i>0}$.

Recall that Chernoff's inequality \cite{vershynin2018high} states that $\mathbb{P}(\tilde{N}_1\geq (1+\delta)\mu)\leq e^{-\frac{\delta^2\mu}{2+\delta}}$, where $\mu=\alpha_0N$ is the expectation of $\tilde{N}_1$. Thus, choosing $\delta=1$ yields
\begin{equation}
    \mathbb{P}(N_i\geq 2\alpha_0 N)\leq \mathbb{P}(\tilde{N}_i\geq 2\alpha_0 N)\leq e^{-\frac{\alpha_0 N}{3}}.
\end{equation}

We further have
\begin{equation}
    \mathbb{P}(Z\geq 2\alpha_0N)=\prod_{i=1}^{\tau^{-d_o}}\mathbb{P}(N_i\geq 2\alpha_0N)\leq (e^{-\frac 1 3 N C_1\tau^{d_0}})^{\tau^{-d_0}}=e^{-\frac 1 3 C_1 N}=o(1),
\end{equation}
which means $\mathbb{P}(Z\leq 2\alpha_0 N)=1-o(1)$.

\textit{Step 7: Deriving the bound.}

We are finally ready to derive the convergence lower bound.

Conditioned on event $A$,
\begin{align}
&\eee\|\hat{h}-h^*\|^2_{L^{\infty}([0,1]^{d_0};\xxx)}\\
     =&\eee_Z\eee\big[\|\hat{h}-h^*\|^2_{L^{\infty}([0,1]^{d_0};\xxx)}\big|Z\big]\\
     \geq& \eee_Z \max\{\frac {C_2}{Z}-2\tau(L_0+L_1)\sqrt{\frac{2d_0C_3}{Z}},\ 0\}\\
     \geq & \eee_Z\Big(\max\{\frac {C_2}{Z}-2\tau(L_0+L_1)\sqrt{\frac{2d_0C_3}{Z}},\ 0\}\Big)\cdot \bm1_{Z\leq2\alpha_0N}.\label{eq152}
\end{align}

Define
\begin{equation}
    F(x):= {C_2}x^2-2\tau(L_0+L_1)\sqrt{2d_0C_3}x,\ \ x\in[\frac 1 {\sqrt{2\alpha_o N}} ,\infty)
\end{equation}
then the expression including $Z$ above becomes $F(Z^{-\frac 1 2})$.

Note that 
\begin{align}
    F(\frac 1 {\sqrt{2\alpha_0 N}})
    =& \frac{C_2}{2C_1\tau^{d_0}N}-\frac{2\tau(L_0+L_1)\sqrt{\cancel{2}d_0C_3}}{\sqrt{\cancel{2}C_1\tau^{d_0}N}}   \\
    =& \frac{C_2}{2C_1}N^{-1} (\frac{C_2}{8(L_0+L_1)\sqrt{d_0C_1C_3}})^{-\frac{2d_0}{d_0+2}} N^{\frac {d_0} {d_0+2}}\\
    &- \frac{2(L_0+L_1)\sqrt{d_0C_3}}{\sqrt{C_1}}N^{-\frac 1 2}  \bigg((\frac{C_2}{8(L_0+L_1)\sqrt{d_0C_1C_3}})^{\frac{2}{d_0+2}} N^{-\frac 1 {d_0+2}} \bigg)^{1-\frac {d_0}2}\\
=&  2^{\frac{5d_0-2}{d_0+2}}C_1^{-\frac 2 {d_0+2}} C_2^{-\frac{d_0-2}{d_0+2}}\big((L_0+L_1)\sqrt{d_0C_3}\big)^{\frac{2d_0}{d_0+2}} N^{-\frac 2 {d_0+2}}\\
&-
2^{\frac{4d_0-4}{d_0+2}}C_1^{-\frac 2 {d_0+2}} C_2^{\frac{2-d_0}{d_0+2}}\big((L_0+L_1)\sqrt{d_0C_3}\big)^{\frac{2d_0}{d_0+2}}N^{-\frac 2 {d_0+2}}\\
    =& 2^{4-\frac{12}{d_0+2}}C_1^{-\frac 2 {d_0+2}} C_2^{-\frac{d_0-2}{d_0+2}}\big((L_0+L_1)\sqrt{d_0C_3}\big)^{\frac{2d_0}{d_0+2}} N^{-\frac 2 {d_0+2}}=\Omega(N^{-\frac 2 {d_0+2}}).\label{eqn159}
\end{align}

Since $F(x)$ is monotone increasing in the region $\{x>0; F(x)>0\}$, we have that for all $Z\leq 2\alpha_0 N$, which corresponds to $Z^{-\frac 1 2}\geq \frac 1 {\sqrt{2\alpha_0 N}}$, $F(Z^{-\frac{1}{2}})\geq F(\frac 1 {\sqrt{2\alpha_0 N}})$.
Therefore, we conclude from \cref{eq152} that $\eee\big[\|\hat{h}-h^*\|_{L^{\infty}([0,1]^{d_0};\xxx)}\big| A\big]=\Omega(N^{-\frac 1 {d_0+2}})$. By \cref{eqq124} we further derive $ \eee\|\hat{f}-clos^*\|_{L^\infty(\Omega^*,\rvv)}=\Omega(N^{-\frac 1 {d_0+2}})$. By triangular inequality $\|clos_{\theta^*}-clos^*\|\geq\|\hat{f}-clos^*\|-\|\hat f-clos_{\theta^*}\|$ we yield the bound (\ref{lower_bound}).

This completes the proof.    
\end{proof}

Note that the bound above (\ref{eqn159}) does not depend on the component index of the vector field $clos^*$, $\hat{f}$ and $clos_{\theta^*}$, together with \cref{v117} we yield the proof for \Cref{rmk-d-4}.

\section{Proof of \Cref{thm: pino}}
\label{apdx: thm_pino}

We first state the formal version of \Cref{thm: pino}.

\begin{thm}\label{thm: pino_apdx}
Suppose a neural operator \( \mathcal{G}_\theta \) is trained to approximate the system evolution over time \( h \) for an arbitrary $h>0$, and long-term statistics are estimated by iterating \( \mathcal{G}_\theta \) from coarse-grid initial condition, namely the estimation of invariant measure is $\hat{\mu}_{h,\theta}:=\lim\limits_{N\to\infty}\frac 1 N \sum\limits_{n=1}^N\delta_{\mathcal{G}_\theta^n v_0(x)},$ any $v_0(x)$ with $x\in D'$.
    For any $\epsilon>0$, there exists $\delta>0$ s.t. as long as $\|(\mathcal{G}_\theta u)(\cdot,h)-S(h)u\|_\mathcal{H}<\delta, \forall u\in\mathcal{H}$, we have $\mathcal{W}_\mathcal{H}(\hat{\mu}_{h,\theta},\rho_1^*)<\epsilon$, where $\rho_1^*$ is the best approximation of $\mu^*$ in the filtered space, as discussed in \Cref{apdx_prop_c5}.
\end{thm}

As a preliminary result, we show the following property of the dynamical systems under consideration.
\begin{lemma}
    For any $h>0$, any initial condition $u_0\in\hhh$, $\lim\limits_{N\to\infty}\frac 1 N\sum\limits_{n=1}^N \delta_{S(nh)u_0}=\mu^*.$
\end{lemma}
\begin{proof}
    Denote $G:=S(h)$. From \Cref{ass:mixing and compact}, we know that for any measurable set $A,\ B\subset\mathcal{H}$, $\lim\limits_{t\to\infty}\mu(A\cap S(t)(B))=\mu(A)\mu(B).$ In particular $\lim\limits_{n\to\infty}\mu(A\cap S(nh)(B))=\mu(A)\mu(B).$
    This suggests that the system defined by 
    \begin{equation}
        u_{n+1}=Gu_n
    \end{equation}
    is mixing.

    From \Cref{mix ergo}, this system is ergodic, thus having an invariant measure w.r.t $G$. Since $G_{\#}\mu^*=\mu^*$, due to the uniqueness of invariant measure, we derive the proof.
\end{proof}

In the following proof, we will use the notation $G:=S(h)$, which is the learning target (ground-truth operator), and $\hat{G}$ for the approximate operator we obtain after training. By the design of the neural operator, the input of $\hat{G}$ can be vectors in $\rrr^d$ for any dimensionality $d$, serving as various discretizations of a particular function from $\hhh$. Here we violate the concepts a little bit by denoting $\hat{G}$ as a mapping from $\hhh$ to $\hhh$ (approximating $S(h)$) instead of $\hhh$ to $\hhh\times[0,h]$ (approximating $u\to\{S(t)u\}_{t\leq h}$) as is in our algorithm in the main text. In practice, we only use the last element of the output sequence, corresponding to the prediction for $S(h)u$, for estimating statistics. To be specific, to estimate long-term statistics in coarse-grid systems with learned operator $\hat{G}$, we use as input a function in reduced space $v(x)\in\mathcal{F}(\hhh),\ x\in D'$ (equivalent to an $\rrr^{|D'|}$ vector consisting of the function values on the grids), and autoregressively compute $\hat{G}^{(n)}v,\ n\in\nnn$. The invariant measure is estimated by 
\begin{equation}\label{def est inv}
    \hat{\mu}_{D'}:=\lim\limits_{N\to\infty}\frac 1 N \sum_{n=1}^N \delta_{\hat{G}^{(n)}v}.
\end{equation}

We first remind readers of the following fact.
\begin{fact}\label{fact2}
    For any function $u_0\in\hhh$, let $\vec{u}:=(u_0(x_1),...u_0(x_n))^T,\ n=|D'|$, be the discretization of $u_0$ in the coarse-grid system.
    There exists $u\in\hhh$ (possibly different from $u_0$) such that 
    \begin{equation}
        \vec{u}=I_{D'}u',\quad \hat{G}(\vec{u})=I_{D'}\hat{G}u.
    \end{equation}
\end{fact}

Now, we prove our main result.
\begin{thm}[=\Cref{thm: pino_apdx}]\label{apdxthm: pino}
 Under the assumptions in \Cref{app:ass}, for any $h>0$ and any $\epsilon>0$, there exists $\delta>0$ such that, as long as
$\|\hat{G}u-Gu\|_\hhh<\delta,\ \forall u\in\hhh$, we will have 
$\mathcal{W}_\hhh(\hat{\mu}_{D'},\rho_1^*)<\epsilon$.
\end{thm}
\begin{rmk}  
    This theorem aims to examine whether a neural operator trained only for short-term prediction (and corresponding loss functions) can still yield accurate estimations of long-term statistics. It does not address the approximability of the semigroup itself—specifically, whether the condition $\|\hat{G}u - Gu\|_\hhh < \delta$ holds. Theoretical results and technical assumptions regarding the approximability of mappings in function space by neural operators are already well-established in the literature \cite{kovachki2023neural,lanthaler2023nonlocal}.  
\end{rmk}
\begin{proof}
We will first deal with dynamics in the original space $\hhh$. We have two dynamics, the exact one and the approximate one,
\begin{align}
    u^{n+1}&=Gu^n,\ u^0=u_0\in\hhh;\\
    \hat{u}^{n+1}&=\hat{G}\hat{u}^n,\ \hat{u}^0=u_0\in\hhh.
\end{align}
    Both dynamics will converge to an attractor, $\Omega$ and $\hat{\Omega}$, respectively. 

With \Cref{lemma:shadow}, we know that there exists
$\eta_0,\eta_1>0$ such that for any $\{u_n\}_{n\in\mathbb{N}}\subset\hhh$ satisfying\\
(i) $d(u_n,\Omega)<\eta_0$; (ii) $\|u_{n+1}-G(u_n)\|<\eta_1$ for all $n$, there exists $\tilde{u}\in \hhh$ such that\\
\begin{equation}
\|u_n-G^{(n)}\tilde{u}\|<\epsilon,\ \forall n\in\nnn.    
\end{equation}

From Therorem 1.2 in Chapter 1 of \cite{temam2001navier}, we know that there exists $\eta_2>0$ such that 
\begin{equation}
    \|Gu-\hat{G}u\|<\eta_2,\ \forall u\in\hhh \Rightarrow \ dist(\Omega,\hat{\Omega})<\frac {\eta_0}5.
\end{equation}
Now we choose $\delta$ as $\min\{\eta_1,\eta_2\}$ and define the approximate operator $\hat{G}$ as well as dynamics and attractor accordingly.

    We next choose $n_0\in\nnn$ such that $\sup_{n\geq n_0}dist(\hat{u}^{n},\hat{\Omega})<\frac {\eta_0}5$. This implies that $\sup_{n\geq n_0}dist(\hat{u}^{n},{\Omega})<\frac {3\eta_0}5$.
We apply \Cref{lemma:shadow} to obtain a function $\tilde{u}\in\hhh$ such that 
\begin{equation}
\|\hat{u}^n-G^{(n-n_0)}\tilde{u}\|<\epsilon,\ \forall n\geq n_0.
\end{equation}
    For any $N\in\nnn$, we define 
    \begin{align}
        {\mu}_N:&= \frac 1 N\sum_{n=0}^{N}\delta_{G^{(n)}\tilde{u}}\\
    \hat{\mu}_N:&= \frac 1 N\sum_{n=n_0}^{n_0+N}\delta_{\hat{u}^{n}}.
    \end{align}
By constructing the transport mapping $T: u^n\mapsto G^{(n-n_0)}\tilde{u},\ n_0\leq n\leq n_0+N$, we have that 
\begin{equation}
    \mathcal{W}_\hhh(\hat{\mu}_N,\mu_N)<\epsilon.
\end{equation}
Note that $\hat{\mu}_N\to\hat{\mu}_D$(estimated invariant measure with fine-grid simulations) as $N\to\infty$, we derive
\begin{equation}
    \mathcal{W}_\hhh(\hat{\mu}_D,\mu^*)\leq \epsilon.
\end{equation}
    Recall that $P$ is the orthonormal projection towards $\mathcal{F}(\hhh)$ and that $\|Pu-Pu'\|\leq\|u-u'\|$ for any $u,u'\in\hhh$.
    In light of \Cref{fact2}, we derive   $\mathcal{W}_\hhh(\hat{\mu}_{D'},\rho_1^*)\leq\epsilon$.
\end{proof}

\section{Experiment Setup and Data Generation}
\label{apdx: data}
\subsection{Kuramoto–Sivashinsky Equation}\label{KS_detail}
We consider the following one-dimensional KS equation for $u(x,t)$,
\begin{equation}
\partial_t u+u\partial_x u+\partial_{xx}u +\nu \partial_{xxxx}u=0,\ \ \quad (x,t) \in[0,L]\times\mathbb{R}_{+},    
\end{equation}
with periodic boundary conditions. The positive viscosity coefficient $\nu$ reflects the traceability of this equation. The smaller $\nu$ is, the more chaotic the system is. We study the case for $\nu=0.01,\ L=6\pi$.

FRS is conducted with exponential time difference 4-order Runge-Kutta (ETDRK4)\cite{kassam2005fourth} with 1024 uniform spatial grid and $10^{-4}$ time grid. The CGS is conducted with the same algorithm except with 128 uniform spatial grids and $10^{-3}$ timegrid. We choose $h=0.1$ for our model, i.e., the model learns to predict the dynamics over a time evolution of 0.1 time unit.

\paragraph{Dataset}
The training dataset for the neural operator consists of two parts, the CGS data and FRS data.
The CGS dataset contains 6000 snapshots from 100 CGS trajectories. Snapshots are collected from time $t=20+k$, $k=1,2....60$. The data appears as input-label pairs $(v(\cdot,t),v(\cdot, t+h))$, where $h=0.1$ for KS.
The FRS dataset contains 105 snapshots from 3 FRS trajectories. Snapshots are collected from $t=20+2k$, $k=1,2,...35$. The data appears as input-label pairs\\ $(u(\cdot,t),\{u(\cdot, t+\frac k 4h)| k=1,2,3,4\})$.

As for input functions of PDE loss, they come from adding Gaussian random noise to FRS data.

\paragraph{Estimating Statistics}
For all methods in this experiment, statistics are computed by averaging over $t\in[20,150]$ and 400 trajectories with random initialization.
The experiment results are in \Cref{apdx_visul_ks}.

% For our model, we choose $h=0.1$. The total amount of FRS training data is 105 snapshots coming from 3 trajectories.
% When we complete the training, the $L^2$ relative error of our model on the test set is $12\%$. 
% To make a fair comparison, other learning-based methods are restricted to the same amount of training data. This setting will be the same for NS.

%The burning time of this system is around $50$.

% and complete the training when $L^2$ relative error on the test set is smaller than 

\subsection{Navier-Stokes Equation}
\label{apdx:NS basic}
We consider two-dimensional Kolmogorov flow (a form of the Navier-Stokes equations) for a viscous incompressible fluid (fluid field) $\rvu(x,y,t)\in\mathbb{R}^2$,
\begin{equation}\label{ns-original}
    \partial_t \rvu=-(\rvu\cdot\nabla)\rvu-\nabla p+\nu\Delta\rvu+(\sin(4y),0)^T,\quad \nabla\cdot \rvu=0,\quad (x,y,t)\in [0,L]^2\times \mathbb{R}_{+},
\end{equation}
with periodic boundary conditions. 
In the experiment, we deal with the vorticity form of this equation. 
\begin{equation}\label{ns-vorticity0}
    \partial_t w=-\rvu\cdot\nabla w+\nu\Delta w+\nabla\times(\sin(4y),0)^T,
\end{equation}
where $w=\nabla\times\rvu$.
The Reynolds number ($Re$) is an important factor for Navier-Stokes equations. In general, the larger $Re$ is, the more chaotic the system is. It is defined as $\frac{\overline{u}l}{\nu}$, where $\nu$ is the viscosity appearing in the equation, $\overline{u}$ is the root mean square of velocity scales $|\rvu|$, and $l$ is the length scale of the domain. In the isotropic case we consider here, $l$ is the domain length $L$. We carry out experiments in the case with $L=2\pi$, and different $\nu$. The corresponding Reynolds numbers for our two experiments on Navier-Stokes are 100 and $1.6\times 10^4$.

% The positive coefficient $Re$ is the Reynolds number. The larger $Re$ is, the more chaotic the system is. We consider the case $Re=100,\ L=2\pi$.
For the $Re=100$ case, 
FRS is conducted with pseudo-spectral split-step \cite{chandler2013invariant} with $128*128$ uniform spatial grid and self-adaptive time grid. The CGS is conducted with the same algorithm except with $16*16$ uniform spatial grids. For our model, we choose $h=1$. 

For the $Re=1.6\times 10^4$ case, 
FRS is conducted with the same numerical algorithm with $512*512$ uniform spatial grid and self-adaptive time grid. The CGS is conducted with the same algorithm except with $48*48$ uniform spatial grids. For our model, we choose $h=0.5$. To gain a sense of how the fluid fields behave, we visualize snapshots from one FRS trajectory in \Cref{fig:apdx: dataset}.

\begin{figure}[t]
\centering
\begin{subfigure}{0.27\textwidth}
  \includegraphics[width=\linewidth]{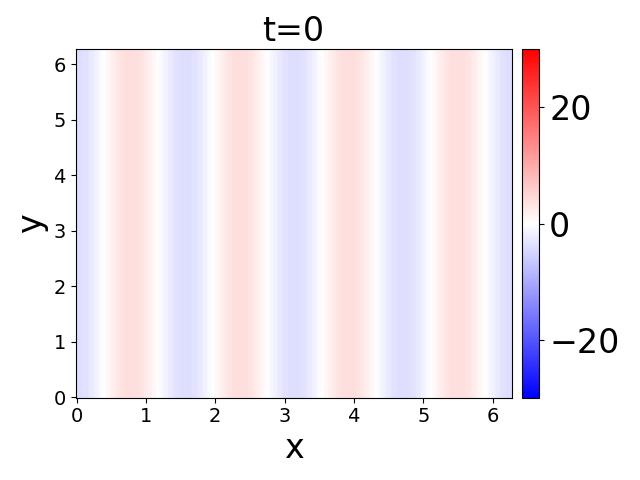}
\end{subfigure} % \hfil for filling the middle
\begin{subfigure}{0.27\textwidth}
  \includegraphics[width=\linewidth]{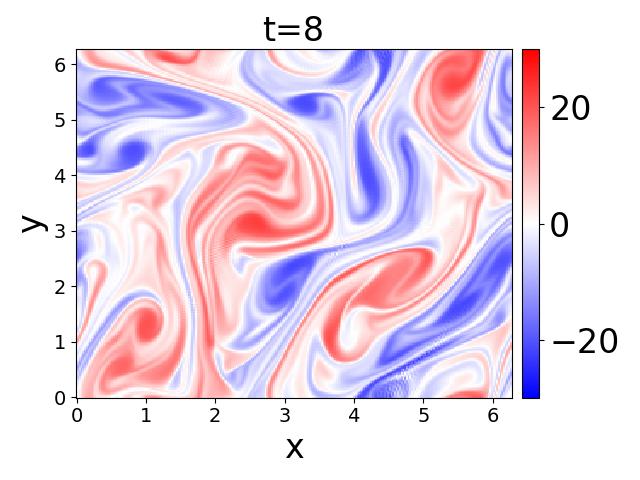}
\end{subfigure}
\begin{subfigure}{0.27\textwidth}
  \includegraphics[width=\linewidth]{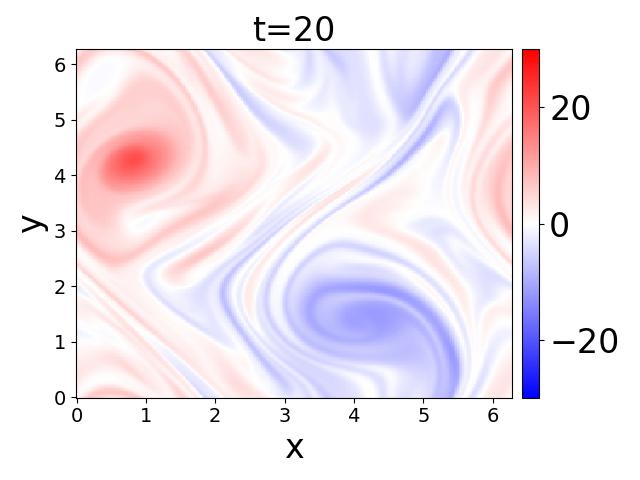}
\end{subfigure}
\begin{subfigure}{0.27\textwidth}
  \includegraphics[width=\linewidth]{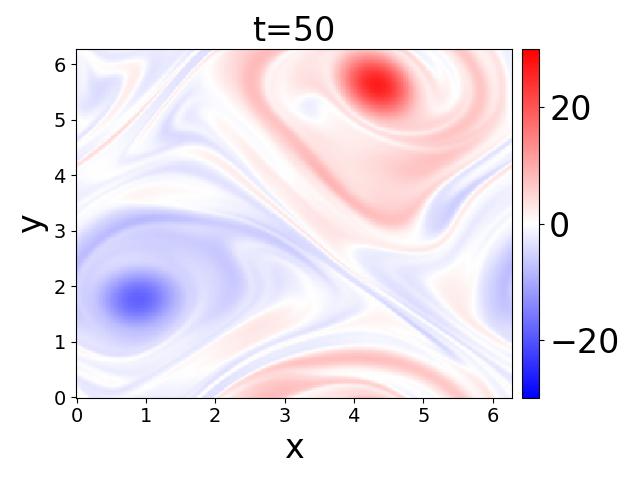}
\end{subfigure}
\begin{subfigure}{0.27\textwidth}
  \includegraphics[width=\linewidth]{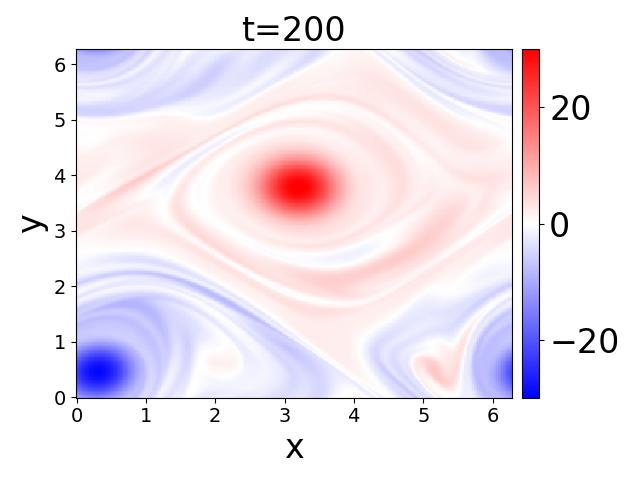}
\end{subfigure}
\begin{subfigure}{0.27\textwidth}
  \includegraphics[width=\linewidth]{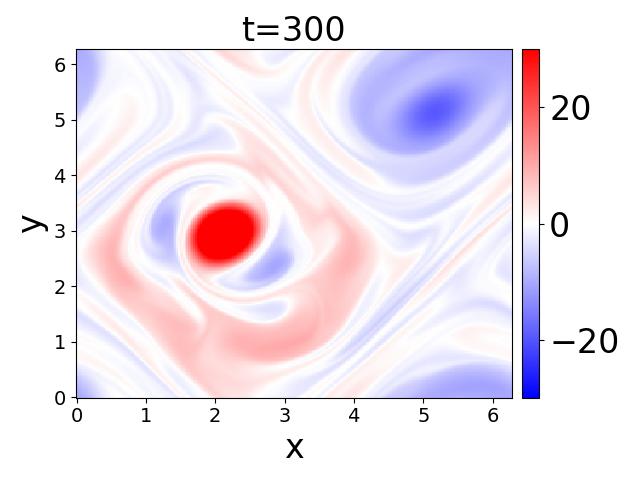}
\end{subfigure}
\caption{\textbf{Visualization of Dataset}
These figures are snapshots of one FRS trajectory for the Navier-Stokes equation with $Re=1.6\times 10^4$. From the plots, the fluid field appears to approach a dynamical equilibrium after $t=50$.
}
\label{fig:apdx: dataset}
\end{figure}

\paragraph{Dataset}
The training dataset for the neural operator consists of CGS data and FRS data.

For the $Re=100$ case, 
the CGS dataset contains 8000 snapshots from 80 CGS trajectories. Snapshots are collected from time $t=80+4k$, $k=1,2....100$. The data appears as input-label pairs $(u(\cdot,t),\{u(\cdot, t+\frac k {16}h)| k\in[16]\})$, where $h=1$ here.
The FRS dataset contains 110 snapshots from 1 FRS trajectories. Snapshots are collected from $t=50+3k$, $k=1,2,...110$. The data appears as input-label pairs $(u(\cdot,t),\{u(\cdot, t+\frac k {16}h)| k\in[16]\})$.

For the $Re=1.6\times 10^4$ case, 
the CGS dataset contains 6000 snapshots from 100 CGS trajectories. Snapshots are collected from time $t=100+3k$, $k=1,2....60$. The data appears as input-label pairs $(u(\cdot,t),\{u(\cdot, t+\frac k {16}h)| k\in[16]\})$, where $h=0.5$ here.
The FRS dataset contains 384 snapshots from 1 FRS trajectories. Snapshots are collected from $t=50+k$, $k=1,2,...384$ and downsampled to $256*256$ grids. The data appears as input-label pairs $(u(\cdot,t),\{u(\cdot, t+\frac k {32}h)| k\in[32]\})$.

As for input functions of PDE loss, they come from adding Gaussian random noise to FRS data.

\paragraph{Estimating Statistics}
For the $Re=100$ case, 
for all methods in this experiment, statistics are computed by averaging over $t\in[1800,3000]$ and 400 trajectories with random initialization.
The experiment results are in \Cref{apdx_visul_kf}.

For the $Re=1.6\times 10^4$ case, 
for all methods in this experiment, statistics are computed by averaging over $t\in[320,500]$ and 100 trajectories with random initialization.
The experiment results are in \Cref{apdx_visul_ns1w}.

\section{Implementation Details}
\label{apdx: implement}

\subsection{Hyperparameters Related to Dynamics}\label{apdx-choose-h}
In this subsection, we will provide details on how to choose parameters related to estimating statistics and the mapping being approximated with neural operators.

Recall that we train a neural operator to predict the evolution of a function within $t\in[0,h]$, where $h$ is a parameter. If $h$ is set to a small value, then the semigroup $S(h)$ is nearly an identity map. Although this is easy to approximate, it reduces the efficiency advantages of neural operator method in comparison to numerical simulation (both fully-resolved and coarse-grained simulation with closure models). In contrast, if $h$ is set too large, the corresponding $S(h)$ will possibly have a large Lipschitz constant in the setting of chaotic dynamics, posing challenges for learning this mapping. Given this tradeoff, in practice, we apply the following empirical manner to choose $h$ that is both easy to tackle and as large as possible. 

We compute the relative error of prediction based on the identity map, which corresponds to 
\begin{equation}
    \frac{\|u_0(x)-u(x,t)\|}{\|u(x,t)\|},
\end{equation}
for a randomly chosen $u_0$. An alternative is to compute the average relative error in a random batch of $u_0$. We then visualize it as a function of $t$. Intuitively, this function starts with 0 at $t=0$ and increases as $t>0$. We choose $h$ to be the moment when this relative error is around $ 80\%-100\%$.

To estimate long-term statistics, an important concept is the burning time $T_{burn}$, which characterizes the moment when a trajectory approaches the attractor close enough. 
We adopt the following empirical procedure to properly decide $T_{burn}$ for a general dynamics with no known results. We continue to run the simulation and monitor the energy spectrum based on the snapshots we have at the moment. Once the spectrum stabilizes and exhibits a persistent pattern, we infer that the trajectory has likely reached the attractor. In practical applications, one can usually estimate $T_{burn}$ based on theoretical scaling results \cite{temam2012infinite}.

\subsection{Physics-Informed Operator Learning}
%  \begin{algorithm}[htbp]
%     \caption{Multi-stage Physics-Informed Operator Learning}
%     \label{alg:main}
%     \hspace*{0.02in} \textbf{Input:} Neural operator $\mathcal{G}_{\theta}$; training data set $\mathfrak{D}_{c}$(CGS), $\mathfrak{D}_{f}$(FRS),\ $\mathfrak{D}_{p}$(randomly sampled).\\
%     % \hspace*{0.02in} \textbf{Output:} Learned operator  $\mathcal{G}_{\theta}$\\
%     \hspace*{0.02in} \textbf{Hyper-parameters:}
%     Training iterations $N_i(i=1,2,3)$. Weights combining two loss $\lambda_i(t)\\(i=1,2)$, which decay as $t$ increases. Parameters regarding optimizer.
%     % Number of total training iterations $M$; number of iterations and step size of inner loop $K,\eta$; weight for combining the two loss term $\lambda$
%     \begin{algorithmic}[1]
%     \For{$t=1,\cdots, N_1$}
%     \State Minimize $J(\theta;\mathfrak{D}_c)$
%     \EndFor
%     \For{$t=1,\cdots, N_2$}
%     \State Minimize $\lambda_1(t) J_{data}(\theta;\mathfrak{D}_c)+J_{data}(\theta;\mathfrak{D}_f)$
%     \EndFor
%     \For{$t=1,\cdots, N_3$}
%    \State Minimize $\lambda_2(t) J_{data}(\theta;\mathfrak{D}_f)+J_{pde}(\theta;\mathfrak{D}_p)$
%    \EndFor
%     \State \Return $\mathcal{G}_{\theta}$
%     \end{algorithmic}\label{alg:apdx}
% \end{algorithm}

Following the notations introduced in the main text, we formally summarize our algorithm as in $\Cref{alg:mainn}$. The loss functions in the algorithm are regression loss $J_{data}$ for fitting data and physics-informed loss $J_{pde}$.
For a training dataset $\mathfrak{D}$, 
\begin{equation}
    J_{data}(\theta;\mathfrak{D})=\frac{1}{|\mathfrak{D}|}\sum_{i\in\mathfrak{D}}
    \|\mathcal{G}_\theta u_i-S([0,h])u_i\|,
\end{equation}
where the norm $\|\cdot\|$ can be either $L^2$ norm (corresponding to MSE loss) or other norms, e.g. $H^1$.

Given a set of input functions, $\mathfrak{D}$ , the physics-informed loss is defined as
\begin{equation}
    J_{pde}(\theta;\mathfrak{D})=\frac{1}{|\mathfrak{D}|}\sum_{i\in\mathfrak{D}} \|(\partial_t-\mathcal{A})\mathcal{G}_\theta u_{0i}(x)\|_{L^2(\Omega\times[0,h])},
\end{equation}
where the initial values $u_{0i}$ in the loss function could be any fine-grid functions and do not have to come from FRS trajectories. $\Omega$ is the spatial domain of these functions. For equations with non-periodic boundaries, there will be an additional term in the loss function to fit boundary conditions, as is the general approach in physics-informed methods \cite{karniadakis2021physics}.

For input initial value $u_0$ (function restricted on the grid, which is a 1D tensor for KS and 2D tensor for NS), we repeat $u_0$ for $T$ times to make it a 2D tensor or 3D tensor $(u_0,u_0,...u_0)$, respectively, where $T$ is a hyperparameter. 
% Then we concatenate the position embedding of spatiotemporal coordinates with this tensor
For the implementations, the neural operator will learn to predict the mapping
\begin{equation}
(u_0,u_0,...u_0)\to \mathop{\bigoplus_{j=1}^{T}} S\left(\frac j T h\right)u_0,
\end{equation}
which is a discretization of $\{S(t)u_0\}_{t\in[0,h]}$. Before being passed to the neural operator model, the input is concatenated with its spatiotemporal position embedding, namely the feature at each grid point $(t,x,y)$ (take spatial 2D function as an example) is four-dimensional $(u(x,y,t),x,y,t)^T$.

\paragraph{KS Equation}
Following the architecture in the original FNO paper~\cite{li2020fourier}, our model is a 4-layer FNO with 32 hidden channels and 64 projection channels. The max modes used in the Fourier layers are 16 for temporal dimension and 64 for spatial dimensions.
We choose $h=0.1$, $T=64$. The data loss will only be computed for the time grid where there is label information.

We first train the model with CGS data, we use ADAM for optimization, with learning rate 5e-2, scheduler gamma 0.7 and scheduler stepsize 100. We train with batchsize 32 for 1000 epochs.

Then we train the model with CGS data and FRS data. $\lambda_1(0)=1$ and halves every $100$ epochs. We train with batchsize 32 for 250 epochs.

Finally, we train the model with PDE loss. We train with batch size 8 for 1487 epochs. Each batch contains 4 functions for computing the data loss and 4 functions for computing the PDE loss. $\lambda_2(t)$ decreases by 1.7 for every 500 epochs.

When we finish training, the $L^2$ relative error on the FRS test set is $\sim 12\%$.

\paragraph{NS Equation}

\textbf{(1) Low Reynolds number case.} 
Our model is a 4-layer FNO, with 32 hidden channels and 64 projection channels. 
The max modes used in the Fourier layers are 8 for temporal dimension and 16 for spatial dimensions.
We choose $h=1$, $T=32$. The data loss will only be computed for time grid where there is label information.

We first train the model with CGS data, we use ADAM for optimization, with learning rate 4e-3, scheduler gamma 0.6 and scheduler stepsize 50. We train with batch size 32 for 60 epochs.

Then we train the model with CGS data and FRS data. $\lambda_1(t)=\textbf{1}_{t\leq 20}$. We train with batch size 8 for 53 epochs.

Finally, we train the model with PDE loss. We train with batch size 16 for 1530 epochs. Each batch contains 8 functions for computing data loss and 8 functions for computing PDE loss. $\lambda_2(t)$ decreases by 1.8 for every 60 epochs.

When we finish training, the $L^2$ relative error on the FRS test set is $\sim 19\%$. The training takes $\sim 40$ minutes to complete.

\textbf{(2) High Reynolds number case.} 
Our model is a 4-layer FNO, with 28 hidden channels and 64 projection channels. 
The max modes used in the Fourier layers are 8 for temporal dimension and 48 for spatial dimensions.
We choose $h=0.5$, $T=32$. The data loss will only be computed for time grid where there is label information.
At the second and third stages of our algorithm, where the model receives high resolution inputs, we applied gradient accumulation and CPU offloading to reduce peak CUDA memory consumption during training.

We first train the model with CGS data, we use ADAM for optimization, with learning rate 4e-3, scheduler gamma 0.6 and scheduler stepsize 20. We train with batch size 32 for 42 epochs.

Then we train the model with CGS data and FRS data. $\lambda_1(t)=1$ for $t\leq 10$ and halves every $10$ epochs. We train with batch size 16 for 60 epochs.

Finally, we train the model with PDE loss. We train with batch size 54 for 1800 epochs. Each batch contains 6 functions for computing data loss and 48 functions for computing PDE loss. $\lambda_2(t)$ is initialized as 2 and decreases by 1.2 for every 60 epochs.

% For the high-Reynolds number case, we adopt gradient accumulation to fit in 

When we finish training, the $L^2$ relative error on the FRS test set is $\sim 14\%$. 
% The training takes $\sim 40$ minutes to complete.

\subsection{Baseline Method: Single State Closure Model}
For KS equation, the closure term is 
\begin{equation}
    \overline{u\partial_x u}-\ovu\partial_x\ovu=\frac 1 2 \partial_x (\overline{u^2}-\ovu^2).
\end{equation}

The closure model ansatz is $clos(\ovu;\theta)=\frac 1 2 \partial_x M_\theta(\ovu)$, where neural networks directly parameterize $M_\theta$ (with input and output both being coarse-grid functions). $M_\theta$ is trained to fit label $\overline{u^2}-\ovu^2$.

For NS equation \ref{ns-original}, the closure term is 
\begin{equation}
    \overline{(\rvu\cdot\nabla)\rvu}-(\overline{\rvu}\cdot\nabla)\overline{\rvu}   =\overline{\nabla\cdot(\rvu\times\rvu)}-\nabla\cdot(\overline{\rvu}\times\overline{\rvu})=\nabla\cdot (\overline{\rvu\times\rvu}-\overline{\rvu}\times\overline{\rvu}),
\end{equation}
where in the first equality we use the incompressible condition $\nabla\cdot \rvu=0$.

The closure model ansatz for \cref{ns-vorticity0} is $clos(\overline{w})=\nabla\times (\nabla\cdot M_\theta(\overline{\rvu}))$, where we use the relation $\rvu=\nabla^\perp(-\Delta)^{-1}w$, and neural networks directly parameterizd $M_\theta$ which is a mapping from $n\times n\times 2$ tensor (the coarse-grid velocity field $\overline{\rvu}$, $n$ being the resolution) to $n\times n\times 2\times 2$ tensor, $\overline{\rvu\times\rvu}-\overline{\rvu}\times\overline{\rvu}$.

%%%%%%%%%%%%%%%%%%%%%%%%

The network follows the Vision Transformer \cite{dosovitskiy2020image} architecture.
For KS equation, the input was partitioned into $1\times 4$ patches, with 2 transformer layers of 6 heads. The hidden dimension is 96 and the MLP dimension is 128. 
For both NS equations with two different Reynolds numbers we considered, the input was partitioned into $4\times 4$ patches, with 2 transformer layers of 6 heads. The hidden dimension is 96 and the MLP dimension is 128. 
For all experiments, we use AdamW optimizer \cite{loshchilov2017decoupled} with learning rate $1e-4$ and weight decay $1e-4$.

% \newpage
\section{More Experiment Results and Visualizations}
\label{apdx: exp_visual}

\subsection{Statistics}
We formally introduce the statistics we consider.

\paragraph{Total Variation for Invariant Measures}
As is mentioned in the main text, we propose to directly compare the estimated invariant measure resulting from the time average of simulations and that of ground truth.

Recall that we have expanded $u\in\hhh$ onto orthonormal basis $u=\sum_{i=1}^{\infty}z_i\psi_i$. In particular, for $v\in\mathcal{F}(\hhh)$, $v\in\operatorname{span}\{\psi_i: i\leq|D'|\}$. We compute the total variation (TV) distance for the (marginal) distribution of each $z_i$, where TV distance of two distributions (probability densities) $\nu,\ \mu$ is defined as
\begin{equation}
    d_{TV}(\mu,\nu)=\frac 1 2\int |\mu(x)-\nu(x)|dx.
\end{equation}
The value of TV distance ranges from 0 to 1, with a smaller value corresponding to more similarity between the two distributions.

% A convenient metric is 
For experiments we consider in this work, a natural choice of $\psi_i$ is the Fourier basis functions,
$\{e^{i\frac{2k\pi}Lx}\}_{k\in\mathbb{Z}}$ for 1D KS and $\{e^{i\frac{2\pi}L(kx+jy)}\}_{k,j\in\mathbb{Z}^2}$ for 2D NS.

With a little abuse of definition, the corresponding $z_i$ are complex numbers.
\cite{eckmann1985ergodic} shows that the limit distribution of $Arg z_i$ is uniform distribution on $[0,2\pi]$. Thus, it suffices to check the distribution of mode length $|z_i|$.

\paragraph{Other Statistics}
In the following, we use $\hat{u}_k$ to denote the $k$-th Fourier mode of $u$. When $u$ is a multi-variate function, $k$ is a tuple.
\begin{itemize}
    \item Energy Spectrum. 
    
    $O_e(u;k)=|\hat{u}_k|^2$ (1D),
$O_e(u;k_0)=\sum_{|k|_1=k_0}|\hat{u}_k|^2$ (general).

The $k_0$-th energy spectrum is $\ooo_e(k):=\mathbb{E}_{u\sim\mu^*}O_e(u;k_0)$.

\item Spatial Correlation. 

$O_s(u;h)=\int u(x)u(x+h)dx$.
The $h$ spatial correlation is $\ooo_s(h):=\mathbb{E}_{u\sim\mu^*}O_s(u;h)$.

\item Auto Correlation Coefficient.

Note that $\ooo_s$ is a function of $h$.

The $k$-th Auto Correlation Coefficient is
$\ooo_a(k):=|(\hat{\ooo_s})_k|^2$.

\item The distribution of Vorticity ($w(x)$ for NS) and Velocity ($u(x)$ for KS).

\item The variance of the function value.

\item The root mean square of the velocity field $\sqrt{\langle|\rvu|^2\rangle}$ (only for NS equation). Here $\langle,\rangle$ stands for the average over all time snapshots from all trajectories in the simulation, and the velocity $\rvu$ is derived from the vorticity $w$ through the relation $\rvu=\nabla^\perp(\Delta^{-1}w)$.

\item Dissipation Rate: $\frac 1 {Re}\dashint u(x)^2dx$, where $\dashint$ refers to averaged integral
\begin{equation}
\dashint_\Omega f(x)dx:=\frac{\int_\Omega f(x)dx}{\int_\Omega dx}.    
\end{equation}

In practice, we usually check the distribution of this quantity.

\item Kinetic Energy: $\dashint (u-\bar{u})^2dx$ where $\bar{u}(x):=\lim\limits_{T\to\infty} \frac 1 T \int_0^T u(x,t)dt$. In practice, we usually check the distribution of this quantity.

\end{itemize}

%%%%%%%%%%%%%%%
\newpage
\subsection{Experiment Results}
\subsubsection{Kuramoto–Sivashinsky Equation}
\label{apdx_visul_ks}
The error of all statistics we considered for KS equation is listed in \Cref{tab:apdx:ks tab} and plotted in \Cref{fig:apdx: ks stat}.

\begin{table}[htbp]
\centering
\setlength{\tabcolsep}{1.99pt}
\caption{\textbf{Error on Different Statistics: KS equation.}
From left to right: Average relative error on energy spectrum, max relative error on energy spectrum, average relative error on auto-correlation coefficient, max error on auto-correlation coefficient, total variation distance from (ground truth) velocity distribution, average component-wise TV distance(error), and max component-wise TV distance(error). The best performances are indicated in \textbf{bold}.}
\begin{tabular}{lccccccc}\toprule
Method & Avg. Eng.($\%$) & Max Eng.($\%$)  & Avg. Cor.($\%$)  & Max Cor.($\%$)  & Velocity & Avg. TV & Max TV \\\midrule
CGS (No closure) & 12.5169  & 77.8223  & 13.1275  & 80.5793  & 0.0282 & 0.0398 & 0.2097 \\
Eddy-Viscosity  & 7.6400  & 48.3684  & 8.7583  & 56.5878  & \textbf{0.0276} & 0.0282 & 0.1462 \\
Single-state  & 12.5323  & 78.6410  & 13.1052  & 81.2461  & 0.0280 & 0.0410 & 0.2111 \\\midrule
\textbf{Our Method} & \textbf{7.4776 } & \textbf{20.4176 } & \textbf{7.8706 } & \textbf{22.7046 } & 0.0284 & \textbf{0.0272} & \textbf{0.0849}
\\\bottomrule
\end{tabular}
\label{tab:apdx:ks tab}
\end{table}
% eddy-vis:\cite{matharu2020optimal}
% single-state:\cite{guan2022stable}

\begin{figure}[htbp]
\centering

\begin{subfigure}{0.366\textwidth}
  \includegraphics[width=\linewidth]{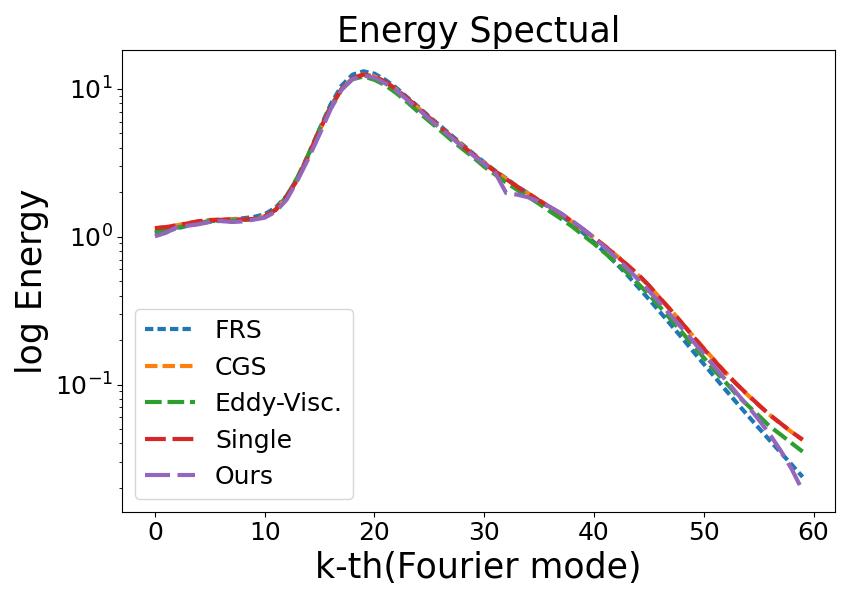}
  \caption{Energy Spectrum}%\label{fig:sub1}
\end{subfigure}\hfil % \hfil for filling the middle
\begin{subfigure}{0.366\textwidth}
  \includegraphics[width=\linewidth]{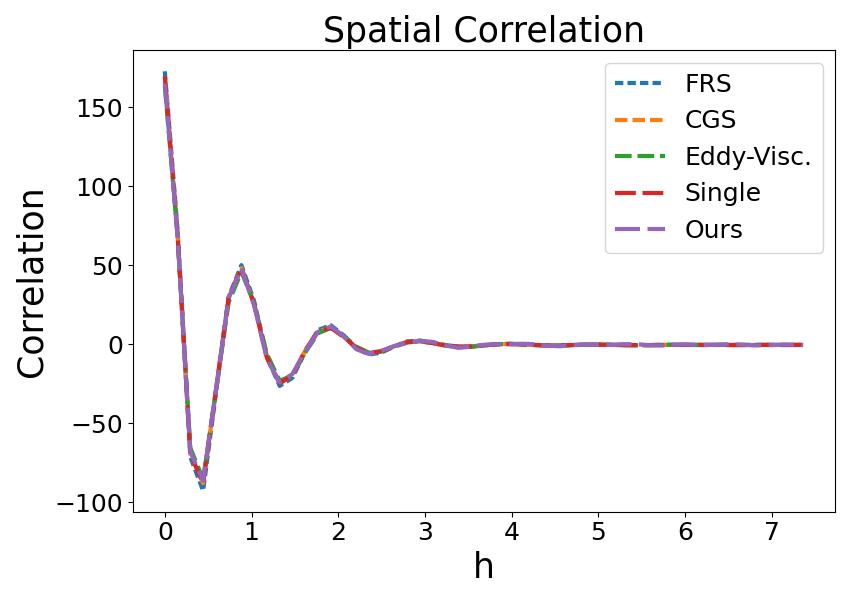}
  \caption{Spatial Correlation}%\label{fig:sub2}
\end{subfigure}\hfil
\begin{subfigure}{0.366\textwidth}
  \includegraphics[width=\linewidth]{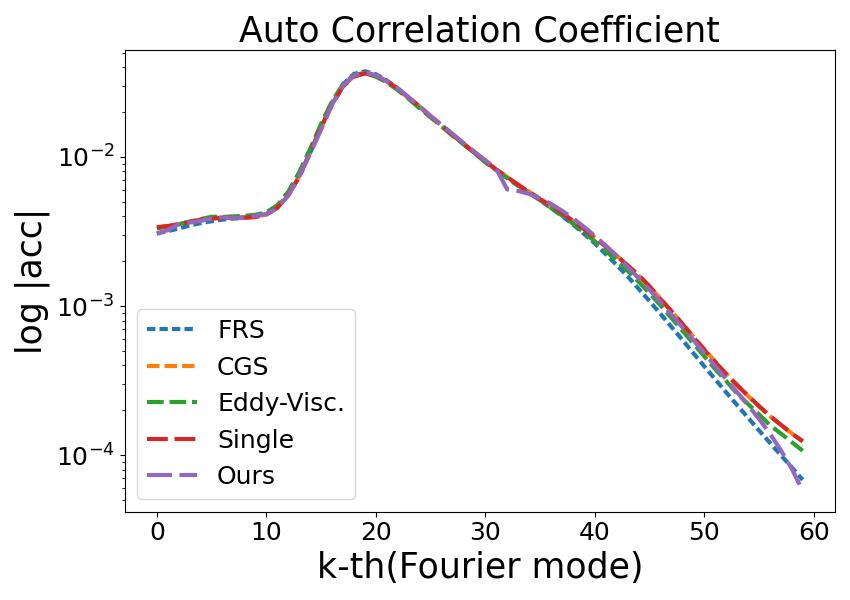}
  \caption{Auto Correlation Coefficient}%\label{fig:sub5}
\end{subfigure}
\caption{\textbf{Experiment Results for KS Equation}
'FRS'(blue curves) refers to fully-resolved simulation, and serves as ground truth. `CGS': coarse-grid simulation(no closure model). `Eddy-Visc.': classical eddy-viscosity model. `Single': learning-based single-state closure model. Our method (purple) is closest to ground truth among all coarse-grid methods.}
\label{fig:apdx: ks stat}
\end{figure}

\newpage
\subsubsection{Navier-Stokes Equation with Reynolds number 100}
\label{apdx_visul_kf}

The error of all statistics we considered for NS equation ($Re=100$) is listed in \Cref{tab: ns:apdx} and plotted in \Cref{fig:apdx: ns stat}.

The visualization of Total variation error for each (marginal) distribution is shown in \Cref{fig:apdx: TV}.
In the visualization, the $(k,j)$-element represents the TV error (compared with the corresponding distributions obtained from FRS) regarding the distribution of the mode length of the component for $(k,j)$ Fourier basis $e^{i\frac{2\pi}{L}(kx+jy)}$.

Our model achieves the best total variation error among all coarse-grid methods, and is much more efficient, namely 124x faster than ground-truth FRS. Our model is even faster than CGS without any closure model because it can perform $O(1)$ time steps instead of tiny time grids during simulations. Learning-based closure model is 58x slower than ours since it needs to invoke a neural network model for every single time step. 
Furthermore, the results of other practical statistics suggest that TV error is a reasonable and fundamental metric. 

\begin{table}[htbp]
\centering
\setlength{\tabcolsep}{1.99pt}
\caption{\textbf{Error on Different Statistics: NS equation, $Re=100$.} From left to right: Average relative error on energy spectrum, max relative error on energy spectrum, total variation distance from (ground truth) vorticity distribution, average component-wise TV distance(error), max component-wise TV distance(error), and relative error on the variance of vorticity.}
\begin{tabular}{lcccccc}\toprule
Method & Avg. Eng.($\%$) & Max Eng.($\%$) & Vorticity & Avg. TV & Max TV & Variance($\%$) \\\midrule
CGS (No closure) & 178.4651  & 404.9923  & 0.1512 & 0.4914 & 0.8367 & 253.4234  \\
Smagorinsky & 52.9511  & 120.0723  & 0.0483 & 0.2423 & 0.9195 & 20.1740  \\
Single-state & 205.3709  & 487.3957  & 0.1648 & 0.5137 & 0.8490 & 298.2027  \\
Dynamic Smag. (DSM) & 74.2150  & 139.5019  & 0.0821 & 0.2803 & 0.8629 & 73.6158  \\
History-aware (RNN) & 181.3914  & 415.1769  & 0.1522 & 0.4938 & 0.8345 & 256.7264  \\
Stochastic (DM) & 182.4392  & 416.8283  & 0.1535 & 0.4930 & 0.8206 & 263.0987  \\
Multi-Fidelity FNO(MFFNO) & 20.7055  & 35.5767  & 0.0115 & 0.2123 & 0.3765 & 20.4410  \\\midrule
\textbf{MF-PINO} & \textbf{5.3276 } & \textbf{8.9188 } & \textbf{0.0091} & \textbf{0.0726} & \textbf{0.2572} & \textbf{2.8666 } \\\bottomrule
\end{tabular}

\label{tab: ns:apdx}
\end{table}
% smag:\cite{smagorinsky1963general} 
% single state:  \cite{guan2022stable} 

\begin{figure}[t]
\centering

\begin{subfigure}{0.249\textwidth}
  \includegraphics[width=\linewidth]{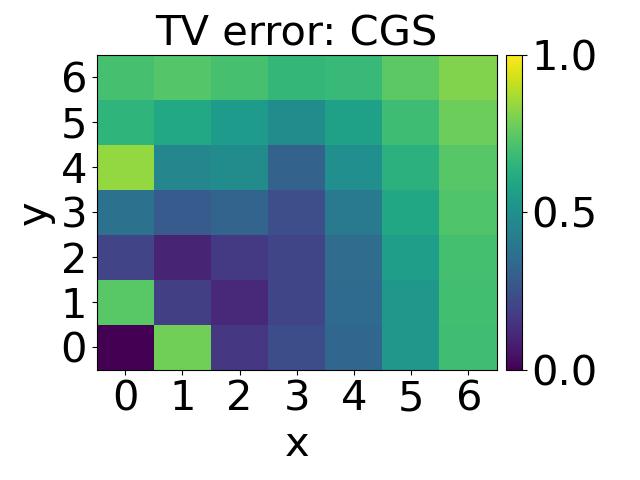}
\end{subfigure}\hfil % \hfil for filling the middle
\begin{subfigure}{0.249\textwidth}
  \includegraphics[width=\linewidth]{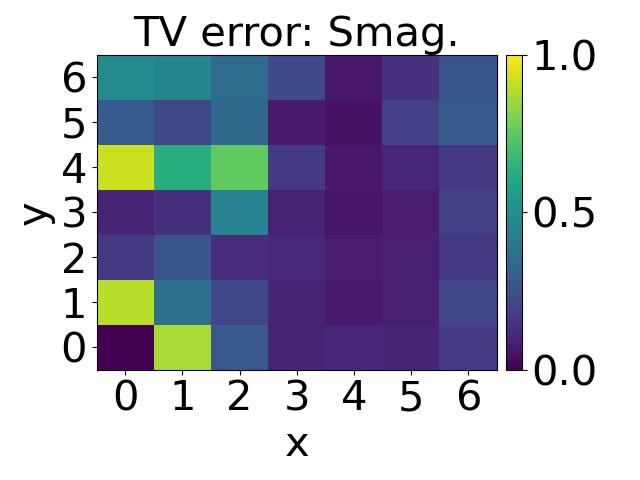}
\end{subfigure}\hfil
\begin{subfigure}{0.249\textwidth}
  \includegraphics[width=\linewidth]{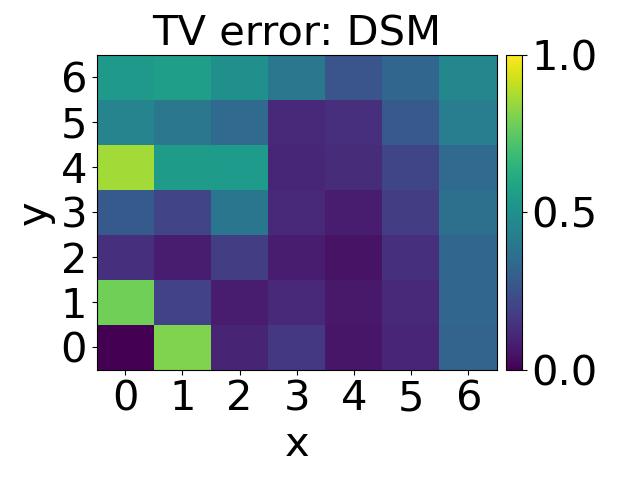}
\end{subfigure}\hfill
\begin{subfigure}{0.249\textwidth}
  \includegraphics[width=\linewidth]{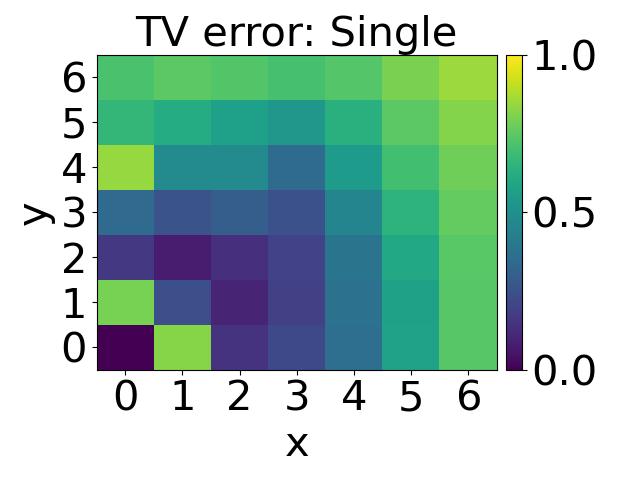}
\end{subfigure}\hfill
\begin{subfigure}{0.249\textwidth}
  \includegraphics[width=\linewidth]{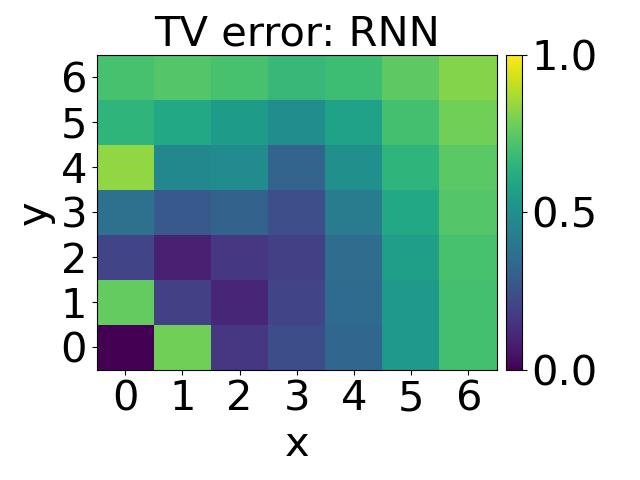}
\end{subfigure}\hfill
\begin{subfigure}{0.249\textwidth}
  \includegraphics[width=\linewidth]{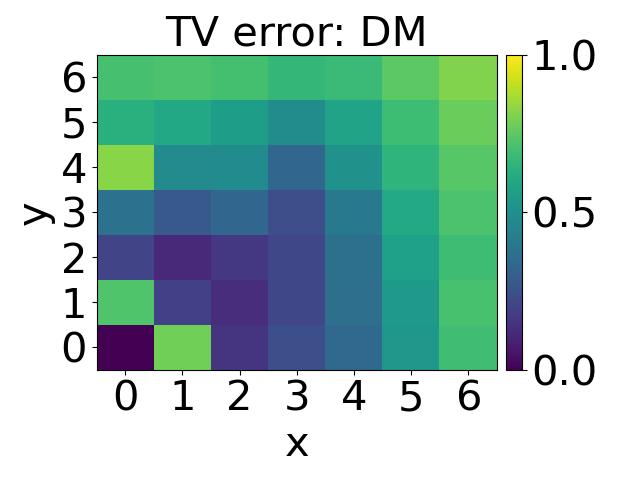}
\end{subfigure}\hfill
\begin{subfigure}{0.249\textwidth}
  \includegraphics[width=\linewidth]{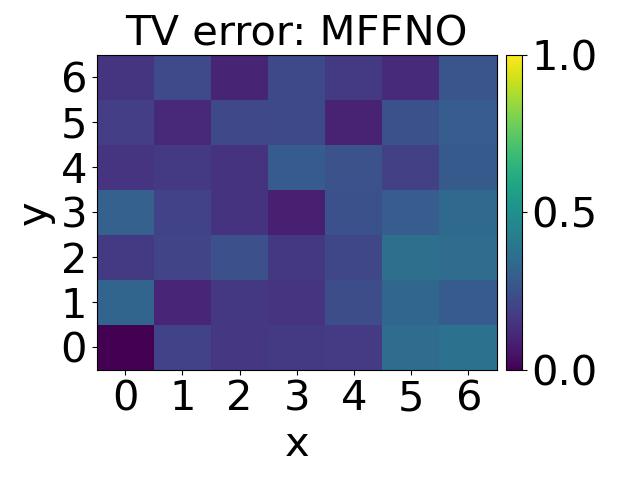}
\end{subfigure}\hfill
\begin{subfigure}{0.249\textwidth}
  \includegraphics[width=\linewidth]{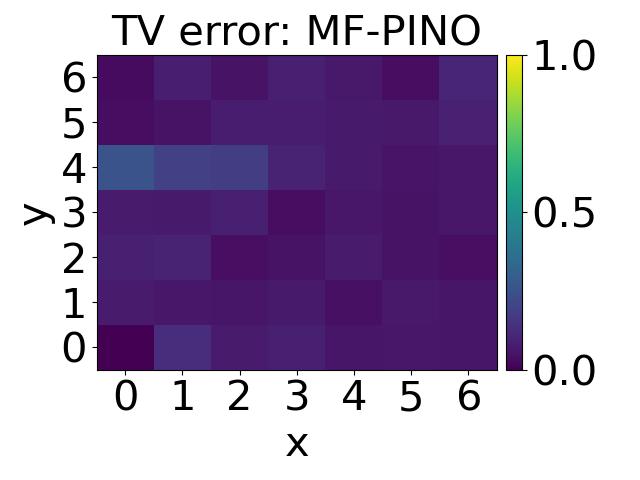}
\end{subfigure}
\caption{\textbf{Total Variation (TV) error for NS Equation with $Re=100$}
The $(k,j)$-element represents the TV error (compared with the results obtained from FRS) regarding the distribution of the mode length of the component for $(k,j)$ Fourier basis $e^{i\frac{2\pi}{L}(kx+jy)}$.
 `CGS': coarse-grid simulation(no closure model). `Smag.': classical Smagorinsky model. `DSM': dynamical Smagorinsky model. `Single': learning-based single-state closure model. `RNN': history-aware closure model with RNN. `DM': stochastic closure with diffusion model.
 `MFFNO': multi-fidelity FNO.
Our primary method with MF-PINO achieves the smallest TV error among all modes, indicating that it obtains the best approximation of the filtered invariant measure among all coarse-grid methods considered.
}
\label{fig:apdx: TV}
\end{figure}

\begin{figure}[htbp]
\centering

\begin{subfigure}{0.499\textwidth}
  \includegraphics[width=\linewidth]{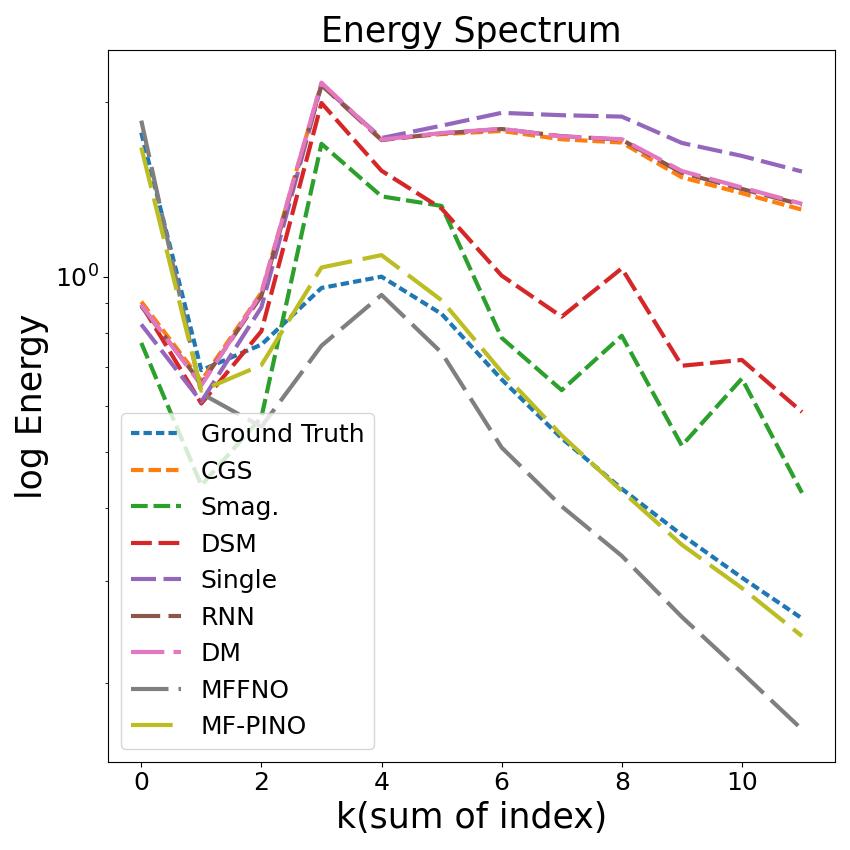}
  \caption{Energy Spectrum}%\label{fig:sub1}
\end{subfigure}\hfil % \hfil for filling the middle
\begin{subfigure}{0.499\textwidth}
  \includegraphics[width=\linewidth]{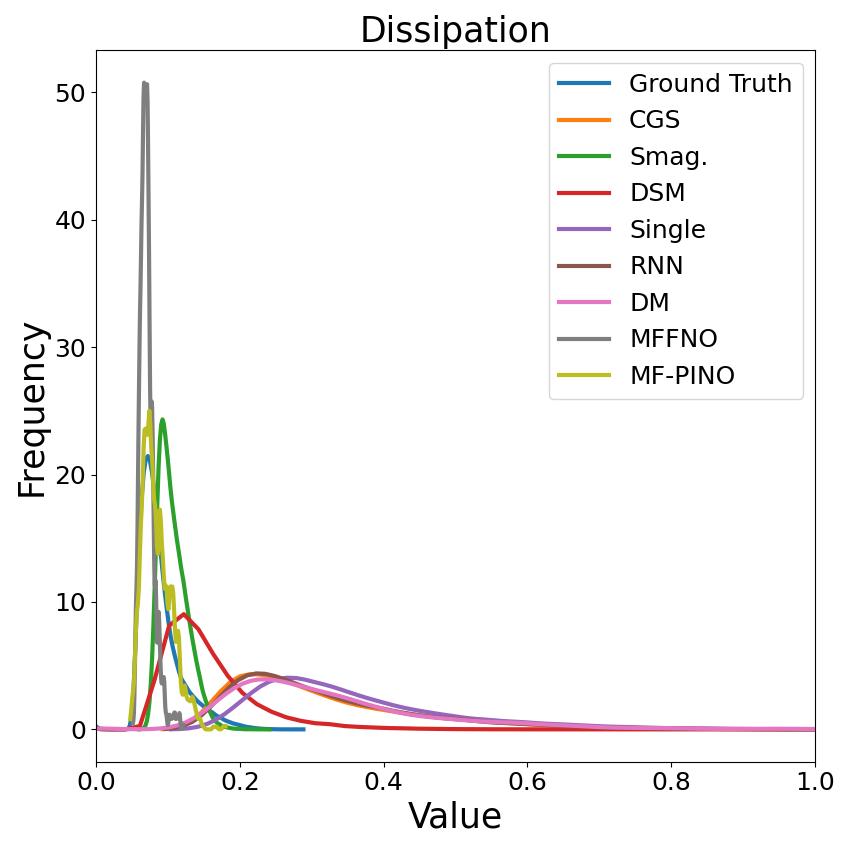}
  \caption{Dissipation Distribution}%\label{fig:sub2}
\end{subfigure}\hfil
\begin{subfigure}{0.499\textwidth}
  \includegraphics[width=\linewidth]{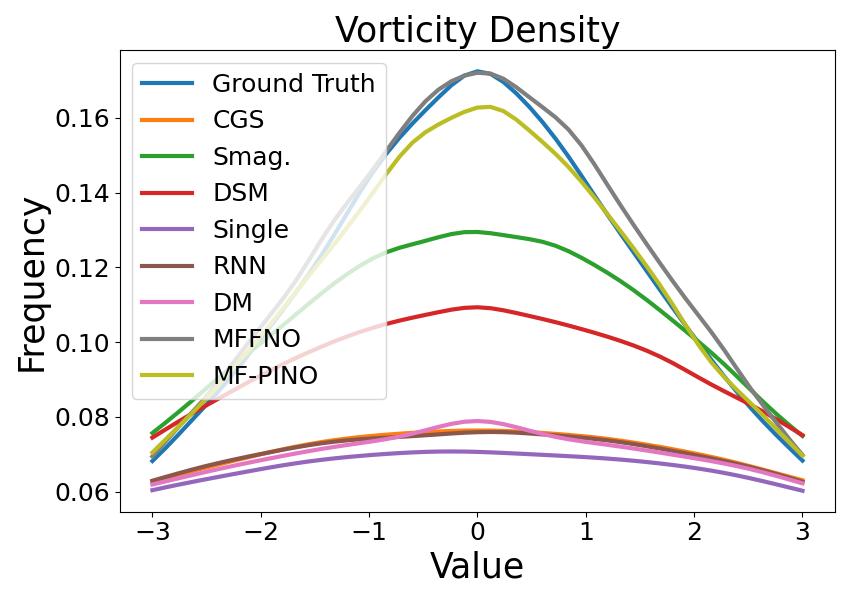}
  \caption{Vorticity Distribution }%\label{fig:sub5}
\end{subfigure}\hfill
\begin{subfigure}{0.499\textwidth}
  \includegraphics[width=\linewidth]{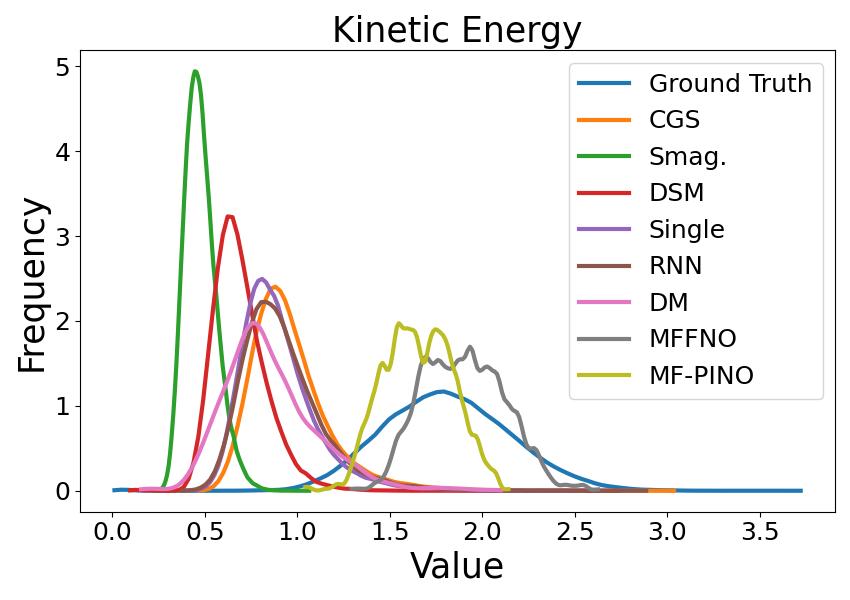}
  \caption{Kinetic Energy Distribution}%\label{fig:sub5}
\end{subfigure}\hfill
\begin{subfigure}{0.499\textwidth}
  \includegraphics[width=\linewidth]{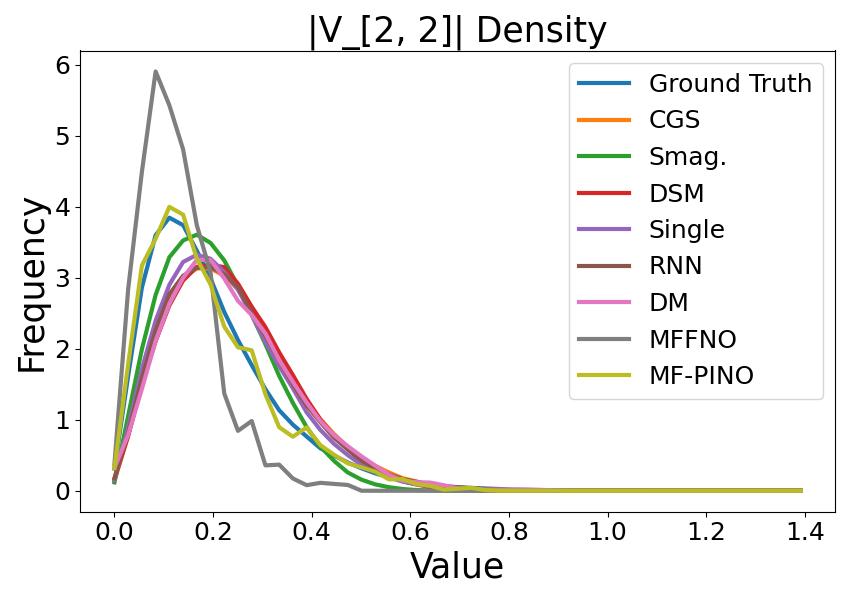}
  \caption{Distribution of component for $(2,2)$ Fourier basis }%\label{fig:sub5}
\end{subfigure}\hfill
\begin{subfigure}{0.499\textwidth}
  \includegraphics[width=\linewidth]{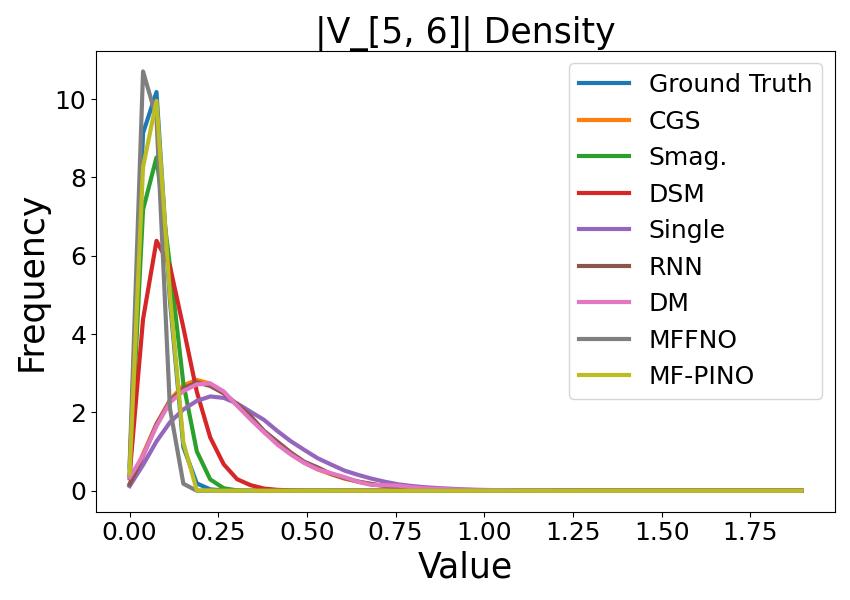}
  \caption{Distribution of component for $(5,6)$ Fourier basis}%\label{fig:sub5}
\end{subfigure}
\caption{\textbf{Experiment Results for NS Equation with $Re=100$.}
The blue curves refer to fully-resolved simulation and serve as ground truth. 
Our primary method with MF-PINO (yellow) is closest to ground truth among all coarse-grid methods.}
\label{fig:apdx: ns stat}
\end{figure}

%%%%%%%%%%%%%%%%%5
\newpage
\subsubsection{Navier-Stokes Equation with Reynolds number $1.6\times 10^4$}
\label{apdx_visul_ns1w}
The error of all statistics we considered for NS equation ($Re=1.6\times 10^4$) is listed in \Cref{tab: ns1w:apdx} and plotted in \Cref{fig:apdx: ns1w stat}. 
% We also compare the inference time of all methods, as is in \Cref{tab:ns1w_time}.

The visualization of TV error for each (marginal) distribution is shown in \Cref{fig:apdx_ns1w: TV}.

We want to clarify here that, due to the system reduction nature of coarse-grid simulation, it cannot precisely recover the invariant measure $\mu^*$ even in ideal cases. Instead, as discussed in previous sections, the best approximation of $\mu^*$ one can obtain with coarse-grid simulations is $P_{\#}\mu^*$ (\Cref{apdx_prop_c5}), where $P$ is the orthogonal projection towards the filtered space $\mathcal{F}(\hhh)$. The gap between $\mu^*$ and $P_{\#}\mu^*$ heavily depends on the reduced space, determined by
the filter $\mathcal{F}$. Consequently, it is possible that estimations of certain statistics based on coarse-grid simulations can never get reasonably close to the ground truth.

%%%%%%%%%%

\begin{table}[htbp]
\centering
\setlength{\tabcolsep}{1.99pt}
\caption{\textbf{Error on Different Statistics: NS equation, $Re=1.6\times 10^4$.} From left to right: Average relative error on energy spectrum, max relative error on energy spectrum, total variation distance from (ground truth) vorticity distribution, average component-wise TV distance(error), relative error on root mean square of velocity, and relative error on the variance of vorticity.}
\begin{tabular}{lcccccc}\toprule
Method & Avg. Eng.($\%$) & Max Eng.($\%$) & Vorticity & Avg. TV & Mean Vel.($\%$) & Variance($\%$) \\\midrule
CGS (No closure) & 139.5876  & 332.6311  & 0.0326 & 0.3642 & 15.9057  & 16.1087  \\
Smagorinsky & 32.8391  & 80.4986  & 0.0351 & 0.2051 & 29.8473  & 55.3507  \\
Single-state & 186.1610  & 422.6007  & 0.0525 & 0.4527 & 19.8919  & \textbf{11.7190 } \\
Dynamic Smag. (DSM) & 17.3967  & 71.3727  & 0.0143 & 0.1614 & 12.3541  & 33.5422  \\
History-aware (RNN) & 72.6070  & 189.4918  & 0.0291 & 0.3611 & 9.2850  & 5.2505  \\
Stochastic (DM) & 91.3567  & 240.6798  & 0.0312 & 0.3425 & 8.7776  & 8.8782  \\
Multi-Fidelity FNO(MFFNO) & 50.2734  & 101.6366  & 0.0237 & 0.5592 & 2.8406  & 27.2564  \\\midrule
\textbf{MF-PINO} & \textbf{13.9455 } & \textbf{34.6842 } & \textbf{0.0109} & \textbf{0.1401} & \textbf{0.2827 } & 18.7052  \\\bottomrule
\end{tabular}

\label{tab: ns1w:apdx}
\end{table}

%%%%%

% \begin{table}[ht]
% \centering
% \caption{\textbf{Inference Time for Navier-Stokes with $Re=1.6\times 10^4$}.
% The comparison is based on the time cost (seconds) for the simulation of one trajectory with $t\in[0,100]$. 
% % 'FRS' refers to fully-resolved simulation, and serves as ground truth for estimating statistics. 'CGS': coarse-grid simulation(no closure model). 'Smag.': classical Smagorinsky model. 'Single': learning-based single-state closure model. 'DSM': dynamical Smagorinsky model. 'MFF': multi-fidelity FNO. Our method is the most efficient among all methods.
% }
% \begin{tabular}{|l|c|c|c|c|c|c|c|c|c|}
% \hline
% Method & FRS & CGS & Smag. &DSM& Single-state & RNN & DM & {MF-FNO} & {MF-PINO} \\ \hline
%  Time [s] & 525.74 & 14.29 & 29.27&71.35 & 96.78 & 54.23 &320.96 & \textbf{1.59} & \textbf{1.59} \\ \hline
% \end{tabular}
% \label{tab:ns1w_time}
% \end{table}
%%%%%%%%%%%%%%%%

%%%%%%%%%%%%%%%%%%%%%%%%%%
\begin{figure}[htbp]
\centering

\begin{subfigure}{0.311\textwidth}
  \includegraphics[width=\linewidth]{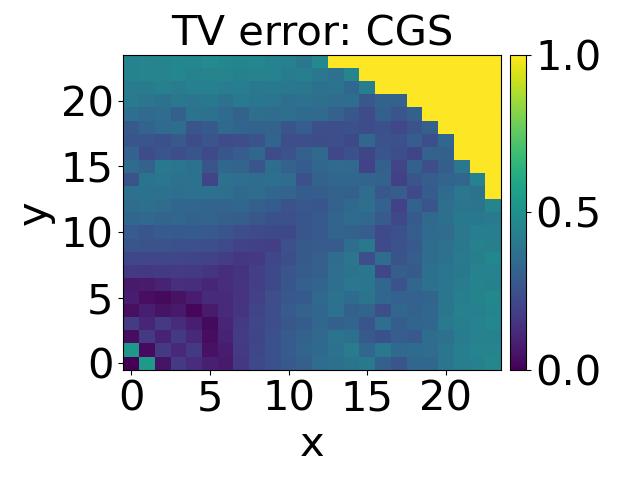}
\end{subfigure}\hfil % \hfil for filling the middle
\begin{subfigure}{0.311\textwidth}
  \includegraphics[width=\linewidth]{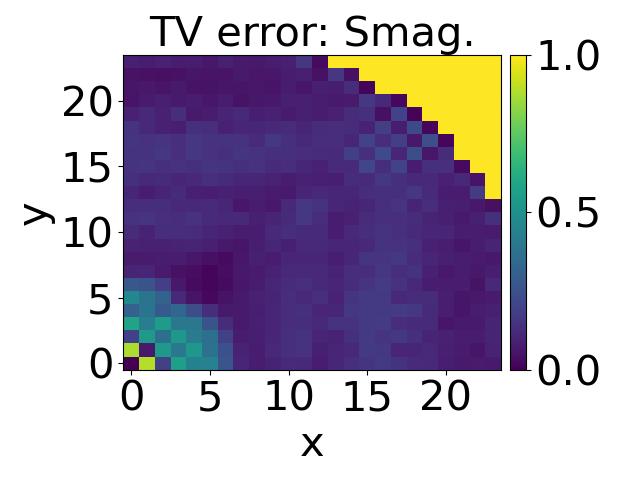}
\end{subfigure}\hfil
\begin{subfigure}{0.311\textwidth}
  \includegraphics[width=\linewidth]{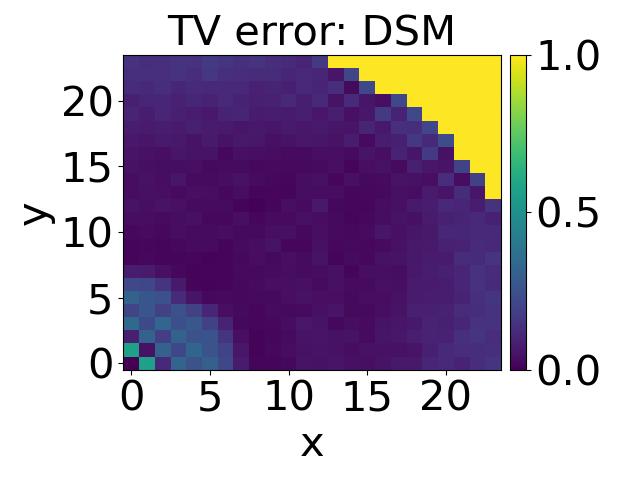}
\end{subfigure}\hfil
\begin{subfigure}{0.311\textwidth}
  \includegraphics[width=\linewidth]{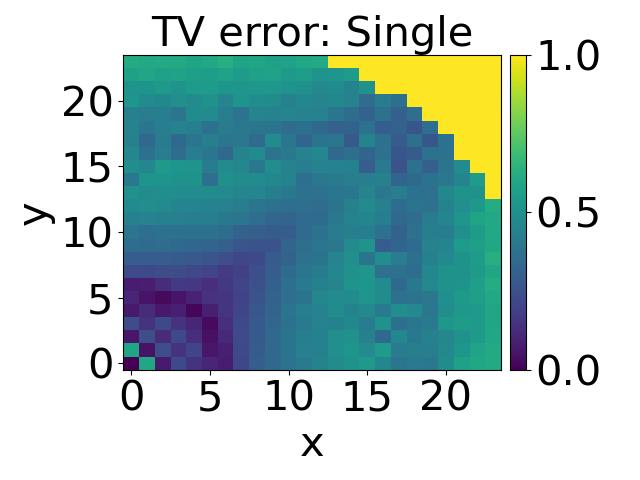}
\end{subfigure}\hfil
\begin{subfigure}{0.311\textwidth}
  \includegraphics[width=\linewidth]{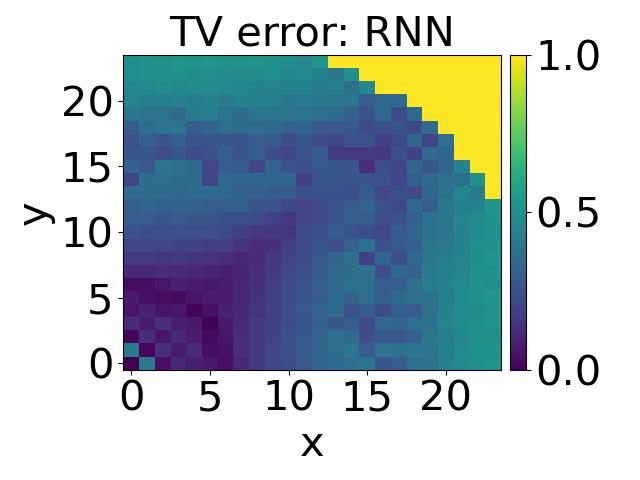}
\end{subfigure}\hfil
\begin{subfigure}{0.311\textwidth}
  \includegraphics[width=\linewidth]{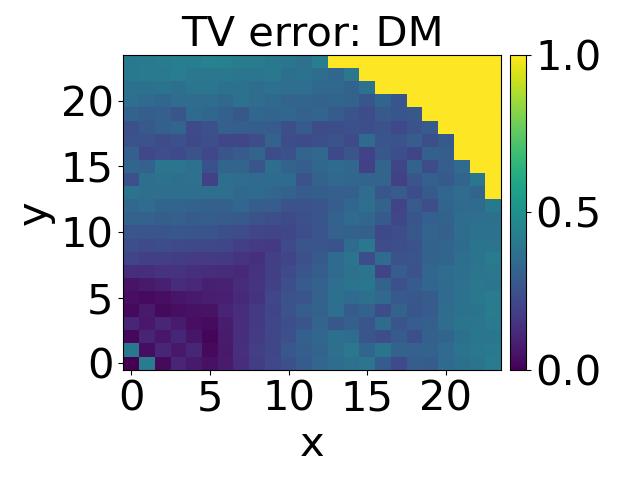}
\end{subfigure}\hfil
\begin{subfigure}{0.311\textwidth}
  \includegraphics[width=\linewidth]{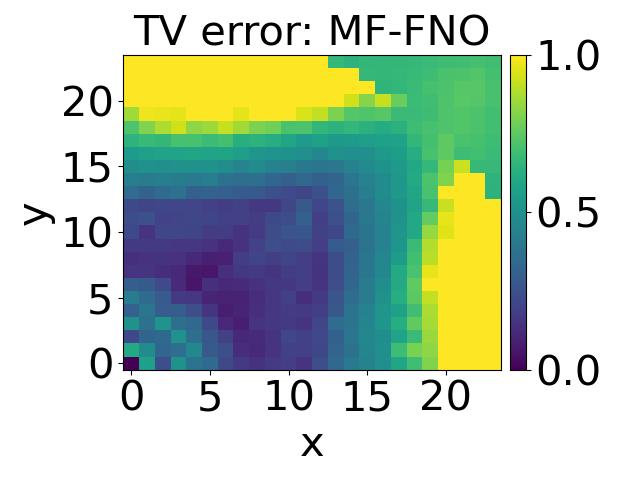}
\end{subfigure}\hfil
\begin{subfigure}{0.311\textwidth}
  \includegraphics[width=\linewidth]{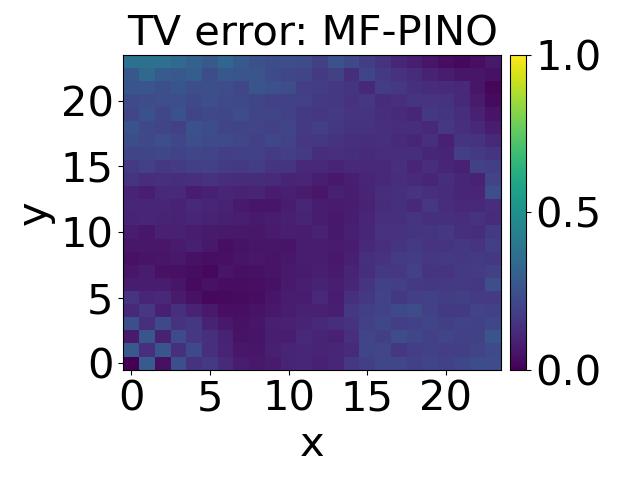}
\end{subfigure}
\caption{\textbf{Total Variation (TV) error for NS Equation with $Re=1.6\times 10^4$}
The $(k,j)$-element represents the TV error (compared with the results obtained from FRS) regarding the distribution of the mode length of the component for $(k,j)$ Fourier basis $e^{i\frac{2\pi}{L}(kx+jy)}$.
 `CGS': coarse-grid simulation(no closure model). `Smag.': classical Smagorinsky model. `DSM': dynamical Smagorinsky model. `Single': learning-based single-state closure model. `RNN': history-aware closure model with RNN. `DM': stochastic closure with diffusion model.
 `MFFNO': multi-fidelity FNO.
Our primary method with MF-PINO achieves the smallest TV error among all modes, indicating that it obtains the best approximation of the filtered invariant measure among all coarse-grid methods considered.
}
\label{fig:apdx_ns1w: TV}
\end{figure}

\begin{figure}[htbp]
\centering

\begin{subfigure}{0.499\textwidth}
  \includegraphics[width=\linewidth]{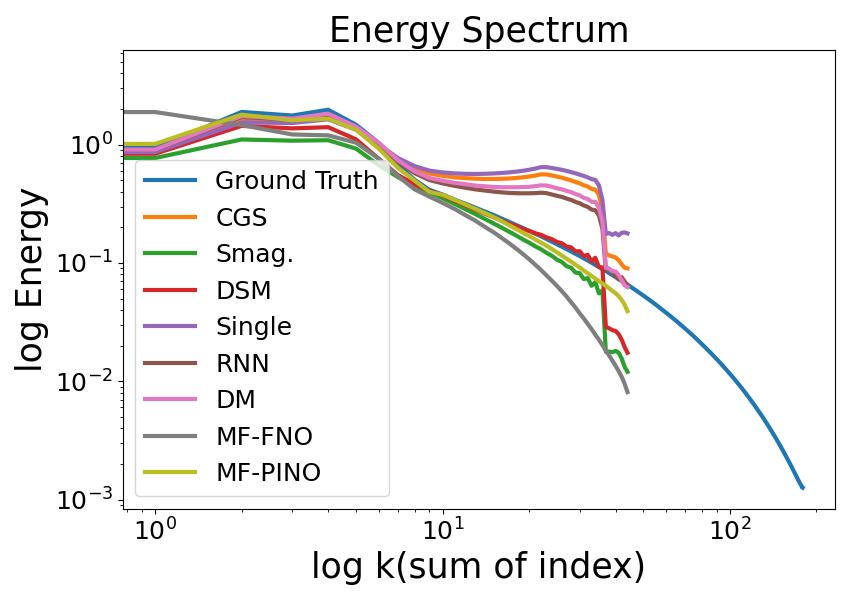}
  \caption{Energy Spectrum}%\label{fig:sub1}
\end{subfigure}\hfil % \hfil for filling the middle
\begin{subfigure}{0.499\textwidth}
  \includegraphics[width=\linewidth]{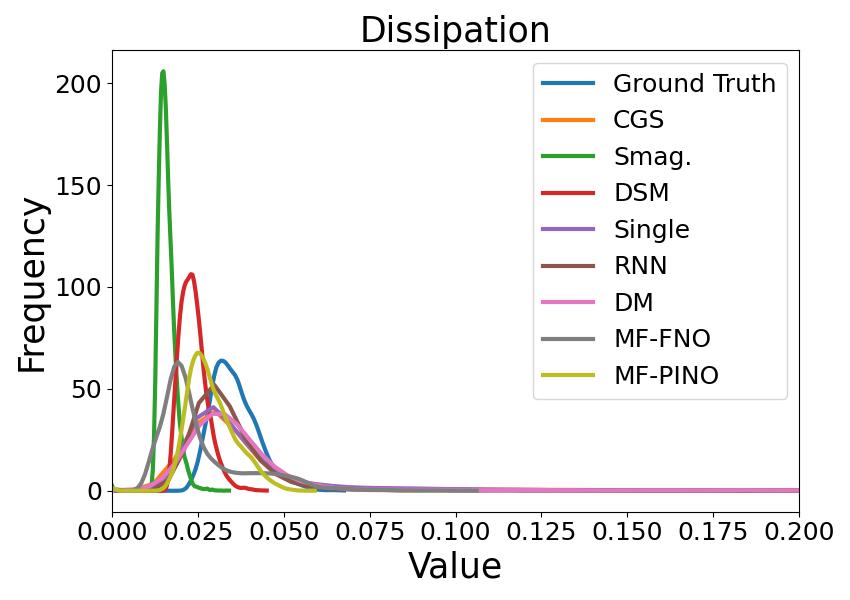}
  \caption{Dissipation Distribution}%\label{fig:sub2}
\end{subfigure}\hfil
\begin{subfigure}{0.499\textwidth}
  \includegraphics[width=\linewidth]{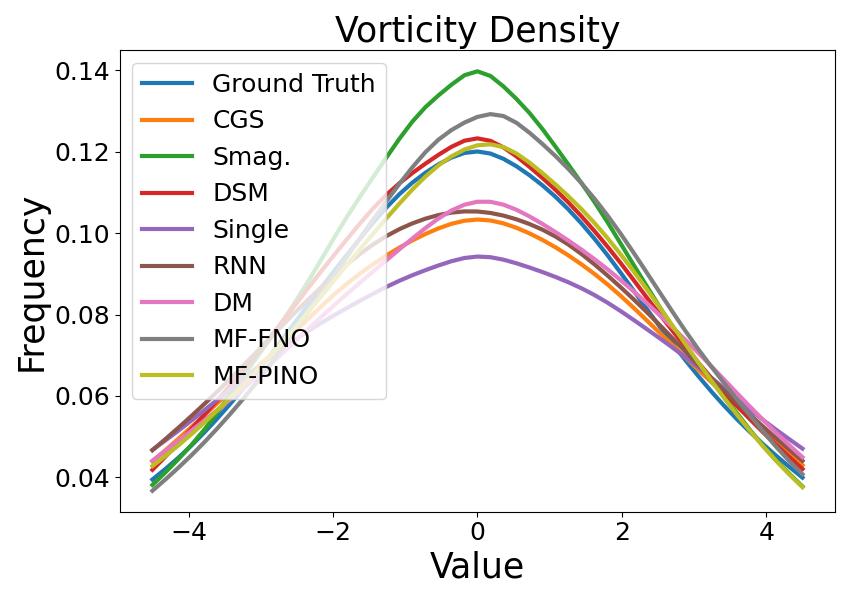}
  \caption{Vorticity Distribution }%\label{fig:sub5}
\end{subfigure}\hfill
\begin{subfigure}{0.499\textwidth}
  \includegraphics[width=\linewidth]{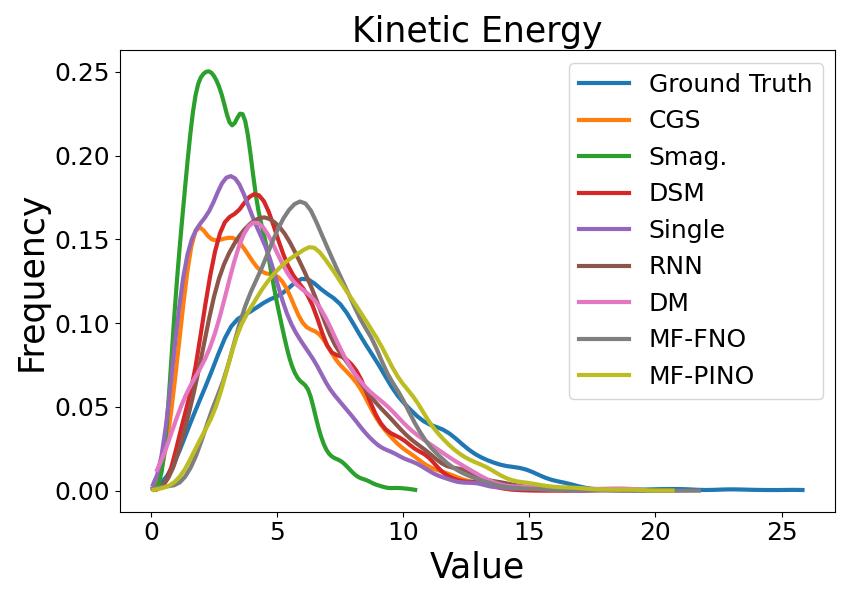}
  \caption{Kinetic Energy Distribution}%\label{fig:sub5}
\end{subfigure}\hfill
\begin{subfigure}{0.499\textwidth}
  \includegraphics[width=\linewidth]{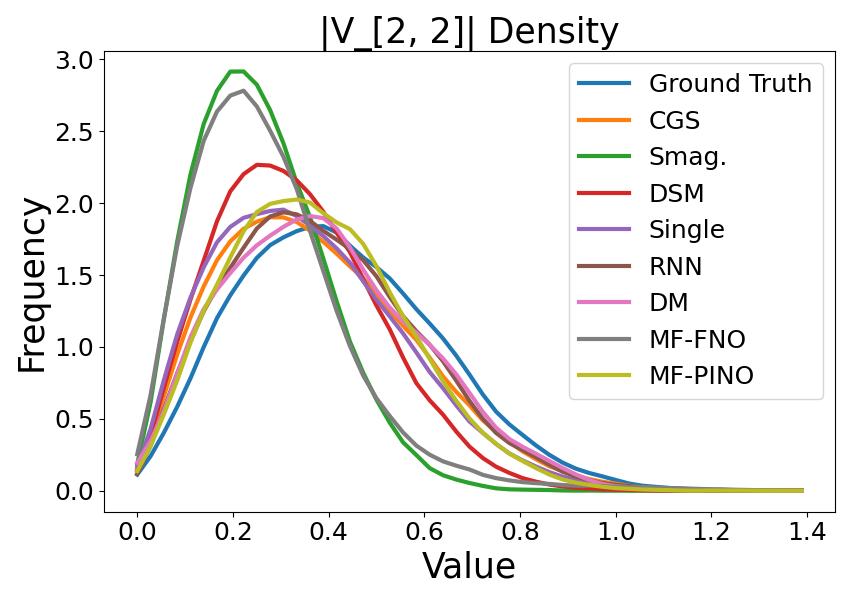}
  \caption{Distribution of component for $(2,2)$ Fourier basis }%\label{fig:sub5}
\end{subfigure}\hfill
\begin{subfigure}{0.499\textwidth}
  \includegraphics[width=\linewidth]{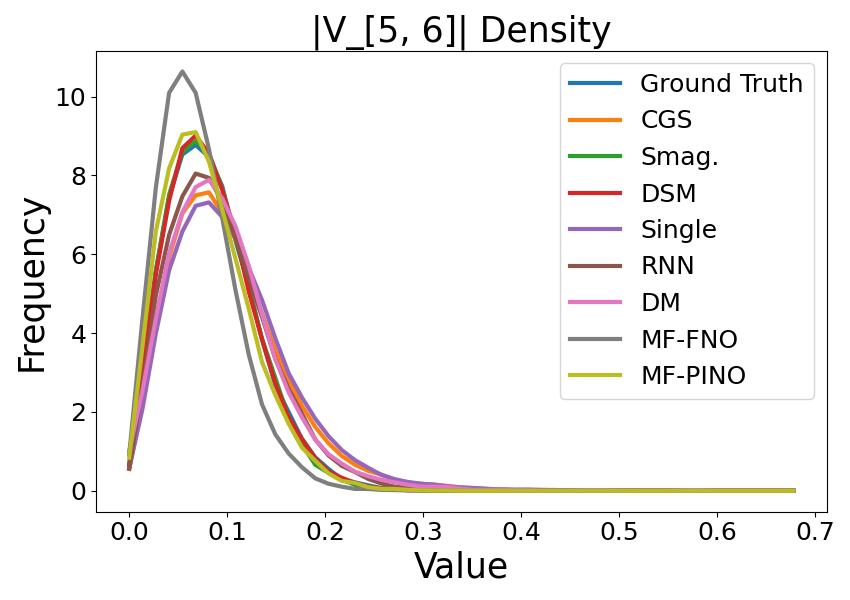}
  \caption{Distribution of component for $(5,6)$ Fourier basis}%\label{fig:sub5}
\end{subfigure}
\caption{\textbf{Experiment Results for NS Equation with $Re=1.6\times 10^4$.}
The blue curves refer to fully-resolved simulation and serve as ground truth. 
Our primary method with MF-PINO (yellow) is closest to ground truth among all coarse-grid methods.
}
\label{fig:apdx: ns1w stat}
\end{figure}

\subsubsection{Additional FourCastNet 3 Results}
\label{apdx_fcn3}
Additional reasults obtained with the multi-fidelity training of FourCastNet 3 are presented. \Cref{fig:fcn3_crps_panel} and \Cref{fig:fcn3_ssr_panel} depict the continuously ranked probability score (CRPS) and spread-skill ratio (SSR) for the low-fidelity, the multi-fidelity and high-fidelity variants of FourCastNet 3. \Cref{fig:fcn3_spectra_panel} depicts the respective angular powerspectra. Finally \Cref{fig:fcn3_fieldmap_panel} and \Cref{fig:fcn3_fieldmap_panel_720h} depict rollouts of an individual ensemble member for the multi-fidelity and high-fidelity variants.

\begin{figure}[htbp]
\centering
\includegraphics[width=1\linewidth]{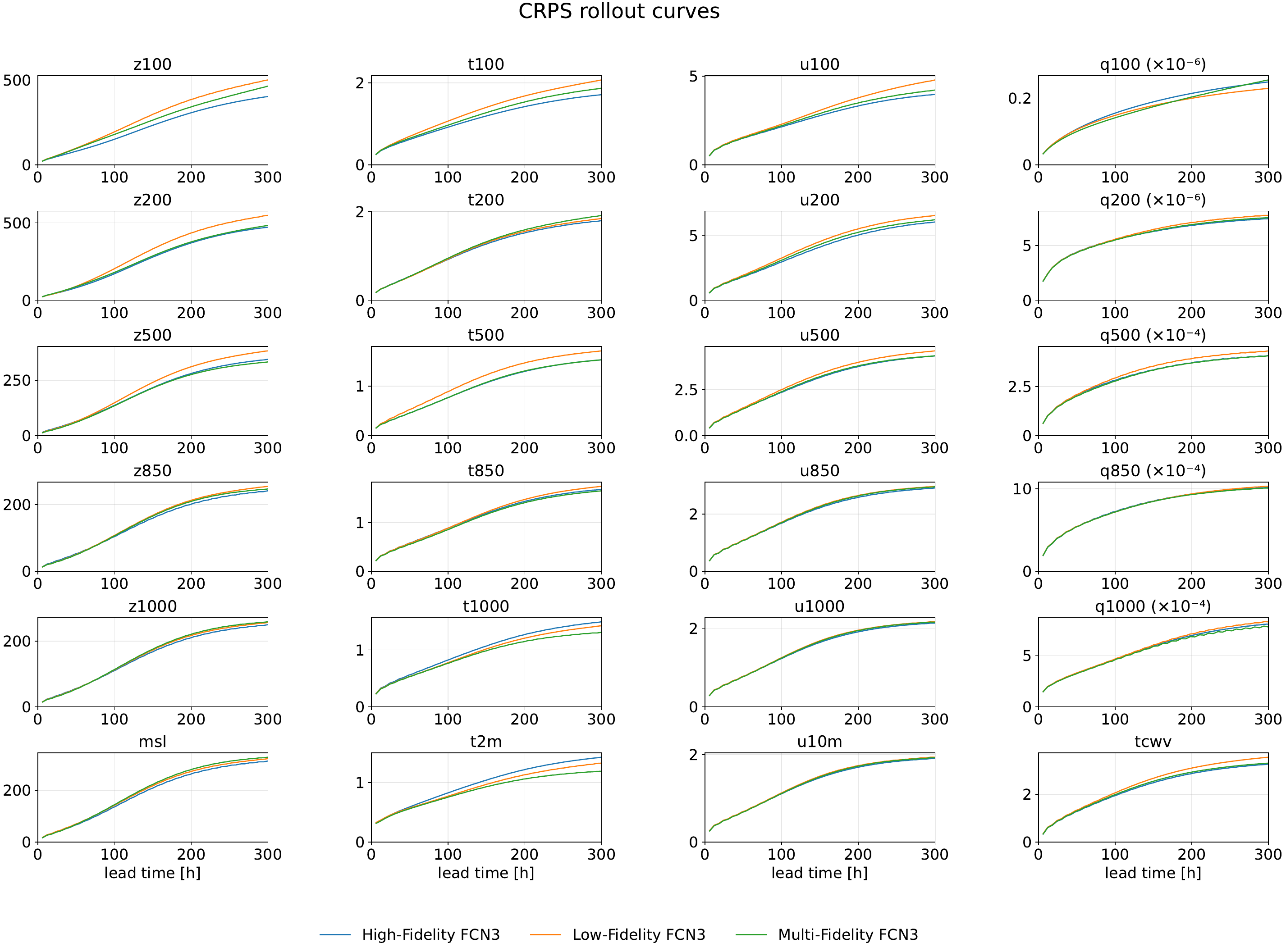}
\caption{
\textbf{Comparison of multi-fidelity and high-fidelity FourCastNet 3 ensemble member predictions.} Both forecasts are initialized at 2018-01-01 00:00:00 UTC and rolled out 20 time
}
\label{fig:fcn3_crps_panel}
\end{figure}

% \begin{figure}[htbp]
% \centering
% \includegraphics[width=1\linewidth]{fig/fcn3/fcn3_rmse_panel.pdf}
% \caption{
% \textbf{Comparison of multi-fidelity and high-fidelity FourCastNet 3 ensemble member predictions.} Both forecasts are initialized at 2018-01-01 00:00:00 UTC and rolled out 20 time
% }
% \label{fig:fcn3_rmse_panel}
% \end{figure}

\begin{figure}[htbp]
\centering
\includegraphics[width=1\linewidth]{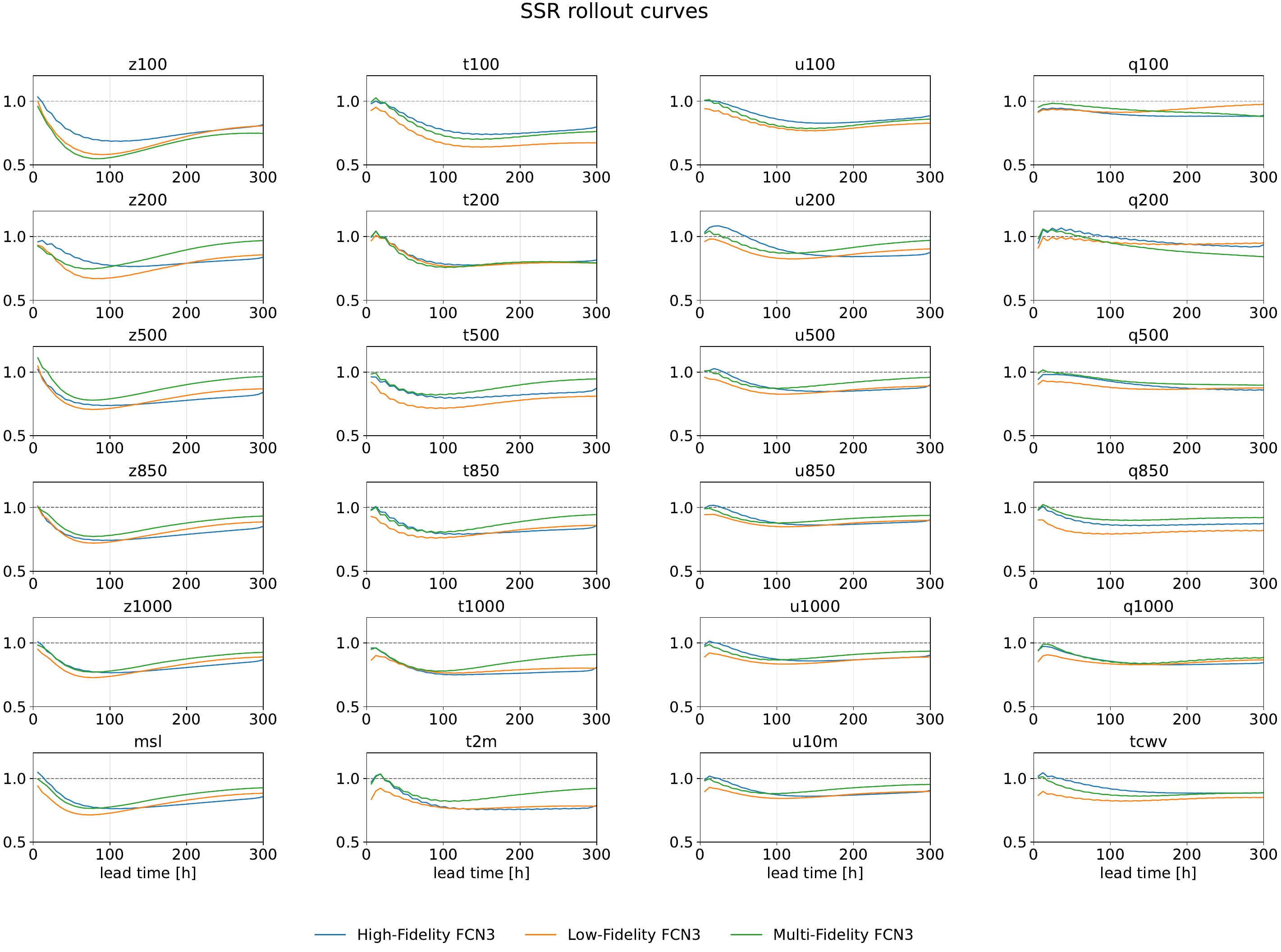}
\caption{
\textbf{Comparison of multi-fidelity and high-fidelity FourCastNet 3 ensemble member predictions.} Both forecasts are initialized at 2018-01-01 00:00:00 UTC and rolled out 20 time
}
\label{fig:fcn3_ssr_panel}
\end{figure}

\begin{figure}[htbp]
\centering
\includegraphics[width=1\linewidth]{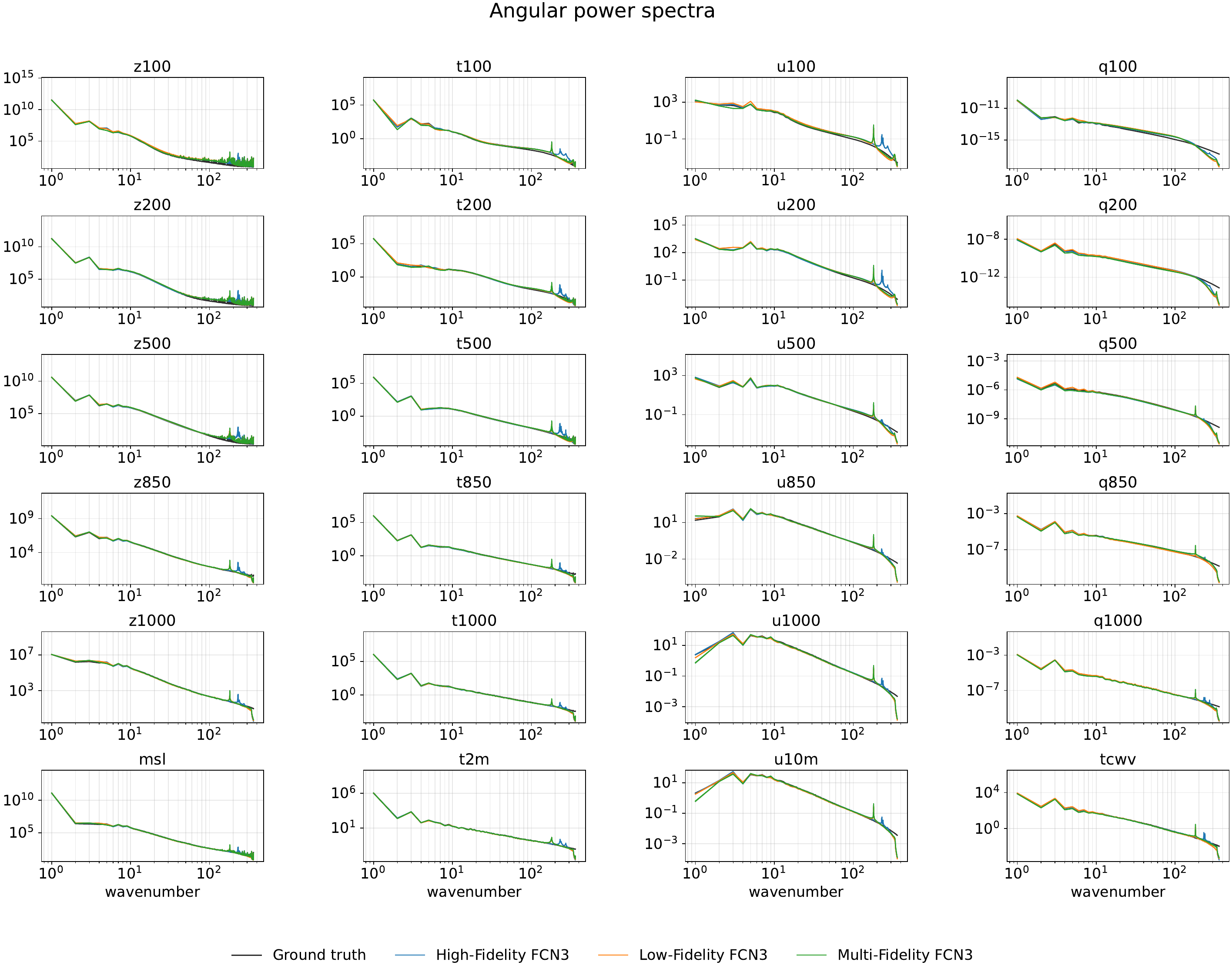}
\caption{
\textbf{Angular power spectra of multi-fidelity and high-fidelity FourCastNet 3 ensemble member predictions at a lead time of 300h.} Power spectra are aggregated over 12-hourly initial conditions in the validaiton year of 2018 and compared against the ERA5 groundtruth.
}
\label{fig:fcn3_spectra_panel}
\end{figure}

\begin{figure}[htbp]
\centering
\includegraphics[width=1\linewidth]{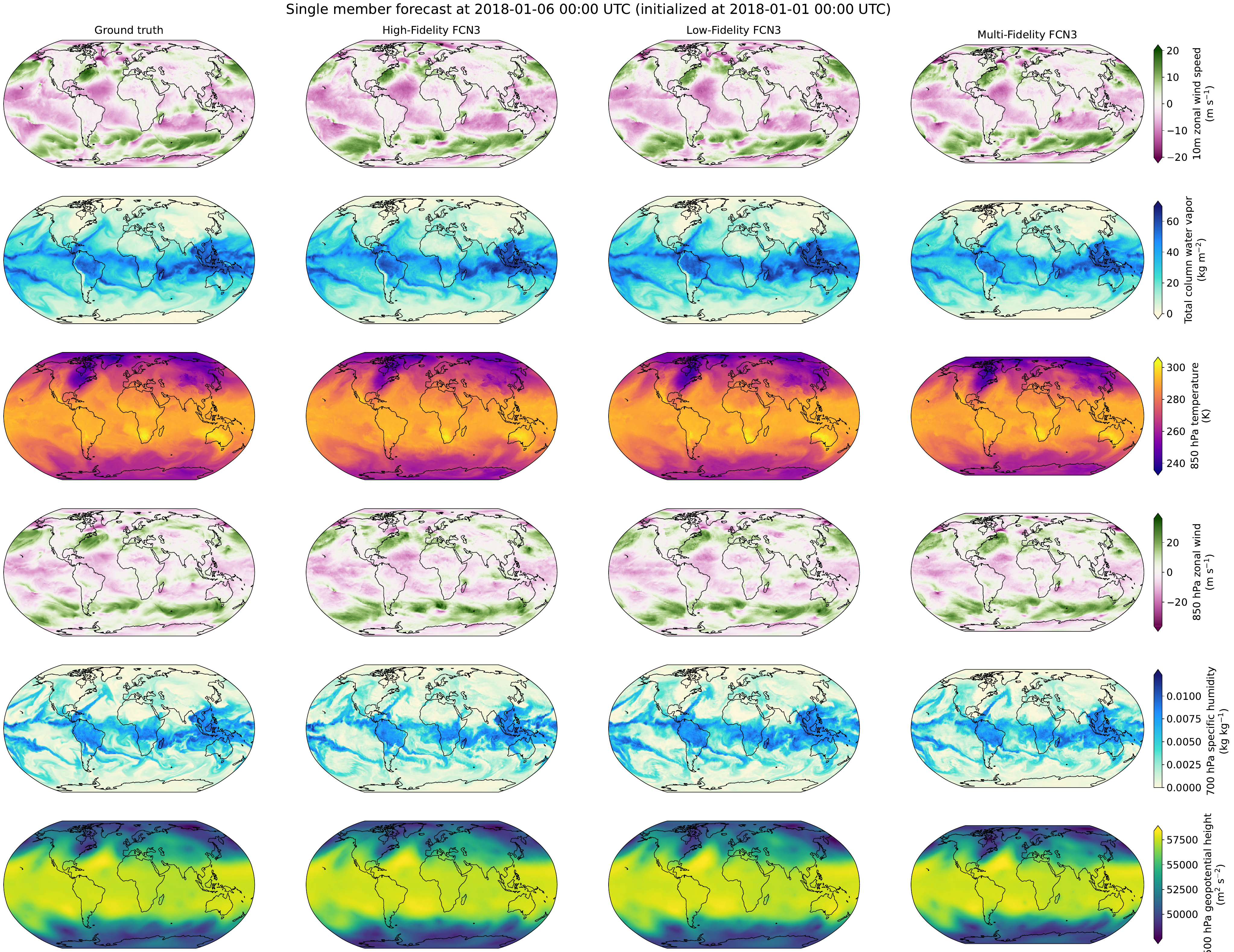}
\caption{
\textbf{High-fidelity, low-fidelity and multi-fidelity FourCastNet 3 ensemble member predictions at 120h lead time.} Both forecasts are initialized at 2018-01-01 00:00:00 UTC and autoregressively rolled out for 20 timesteps.}
\label{fig:fcn3_fieldmap_panel}
\end{figure}

\begin{figure}[htbp]
\centering
\includegraphics[width=1\linewidth]{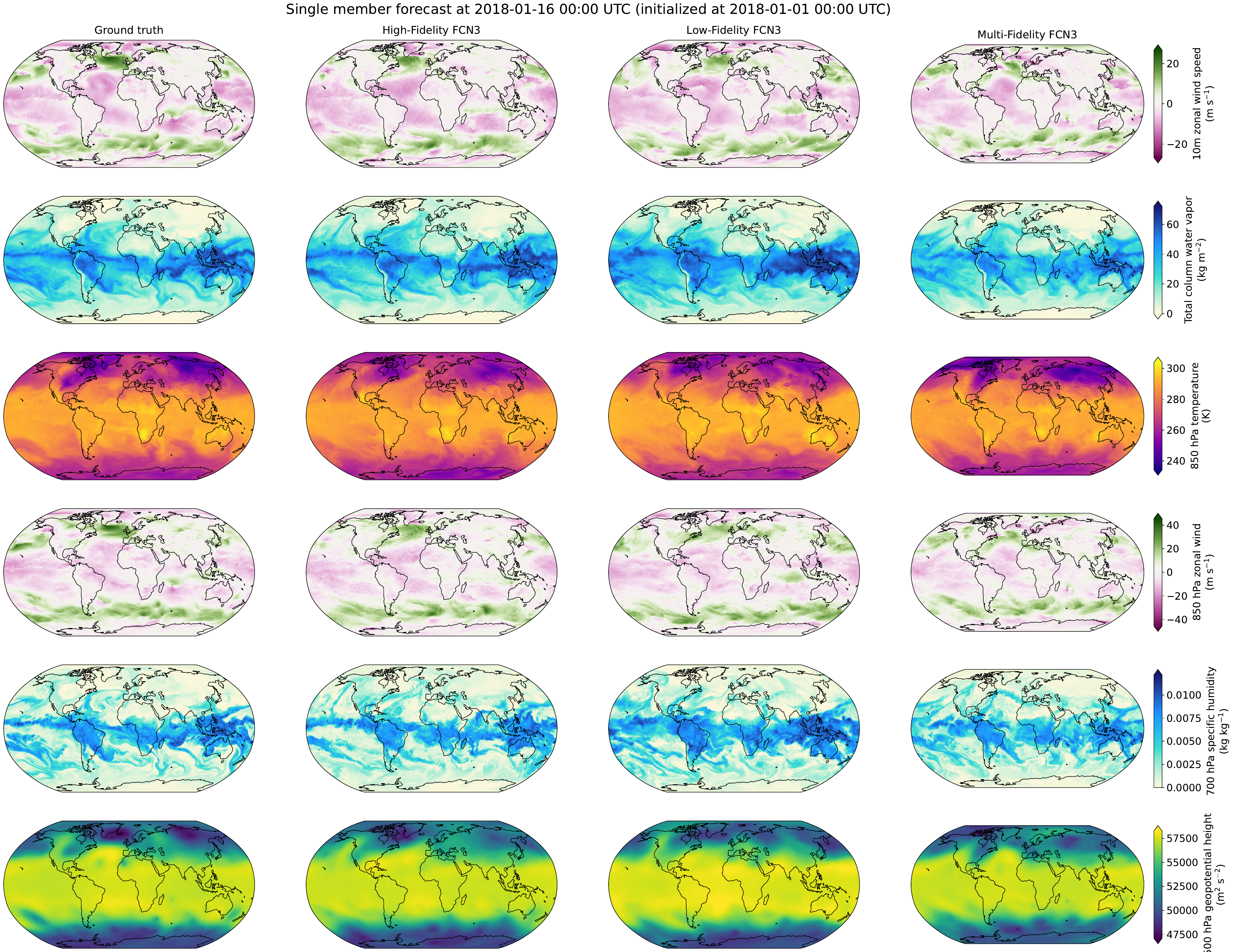}
\caption{
\textbf{High-fidelity, low-fidelity and multi-fidelity FourCastNet 3 ensemble member predictions at 360h lead time.} Both forecasts are initialized at 2018-01-01 00:00:00 UTC and autoregressively rolled out for 60 timesteps.}
\label{fig:fcn3_fieldmap_panel_36`0h}
\end{figure}

\newpage
\section{Training Instability of Learning-based Closure Models}\label{apdx: instable}

\begin{figure}[t]
\centering
\includegraphics[width=1\linewidth]{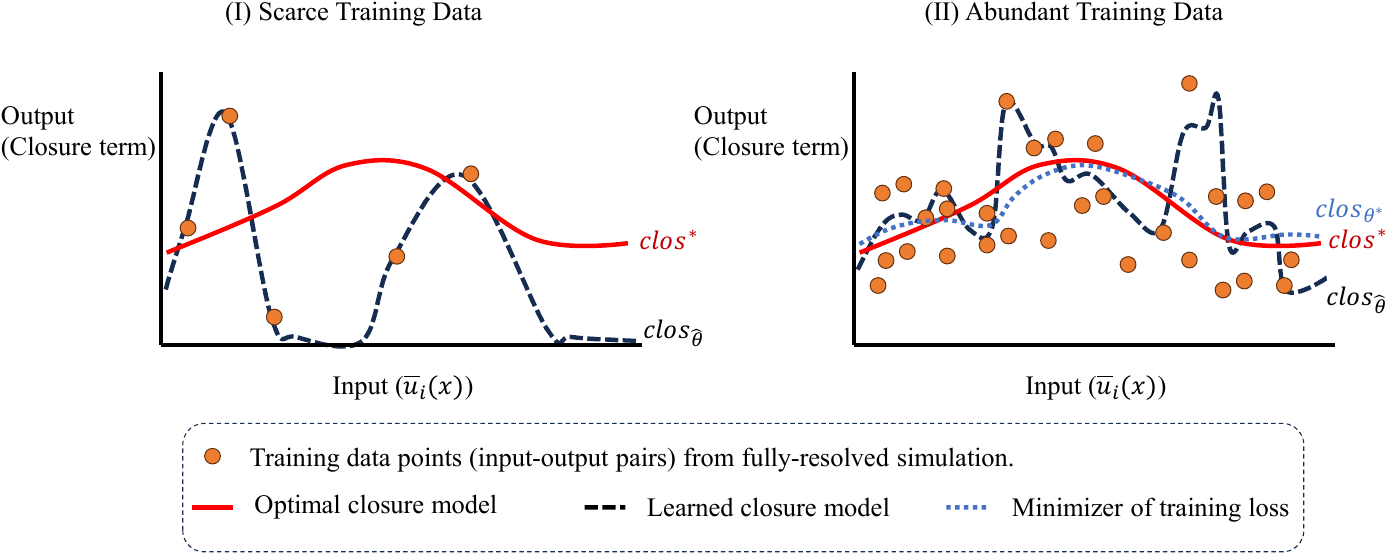}
\caption{
\textbf{Illustration of training results of learning-based closure models} (I) In the limited-data regime, after minimizing the loss function, the model $clos_{\hat{\theta}}$ memorizes the closure terms from training data, and has outputs nearly zero for inputs distant from any training data points. 
(II) In the abundant-data regime, though the minimizer of the training loss ($clos_{\theta^*}$) is a good approximation of the optimal closure model $clos^*$, since the optimization is done with mini-batch, we usually  only get a sub-optimal local minimal of the training loss, corresponding to $clos_{\hat\theta}$.
}
\label{fig:instable}
% \vspace{-2em}
\end{figure}

In this section, we discuss the training instability of learning-based closure models. An illustration of the interpretation is shown in \Cref{fig:instable}.

Please refer to \Cref{thm-lower-notation} for notations used in this section.

For simplicity, we will conduct our discussion in $\ell^2$ representation space of $\hhh$. $\rvv$ corresponds to filtered functions in the reduced space and $\rvw$ corresponds to the unresolved components.

Recall that as discussed in \Cref{apdx-b-3} and \Cref{apdx: thm_liouv} the optimal closure model, which existing learning-based closure modeling approaches are implicitly targeting, is 
\begin{equation}
    clos^*(\rvv)=\eee_{\rho^*(\rvw|\rvv)}\big[g(\rvv,\rvw) \big| \rvv\big],
\end{equation}
where $g(\rvv,\rvw):=f_r(\rvv,\rvw)-f_r(\rvv,0)$.

Assuming the model capacity being sufficient, the minimizer of the training loss
\begin{equation}
    J_{ap}(\theta;\mathfrak{D})=\frac {1} {|\mathfrak{D}|}\sum\limits_{i\in\mathfrak{D}}\|clos(\rvv^i;\theta)-g(\rvv^i,\rvw^i)\|^2,
\end{equation}
is the conditional expectation w.r.t. the empirical measure,
\begin{equation}
    clos_{\theta^*}(\rvv)=\eee_{\hat\rho(\rvw|\rvv)}\big[g(\rvv,\rvw) \big| \rvv\big],
\end{equation}
where $\hat{\rho}$ is the empirical measure of these trainign data points $\mathfrak{D}=\{(\rvv^i,\rvw^i)\}$.

We have analyzed in \Cref{apdx: thm_liouv} that under the limited-data regime, due to the large gap between $\hat{\rho}$ and $\rho^*$, $clos_{\theta^*}$ suffers from a considerable approximation error. In this regime, there might not necessarily be different data points $(\rvv^i,\rvw^i)$ and $(\rvv^j,\rvw^j)$ having a similar input $\rvv^i\approx\rvv^j$ for the closure model, which is identified as the non-uniqueness issue we have shown in \Cref{thm_clos_all}(i). 
However, it is important to emphasize that the desired output corresponding to input $\rvv^i$ is $clos^*(\rvv^i)$, rather than the specific training label $g(\rvv^i,\rvw^i)$. In practice, minimizing the loss leads the model to memorize the observed output $g(\rvv^i,\rvw^i)$ in the training data. For regions of the input space not covered by any training data—an expected situation given the high dimensionality and data scarcity—the model lacks meaningful supervision and tends to revert to its initializations, often producing near-zero or default outputs that have little physical relevance. This is illustrated in the left figure of \cref{fig:instable}

Even under the ideal setting where a sufficient amount of fully-resolved training data is available, i.e. $\hat{\rho}\approx\rho^*$, learning a data-driven closure model suffers from training instability. Though $clos_{\theta^*}$ is a good approximation of $clos^*$ in this setting, we will show that $clos_{\hat\theta}$ the model one gets after training is usually a sub-optimum of the training loss instead of the global minima $clos_{\theta^*}$.

First of all, as identified in \Cref{apdx: thm_liouv}, due to the non-unique issue, the minimal value of the training loss is not zero, but the conditional variance $\texttt{Var}_{\rvw\sim \hat{\rho}(\rvw|\rvv)}[g(\rvv,\rvw)|\rvv]$, which one does not know a priori. Consequently, it becomes difficult to assess whether the training process has converged to the true global minimum or merely settled at a suboptimal local minimum.

\begin{figure}[t]
\centering
\includegraphics[width=0.5\linewidth]{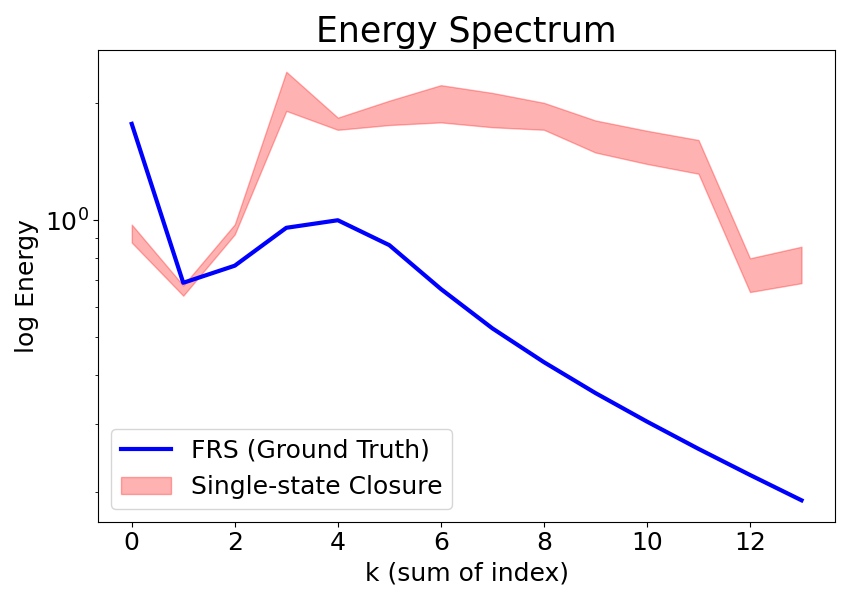}
\caption{
Prediction of energy spectrum with learning-based single-state closure model using different random seed during training.
}
\label{fig:instable2}
% \vspace{-2em}
\end{figure}

Moreover, unlike conventional supervised learning where target labels are fixed, in this setting the effective label $clos^*(\rvv^i)$ is itself an average over multiple realizations, namely the conditional expectation, which is not accessible in practice. Instead, mini-batch optimization algorithms (e.g., SGD, Adam) rely on individual samples $g(\rvv^i,\rvw^i)$, introducing significant variance in the estimation of gradients and, consequently, of the true loss landscape. See \cref{fig:instable} for illustration.

Empirically, we observe that training loss tends to stay still after a number of epochs. However, its converged value is highly sensitive to the random seed for model initialization. Importantly, models trained under different seeds all fail to yield accurate statistical predictions, underscoring the practical instability of this approach. See \cref{fig:instable2} for our experiment results. We train the closure model in $Re=100$ setting with $2.8\times10^5$ data points.

\clearpage
% \begin{refcontext}[sorting = none]
\printbibliography[heading=subbibintoc, title={Appendix References}]
% \end{refcontext}

%ncomm cite
\end{refsection}
% \clearpage

\end{document}